\documentclass[accepted]{uai2021}
\usepackage[american]{babel}
\usepackage{natbib} % has a nice set of citation styles and commands
\usepackage{mathtools} % amsmath with fixes and additions
\usepackage{booktabs} % commands to create good-looking tables
\usepackage{tikz} % nice language for creating drawings and diagrams

\usepackage{multirow}
\usepackage{amsfonts} 
\usepackage{caption}
\usepackage{subcaption}
\usepackage{xcolor}
\usepackage{bm}
\usepackage{bbm}
\usepackage{amsmath}
\usepackage{mathtools, nccmath}
\usepackage{amsthm}
\usepackage{xparse}
\usepackage{tcolorbox}
\usepackage{comment}
\usepackage{wrapfig}
\DeclarePairedDelimiterXPP\Expect[2]{\mathbb{E}_{#1}}[]{}{#2}%

\DeclareMathOperator*{\btheta}{\bm{\theta}}

\DeclareMathOperator{\bx}{\mathbf{x}}
\DeclareMathOperator{\bz}{\mathbf{z}}
\DeclareMathOperator{\bw}{\mathbf{w}}
\DeclareMathOperator{\bb}{\mathbf{b}}
\DeclareMathOperator{\ba}{\mathbf{a}}
\DeclareMathOperator{\bV}{\mathbf{V}}
\DeclareMathOperator{\bW}{\mathbf{W}}
\DeclareMathOperator{\pprime}{{\prime\prime}}

\DeclareMathOperator{\balpha}{\bm{\alpha}}

\DeclareMathOperator{\bPhi}{\bm{\Phi}}
\DeclareMathOperator{\bDelta}{\bm{\Delta}}
\DeclareMathOperator{\bbnorm}{\bigg|\bigg|}

\newcommand{\indicator}[1]{\mathbbm{1}(#1)}
\newcommand{\bind}{\bm{\mathbbm{1}}}

\newtheorem{theorem}{Theorem}
\newtheorem{definition}{Definition}

\newtheorem{lemma}{Lemma}
\newtheorem{proposition}{Proposition}
\newcommand*\mystrut[1]{\vrule width0pt height0pt depth#1\relax}

\usepackage{url}

\usepackage{breakurl}
\usepackage[breaklinks]{hyperref}

\usepackage{amsthm}
\newtheorem{innercustomgeneric}{\customgenericname}
\providecommand{\customgenericname}{}
\newcommand{\newcustomtheorem}[2]{%
  \newenvironment{#1}[1]
  {%
   \renewcommand\customgenericname{#2}%
   \renewcommand\theinnercustomgeneric{##1}%
   \innercustomgeneric
  }
  {\endinnercustomgeneric}
}

\newcustomtheorem{customlemma}{Lemma}

\title{Know Your Limits: Uncertainty Estimation with ReLU Classifiers Fails at Reliable OOD Detection}

\author[1]{\href{mailto:<dennis.ulmer@mailbox.org>}{Dennis~Ulmer}{}} % Lead author
\author[2]{\href{mailto:<giovanni.cina@pacmed.nl>}{Giovanni~Cin\`a}{}}
\affil[1]{%
    ITU Copenhagen\\
    Copenhagen, Denmark
}
\affil[2]{%
    Pacmed BV\\
    Amsterdam, Netherlands\\
}
\begin{document}
\maketitle

\begin{abstract}
    A crucial requirement for reliable deployment of deep learning models for safety-critical applications is the ability to identify out-of-distribution (OOD) data points, samples which differ from the training data and on which a model might underperform. Previous work has attempted to tackle this problem using uncertainty estimation techniques. However, there is empirical evidence that a large family of these techniques do not detect OOD reliably in classification tasks.
    
    This paper gives a theoretical explanation for said experimental findings and illustrates it on synthetic data. We prove that such techniques are not able to reliably identify OOD samples in a classification setting, 
    since their level of confidence is generalized to unseen areas of the feature space.
    % provided the models satisfy assumptions about the partial monotonicity of the input and resulting class probabilities. 
    This result stems from the interplay between the representation of ReLU networks as piece-wise affine transformations, the saturating nature of activation functions like softmax, and the most widely-used uncertainty metrics. 
    % We derive results in both the binary and multi-class classification case for four different methods.
\end{abstract}

%{\let\thefootnote\relax\footnotetext{*The authors contributed equally.}}

\section{Introduction}

Notwithstanding the tremendous improvements achieved in recent years by means of novel and larger deep learning architectures, advanced models still lack certain properties that guarantee their safety in high-stakes applications like health care \citep{he2019practical}, autonomous driving \citep{mcdermid2019towards}, and more. 
Among other traits, the capability to discern familiar data samples seen during the training stage (in-distribution) from abnormal inputs (out-of-distribution) is of  paramount importance in certain contexts. Take for instance a hospital, in which an algorithm is used to predict complications for a patient. Due to factors like changing patient demographics or protocols, but also simply different hospital environments, predictions might become less reliable and cause harm to the patient. A degradation of the model performance might only be detected much later, when the shift in the test data becomes more apparent - at which point further damage accumulates, hence the need arises to implement techniques that can detect OOD samples reliably.

Unfortunately, it is well-known that neural network classifiers tend to be overconfident in their predictions \citep{guo2017calibration}, i.e. exhibiting high levels of certainty when it is unwarranted, and often fail to correctly identify OOD samples \citep{ovadia2019can, nalisnick2018deep}. 
A recent study on medical tabular data has shown that even  techniques specifically developed to quantify the model's uncertainty struggle at detecting OOD samples for a relatively simple classification task \citep{ulmer2020trust}. Crucially, it was shown that neural discriminators tend to project vast areas of high certainty far away from the training distribution - a behaviour that seems completely at odds with reliable OOD detection. 
These observations can easily be replicated on synthetic data, as displayed in Figure \ref{subfig:uncertainty}, where one can observe open areas of constant certainty stretching beyond the training data. The reasons for this behavior in a classification setting are hitherto much less studied.

%\begin{figure}
%    \centering
%    \includegraphics[width=0.4\textwidth]{img/old/multiclass_single.png}
%    \caption{Predictive entropy by a single neural discriminator on a multi-class classification problem. The degree of uncertainty about the prediction is given in increasingly darker shades of purple, with white signifying low uncertainty / high certainty.}
%    \label{fig:multiclass-entropy}
%\end{figure}
\begin{figure*}[h]
    \begin{subfigure}[t]{0.32\textwidth}
        \includegraphics[width=\textwidth]{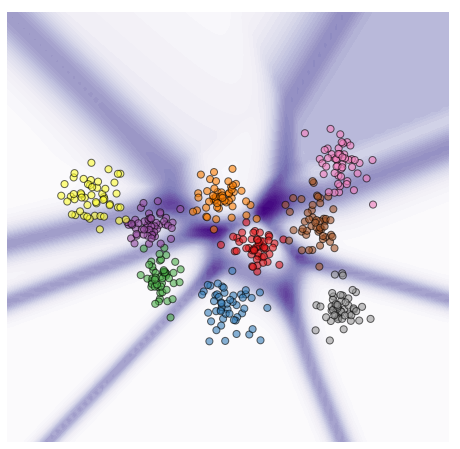}
        \caption{Predictive entropy $\tilde{\mathbb{H}}[p_{\btheta}(y|\bx)]$ of ReLU classifier.}
        \label{subfig:uncertainty}
    \end{subfigure}
    \hfill
    \begin{subfigure}[t]{0.32\textwidth}
        \includegraphics[width=\textwidth]{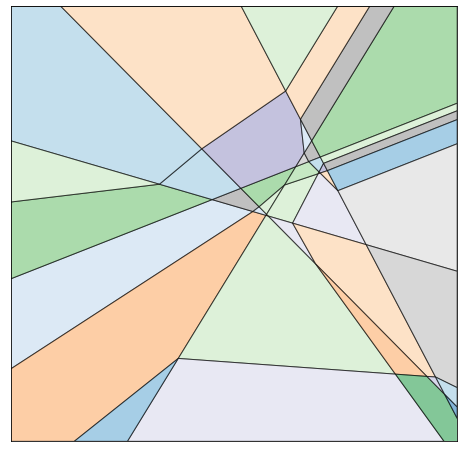}
        \caption{Polytopal linear regions induced by  same classifier \citep{arora2018understanding}.}
        \label{subfig:polytopes}
    \end{subfigure}
    \hfill
    \begin{subfigure}[t]{0.32\textwidth}
        \includegraphics[width=\textwidth]{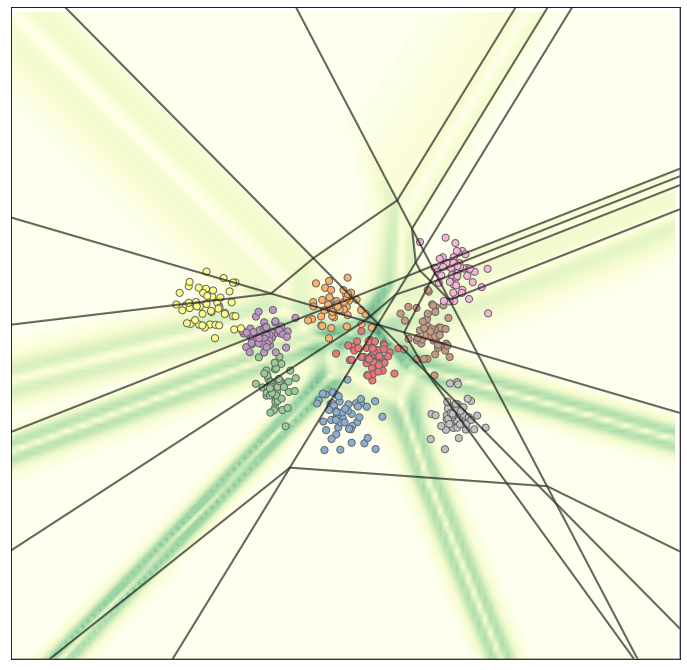}
        \caption{Magnitude of gradient of predictive entropy $||\nabla_{\bx}\tilde{\mathbb{H}}[p_{\btheta}(y|\bx)] ||_2$.}
        \label{subfig:grad-polytopes}
    \end{subfigure}
    \caption{(a) Uncertainty of a neural classifier with ReLU activations measured by predictive entropy on synthetic data, illustrated by increasing shades of purple with white denoting absolute certainty. (b) Polytopal, linear regions in the feature space induced by the same classifier (as introduced by \citet{arora2018understanding}). (c) Norm of the gradient of the predictive entropy plotted by increasing shades of green, showing how small perturbations in the input have a decreasing influence on the uncertainty of the network as we stray away from the training data, creating large areas in which uncertainty levels are overgeneralized. Exceptions to this are the model's decision boundaries, which is discussed in Section \ref{sec:experiments}. Polytopes are drawn using the code of \citet{jordan2019provable}.}\label{fig:intro-figure}
\end{figure*}

In this paper we propose a novel theoretical argument to explain such phenomena, showing that certainty levels are generalized on sub-spaces defined by the network (see Figure \ref{subfig:polytopes} and \ref{subfig:grad-polytopes}). We do this by simulating covariate shift for single feature values of real variables and studying the asymptotic behavior of the model.
Our contributions are as follows:

\begin{enumerate}
    \item Our first result shows that, under mild assumptions about the network's behaviour on certain subspaces, ReLU-based neural network classifiers coupled with widely used uncertainty metrics always converge to a fixed uncertainty level on OOD samples.
    \item We extend this result by proving that variational inference-based and ensembling methods in combination with several uncertainty estimation techniques suffer from the same problem (Theorem \ref{main-theorem}). This phenomenon is illustrated and discussed on synthetic data.
    % \item In the case of a single network, under similar assumptions, we show that the network tends to high certainty in a single class.
\end{enumerate}

These results entail that, when the conditions of the theorem are met, these models cannot be used to reliably detect OOD: since the level of certainty is generalized from seen to unseen data, the models are unable to differentiate between the two. The findings of this article have bearings on OOD detection for several critical applications using neural classifiers with ReLU activation functions.
%(see the discussion in Section \ref{sec:discussion}) \extodo{We don't have these examples in the paper anymore}.

\section{Related Work}

Overconfidence in neural networks has been studied from several angles; we summarize here some of the main lines of research.
One way to counteract the problem of overconfidence lies in the quantification of a network's uncertainty, which is usually divided into aleatoric and epistemic uncertainty. The former denotes \emph{non-reducible} uncertainty, e.g. uncertainty intrinsic to the data generating process, while the latter refers to \emph{reducible} uncertainty. This includes the knowledge about the best model class, as well as about the parameters which explain the data best \citep{der2009aleatory, hullermeier2019aleatoric}. 
In this article, we are interested in the effect of methods that estimate uncertainty post-hoc. 
%There are recent approaches which try to account for both types in their architecture, using prior \cite{malinin2018predictive} and posterior networks \cite{charpentier2020posterior}, these however remain outside the scope of this work. 
One way to evaluate uncertainty estimation methods is the study of their behaviour in presence of OOD samples \citep{ovadia2019can}. 
% Meijerink et al. \cite{meijerink2020uncertainty} and 
\citet{ulmer2020trust} specifically show how the methods used in our work fail in practice to detect clinically relevant OOD groups of patients in a medical context. \citet{kompa2020empirical} conclude that many uncertainty methods produce confidence intervals which do not include the actual observations on OOD data.

A variety of articles approaches the phenomenon of overconfidence from a calibration perspective: starting with the work of  \citet{guo2017calibration}, follow-up work develops improved variants of temperature scaling \citep{laves2019well} or new types of scaling altogether \cite{kumar2018trainable, kull2019beyond}.
%as well as novel training procedures altogether \cite{thulasidasan2019mixup}. 
A separate line of enquiry investigates the effect of different activation functions 
\citep{bridle1990probabilistic, ramachandran2017searching}, exploring alternatives to the sigmoid or softmax function\footnote{Their relationship is outlined in more detail in Appendix \ref{app-softmax-sigmoid-connection} or \citet{bridle1990probabilistic}.}  as the final component of neural discriminators \citep{de2015exploration, hendrycks2017baseline, laha2018controllable, martins2016softmax}. 
%Some of them are motivated by the so-called \emph{softmax bottleneck} issue %\cite{rawat2019sampled,blanc2018adaptive}
%, arguing that it hinders the networks representational capacity \cite{yang2018breaking}.
%,kanai2018sigsoftmax} 
%or leads to premature saturation during training \cite{chen2017noisy}. Kumar et al.  \cite{kumar2018vonmises} motivate their proposed alternative of a continuous embedding output layer by the computational disadvantages that occur when the number of classes is large, like in the field of natural language processing when predicting the next token out of a language's dictionary. General properties of the softmax function in the context of reinforcement learning were discussed by \cite{gao2017properties}, suboptimal convergence rates in the same context as well for general classification where theoretically proven by Mei et al. \cite{mei2020escaping}. Many authors note the relationship between softmax and network overconfidence  
%\cite{aigrain2019softmax}
 But to the best of our knowledge, only \citet{hein2019relu} give a theoretical explanation for this behavior in ReLU-networks showing that, when a point is scaled uniformly, the network's probability mass is placed on a single class. However, they do not extend their findings to uncertainty estimation metrics nor investigate the implications for variational approaches or ensembling. % and do not investigate monotonic behaviour of the network the feature space. 
 \citet{mukhoti2021deterministic} show that epistemic and aleatoric uncertainty cannot be disentangled successfully purely based on the output distribution of a single network.

\section{Preliminaries}

In this section we introduce some notation and relevant definitions for the rest of this work.
%We first introduce some necessary background information: in Section \ref{sec:notation} we introduce notation, in Section \ref{sec:dataset-shift} we explain the notions of data set shift, in Section \ref{sec:uncertainty-metrics} we define the uncertainty metrics we will study, and in Section \ref{sec:general-definitions} we provide some additional definitions that will be used in the rest of this work.
\subsection{Notation}\label{sec:notation}

We denote sets in calligraphic letters, e.g. $\mathcal{P}$ or $\mathcal{C}$. Vectors are represented using lower-case bold letters such as $\bx$ or $\btheta$. 
%For simplicity, we distinguish instantiations from their variables by adding a subscript or a superscript, such as $\bx_i$ and $\bx^\prime$. 
For functions with multiple outputs, a lowercase index refers to a specific component of the output, e.g. with a function $f: \mathbb{R} \rightarrow \mathbb{R}^N$, $f(x)_n$ denotes one of the output components with $n \in 1, ..., N$. Furthermore, we use $\odot$ to denote the Hadamard product and $||\ldots||_2$ for the $l_2$-norm.

\subsection{Out-of-Distribution Data}\label{sec:dataset-shift}

Given a set $\mathcal{C} = \{1, \ldots, C\}$ of numbers denoting class labels, we define a data set to be a finite set $\mathcal{D} \subset \mathbb{R}^D \times \mathcal{C}$ containing $N$ ordered pairs of $D$-dimensional feature vectors $\bx_i$ and corresponding class labels $y_i$ obtained from some (unknown) joint distribution $\bx_i, y_i\ \sim\ p(\bx, y)$ s.t. $\mathcal{D} = \{(\bx_i, y_i)\}_{i=1}^N$. 
    %In the following we are concerned with the case that $y_i \in \{0, 1\}$ 
%\end{definition}
Although there exist many different notions of dataset shift %\cite{shimodaira2000improving,
\citep{shimodaira2000improving, moreno2012unifying}, we particularly focus on \emph{covariate shift}, in which the distribution of feature values - the covariates - differs from the original training distribution $p(\bx)$. %Hence, in what follows a data point is considered OOD whenever it originates from a distribution that has undergone covariate shift with respect to $p(\bx)$.  
We focus on this kind of shift as it is especially common in non-stationary environments like healthcare \citep{curth2019transferring} and many other applications.
%such as image classification \cite{taori2020measuring} or autonomous agents \cite{amodei2016concrete}.\\  %Whenever the test distribution has shifted away from the training distribution, points can be considered OOD. 

To simulate covariate shift, we obtain OOD samples by shifting points away from the training distribution by means of a scaling factor. This approach is in line with recent experiments on covariate shift and OOD detection \citep{ovadia2019can,ulmer2020trust}.
We would expect a reliable OOD detection model to output increasingly higher uncertainty as points stray further and further away from the mass of $p(\bx)$, thus we study the behaviour of OOD detection models in the limit, when the scaling factor is allowed to grow indefinitely in at least one dimension.

\subsection{Uncertainty Metrics}\label{sec:uncertainty-metrics}

We begin by defining a neural discriminator, which we assume to follow common architectural conventions, i.e. to consist of a series of affine transformations with ReLU \citep{glorot2011deep} activation functions. Together with a final softmax function \citep{bridle1990probabilistic}, it  parametrizes a categorical distribution over classes.

\begin{definition} \label{def:logit}
    Let $\bx \in \mathbb{R}^D$ be an input vector and $C$ the number of classes in a classification problem.
    The unnormalized output of the network after $L$ layers is a function $f_{\btheta}: \mathbb{R}^D \rightarrow \mathbb{R}^C$ with the final output following after an additional softmax function $\bar{\sigma}(\cdot)$ s.t. $p_{\btheta} = \bar{\sigma}\circ f_{\btheta}$, so $p_{\btheta}(y=c|\bx) \equiv \bar{\sigma}(f_{\btheta}(\bx))_c$. Thus the discriminator is represented by a function $p_{\btheta}: \mathbb{R}^D \rightarrow [0, 1]^C$ which is parametrized by a vector $\btheta$.
\end{definition}

The softmax function $\bar{\sigma}: \mathbb{R}^C \rightarrow \mathbb{R}^C$ is commonly defined as 
\begin{equation}
    \bar{\sigma}(f_{\btheta}(\bx))_c = \frac{\exp(f_{\btheta}(\bx)_c)}{\sum_{c^\prime=1}^C\exp(f_{\btheta}(\bx)_{c^\prime})}
\end{equation}
We will consider a (non-exhaustive) set of popular uncertainty metrics in this work, which we introduce in turn. \citet{hendrycks2017baseline} introduce a simple baseline, which is the highest probability observed for any class:
\begin{equation*}
    y_{\text{max}} = \max_{c \in \mathcal{C}}p_{\btheta}(y=c|\bx)
\end{equation*}
Ideally, the model's predictive distribution would become more uniform for challenging inputs (e.g. in areas of class overlap) and thus produce a lower confidence score  $y_{\text{max}}$.\footnote{Which is why we measure \emph{un}certainty by $1-y_{\text{max}}$.} The other uncertainty estimation techniques introduced below try to approximate the uncertainty of the predictive distribution for a new data point $\bx^\prime$, which is commonly factorized as follows:
\begin{equation}\label{eq:pred-dist-fac}
    p(y|\bx^\prime, \mathcal{D}) = \int p(y|\bx^\prime, \btheta)p(\btheta|\mathcal{D})d\btheta 
\end{equation}
In the following, we use $p(y|\bx, \btheta) \equiv p_{\btheta}(y|\bx)$. This equation is intractable: the weight posterior $p(\btheta|\mathcal{D})$ cannot be computed precisely using Bayes' rule. Hence the weight posterior is often replaced by a variational posterior $q(\btheta)$ and the expectation formed by this expression is commonly approximated by a Monte Carlo approximation:
\begin{equation}\label{eq:mc-approx}
    \Expect[\Big]{p(\btheta|\mathcal{D})}{p_{\btheta}(y|\bx)} \approx \frac{1}{K}\sum_{k=1}^K p^{(k   )}_{\btheta}(y|\bx);\quad {\btheta}^{(k)} \sim q(\btheta)
\end{equation}
Other methods for which a similar aggregation of predictions can be used include Markov-chain Monte-Carlo procedures \citep{welling2011bayesian, involutive2020neklyudov} 
or simply ensembling \citep{lakshminarayanan2017simple}, which does not require sampling. In this work, we will use the term \emph{instance} to refer to a network characterized by $\btheta^{(k)}$ in order to group all of these methods together. We follow the reasoning of \citet{wilson2020bayesian} to interpret both variational and ensembling methods as examples of Bayesian model averaging. A very intuitive method to aggregate all instances' predictions and estimate uncertainty is to compute the variance over all classes, such as in  \citet{gal2018understanding}: 
\begin{equation*}\begin{aligned}
    \bar{\sigma}^2 & = \frac{1}{C}\sum_{c=1}^C \Expect[\bigg]{{p(\btheta|\mathcal{D})}}{\Big(p_{\btheta}(y=c|\bx)\Big)^2}\\
    & - \Expect[\bigg]{{p(\btheta|\mathcal{D})}}{p_{\btheta}(y=c|\bx)}^2
\end{aligned}\end{equation*}
The more predictions disagree, the larger the  average variance per class will become. Another approach lies in measuring the Shannon entropy $\mathbb{H}$ of the predictive distribution \citep{gal2016uncertainty}:
\begin{equation*}
    \tilde{\mathbb{H}}\Big[p_{\btheta}(y|\bx)\Big] = \mathbb{H}\bigg[\Expect[\Big]{{p(\btheta|\mathcal{D})}}{p_{\btheta}(y|\bx)}\bigg]
\end{equation*}
Entropy is maximal when all probability mass is centered on a single class, i.e. all aggregated predictions agree on a label, and minimum when the predictive distribution is uniform. As these two metrics capture only the total uncertainty, we finally also consider the mutual information between model parameters and a data sample \citep{gal2018understanding}, which aims to isolate epistemic uncertainty:
\begin{equation*}\begin{aligned}
    \underbrace{\mathbb{I}(y, {\btheta}| \mathcal{D}, \bx)}_{\text{Model uncertainty}} & \approx \underbrace{\mathbb{H}\bigg[\Expect[\Big]{{p(\btheta|\mathcal{D})}}{p_{\btheta}(y|\bx)}\bigg]}_{\text{Total uncertainty}}\\
    & - \underbrace{\Expect[\bigg]{{p(\btheta|\mathcal{D})}}{\mathbb{H}\Big[p_{\btheta}(y|\bx)\Big]}}_{\text{Data uncertainty}}
\end{aligned}\end{equation*}
The term itself can be interpreted as the gain in information about the ideal model parameters and correct label upon receiving an input. If we can only gain a little, that implies that parameters are already well-specified and that the epistemic uncertainty is low.

\subsection{Additional Definitions}\label{sec:general-definitions}

% \begin{definition}\label{def:monotonically-increasing-decreasing}
    % A function $f: \mathbb{R} \rightarrow \mathbb{R}$ is called \emph{monotonically increasing} for $x, y \in \mathbb{R}$ if $x \le y \iff f(x) \le f(y)$ and \emph{strictly increasing} iff $f(x) < f(y)$. 

Here we introduce some concepts related to monotonicity, which will become central in the next sections.
In the univariate case, we call a function strictly increasing on an interval $\mathcal{I} = [a, b]$ with $a < b$ and $a, b \in \mathbb{R}$ if it holds that
\begin{equation*}\label{eq:monotonically-increasing}
    \forall x^\prime \in \mathcal{I}\ \bigg( \frac{\partial}{\partial x}f(x)\Big|_{x=x^\prime} > 0 \bigg)
\end{equation*}
 where $\cdot|_{x=x^\prime}$ refers to evaluating the value of the derivative of $f$ at $x^\prime$. 
    %  Conversely, we call the opposite inequalities of the two aforementioned cases  \emph{monotonically decreasing} and \emph{strictly decreasing} functions, respectively.
    % We call a function strictly monotonic on  $\mathcal{I}$ if it is exclusively either strictly increasing or strictly decreasing on the interval.
% \end{definition}
This definition can also be extended to multivariate functions by requiring strict monotonicity (strictly increasing or decreasing) in all dimensions:

\begin{definition}\label{def:monotonicity-multivariate}
    We call a multivariate function $f: \mathbb{R}^D \rightarrow \mathbb{R}$ strictly monotonic on a subspace $\mathcal{P} \subseteq \mathbb{R}^D$ if it holds that
    \begin{equation}\label{eq:monotonicity-multivariate}\begin{aligned}
        \forall d \in 1, \ldots, D,\ &  \Bigg(\forall \bx^\prime \in \mathcal{P},\ \Big(\nabla_{\bx} f(\bx)\Big|_{\bx=\bx^\prime}\Big)_d < 0\ \\ 
        & \vee \ 
        \forall \bx^\prime \in \mathcal{P},\ \Big(\nabla_{\bx} f(\bx)\Big|_{\bx=\bx^\prime}\Big)_d > 0 
        \Bigg)
    \end{aligned}\end{equation}
where $(\cdot)_d$ refers to $\frac{\partial f(x_d)}{\partial x_d}|_{x_d=x^\prime_d}$, the $d$-th component of the gradient $\nabla_{\bx} f(\bx)$ evaluated at $\bx^\prime$. We call a multivariate function $f: \mathbb{R}^D \rightarrow \mathbb{R}^C$ \emph{component-wise strictly monotonic} if the above definition holds for the gradient of every output component $\nabla_{\bx} f(\bx)_c$. 
\end{definition}

We note here that the softmax function is an example for a component-wise strictly  monotonic function. 
%The case in which a whole network is monotonic as per this definition is in fact quite common, see Figure \ref{subfig:nn-maxprob} for an example. 
% Proposal: , ... where the probability for a class increases with the distance from the decision boundary / class overlap.
As later lemmas investigate the behavior of functions in the limit, it is furthermore useful to define regions of the feature space that are unbounded in at least one direction. We call a \emph{partially-unbounded polytope} (henceforth abbreviated by PUP) a convex subspace of $\mathbb{R}^D$ that is unbounded in at least one dimension $d$, i.e. if the polytope's projection onto $d$ is either left-bounded by $-\infty$ or right-bounded by $\infty$, or both.
% \begin{definition}[Partially Unbounded Polytope]
%     We call a convex subspace of $\mathbb{R}^D$ a partially-unbounded polytope $\mathcal{P}$ (PUP) if $\exists d \in 1,\ldots,D$ s.t.
    
%     \begin{equation*}
%          \pi_d(\mathcal{P}) = [a, \infty)  \vee\ \pi_d(\mathcal{P})  = (-\infty, b] \vee\ \pi_d(\mathcal{P})  = (-\infty, \infty)
%     \end{equation*}
    
%     with $a,b \in \mathbb{R}$ and $\pi_d(\mathcal{P})$ is the projection of the subspace onto a single dimension $d$. 
% \end{definition}

\section{Convergence of Predictions on OOD}

\begin{wrapfigure}{r}{0.25\textwidth}
    \centering
    \includegraphics[width=0.25\textwidth]{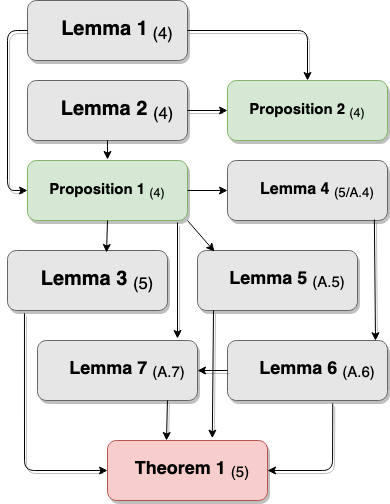}
    \caption{Dependencies between theoretical results. Information in parentheses denotes the section in the document.}
    \label{fig:flow-chart}
\end{wrapfigure}

In this section we will show that, moving the input to the extremes of the feature space, a ReLU classifier will converge to a fixed prediction. To demonstrate this, we must establish how the distance from the training data affects the network's logits. To this end, we utilize a known result stating that
neural networks employing piece-wise linear activation functions partition the input space into polytopes (such as in Figure \ref{subfig:polytopes}; \citeauthor{arora2018understanding}, \citeyear{arora2018understanding}). Given the saturating nature of the softmax, we conclude in Proposition \ref{proposition:overconfidence-softmax} that even for extreme feature values in the limit, the output distribution of the model will not change anymore. In order to help the reader untangle the interdependence of upcoming results, we provide a flow chart in Figure \ref{fig:flow-chart}.
%and an overgeneralization of uncertainty scores.
%In this section we will show that, when required to approximate a strictly monotonic function, a neural discriminator stabilizes in its prediction on OOD samples. 
%The first link in our argumentative chain shows that assuming a monotonic relationship between the input and output probabilities of our data implies a similar relationship for $p_{\btheta}$. 
%To this end, we first lay out a necessary assumption and establish some properties of the function to be learned.

We first describe how to re-write a ReLU network - or any other network with piece-wise linear activation functions - as a piece-wise affine transformation, borrowing from \citet{croce2018randomized, hein2019relu}. We start with the common form of $f_{\btheta}$ as a series of affine transformations, interleaved with ReLU activation functions, which we will denote by $\phi$:
\begin{equation}\begin{aligned}\label{eq:relu-net}
    f_{\btheta}(\bx) = &  \bW_L\phi\big(\bW_{L-1}\phi\big(\ldots \phi\big(\bW_1\bx + \bb_1\big) \ldots\big) \\ 
    & + \bb_{L-1} \big)+ \bb_L
\end{aligned}\end{equation}
In the following, let $f_{\btheta}^l(\bx)$ denote the output of  layer $l$ before applying an activation function. We now define a layer-specific diagonal matrix $\Phi_l \in \mathbb{R}^{n_l \times n_l}$ in the following way, where $n_l$ denotes the hidden units in layer $l$:
\begin{equation*}
    \bPhi_l(\bx) = \begin{bmatrix}
    \indicator{f_{\btheta}^l(\bx)_1 > 0} & \cdots & 0  \\
    \vdots &  \ddots & \vdots \\
    0 & \cdots & \indicator{f_{\btheta}^l(\bx)_{n_l} > 0} \\
    \end{bmatrix}
\end{equation*}
This allows us to rewrite Equation \ref{eq:relu-net} by replacing the usage of $\phi$ with a matrix multiplication using $\bPhi_l$:
\begin{equation}\begin{aligned}\label{eq:relu-net-replaced}
    f_{\btheta}(\bx) = & \bW_L\bPhi_{L-1}(\bx)\big(\bW_{L-1}\bPhi_{L-2}(\bx)\big(\ldots \\
     & \bPhi_1(\bx)\big(\bW_1\bx + \bb_1\big) \ldots\big) + \bb_{L-1} \big)+ \bb_L \\
\end{aligned}\end{equation}
Finally, by distributing the matrix products inside-out (see Appendix \ref{app:relu-linearization} for more detail), we can rewrite the network as a single affine transformation $f_{\btheta}(\bx) = \bV(\bx)\bx + \ba(\bx)$ with
\begin{equation*}
    \bV(\bx) = \bW_L\bigg(\prod_{l=1}^{L-1}\bPhi_l(\bx)\bW_{L-l} \bigg)
\end{equation*}
\begin{equation*}
    \ba(\bx) = \bb_L + \sum_{l=1}^{L-1}\bigg(\prod_{l^\prime=1}^{L-l}\bW_{L+1-l^\prime}\bPhi_{L-l^\prime}(\bx)\bigg)\bb_l
\end{equation*}
Note that the definition of $\bV(\bx)$ corresponds to the Jacobian of $f_{\btheta}(\bx)$, meaning that $v_{cd} = \frac{\partial f_{\btheta}(\bx)_c}{\partial x_{d}}$ This is very useful, as it allows to quickly check whether a network $f_{\btheta}$ is component-wise strictly monotonic by checking $\bV(\bx)$ for entries containing zeros. As \citet{hein2019relu} show, this formulation can also be used to characterize a set of polytopes $\mathcal{Q} = \{Q_1, \ldots, Q_M\}$ induced by $f_{\btheta}$ and that within each polytope, the function has a unique representation as an affine transformation. For this reason, we drop the dependence of $\bV$ and $\ba$ on $\bx$ when we refer to a specific polytope. Such polytopes are constructed by first retrieving the half-spaces induced by each of the network's neurons and then intersecting all said half-spaces to generate convex regions or polytopes.\footnote{We refer the reader to Appendix \ref{app:polytopes} or \citet{hein2019relu} for details on the construction, since it is not central to our reasoning.} We are especially interested in polytopes that are unbounded in at least one direction. In this regard, the results of \citet{croce2018randomized} and \citet{hein2019relu} show that there is a finite number of polytopes corresponding to the given network, and their Lemma 3.1 proves the existence of at least one unbounded polytope. Furthermore, under a mild condition on $\bV$, we can ascertain that $f_{\btheta}$ will be component-wise strictly monotonic on any polytope.

%It follows that for any such network we can identify finitely many PUPs.
% Such a representation exists for every neuron of every network layer, and the resulting affine transformation can be interpreted as a linear classifier that splits the feature space into half spaces according to its decision boundary. The feature space then becomes tessellated by polytopes created from the intersection of all the half spaces created by all the network's neurons. 

%\begin{lemma}\label{lemma:unbounded-polytopes}
%    For any ReLU network there exists a finite subset of PUPs $\mathcal{Q}^* \subseteq \mathcal{Q}$.
%\end{lemma}
%\begin{proof}
%The results of \citet{croce2018randomized} and \citet{hein2019relu} show that there is a finite number of polytopes corresponding to the given network, and their Lemma 3.1 proves the existence of at least one unbounded polytope.
%\end{proof}
%\begin{assumption}\label{assumption:convergence}
%   Suppose $p(y|\bx): \mathbb{R}^D \rightarrow [0, 1]^C$ is a component-wise strictly monotonic function on a partially unbounded $D$-interval $\mathcal{I}$ mapping a data sample onto a probability simplex. We assume in the following that a neural discriminator $p_{\btheta}$ approximating said function will also learn to be component-wise strictly monotonic on $\mathcal{I}$ up to some approximation error. 
%\end{assumption}

\begin{lemma}\label{lemma:strictly-monotonic}
    Suppose $f_{\btheta}$ is a ReLU network according to Definition \ref{def:logit}. Then $f_{\btheta}$ is a component-wise strictly monotonic function on every of its polytopes $Q\in\mathcal{Q}$, as long as its corresponding matrix $\bV$ has no zero entries. 
\end{lemma}
\begin{proof}
Let $Q$ be one such polytope. As discussed, when restricted to $Q$, the network corresponds to an affine transformation $f_{\btheta}(\bx) = \bV\bx + \ba$ with $\bV \in \mathbb{R}^{C \times D}$ and $\ba\in \mathbb{R}^{C}$. $f_{\btheta}(\bx)_c$ thus corresponds to the dot product of the $c$-th row of $\bV$ and $\bx$ plus the $c$-th element of $\ba$. It follows that the partial derivative of $f_{\btheta}(\bx)_c$ with respect to a dimension $d$ would then amount to the element $v_{cd}$ in $\bV$. This entails that, if $v_{cd} \neq 0$, at any point $\bx \in Q$ the gradient will be always positive or always negative. 
\end{proof}

We note here that the component-wise strict monotonicity of $f_{\btheta}$ and softmax do not entail the same property for $p_{\btheta}$.\footnote{To see a counterexample, the reader can check that even assuming component-wise strict monotonicity for $f_{\btheta}$, if the matrix $\bV$ associated to $f_{\btheta}$ on a specific polytope has a column $d$ filled with the same value $a$, then the resulting $p_{\btheta}$ will have a gradient of 0 at dimension $d$, regardless of what class we are considering. This is because the partial derivatives of the softmax, when all multiplied by the same constant $a$, add up to zero.}  Nonetheless, the monotonic behaviour of $f_{\btheta}$ is sufficient to drive the logits to plus or minus infinity in the limit, a phenomenon that constrains the output of $p_{\btheta}$ as we scale a data sample away from training data.

% Using this result, we can later show that although logit values of $f_{\btheta}(\bx)$ grow when the the input is scaled in the limit, this will not be reflected in the softmax probabilities due to the saturating properties of the function.\\

We begin our investigation of behaviour in the limit by showing that if we scale a vector only in a single dimension, we eventually always remain within a unique PUP.

\begin{lemma}\label{lemma:unique-pup}
     Let $\bx^\prime \in \mathbb{R}^D$ and $\mathcal{Q} = \{Q_1, \ldots, Q_M\}$ be the finite set of polytopes generated by a network $f_{\btheta}$. Let $\balpha \in \mathbb{R}^D$ be a vector s.t. $\forall d^\prime \neq d, \alpha_{d^\prime} = 1$. There exist a value $\beta >0$ and $m\in {1, \dots, M}$ such that for all $\alpha_d >\beta$, the product $\bx^\prime\odot\balpha$ lies within $Q_m$.
\end{lemma}

\begin{proof}
The proof mirrors the proof of Lemma 3.1 in \citet{hein2019relu}, so we only provide the intuition. By contradiction, suppose that there is no unique polytope and thus the point $\bx^\prime\odot\balpha$ must traverse  different polytopes as we scale up $\alpha_d$. Since there are finitely many polytopes, eventually the same polytope $Q_m$ will have to be traversed twice. Since the polytopes are convex, all the points on the line connecting the locations of where the boundary of $Q_m$ was crossed the first and second time must lie within $Q_m$, but this contradicts the fact that the scaled point traverses different polytopes.
\end{proof}

From here onward, we adapt the following shorthand to simplify notation: Given a scaling vector  $\balpha \in \mathbb{R}^D$ s.t. $\forall d^\prime \neq d, \alpha_{d^\prime} = 1$, we use $\mathcal{P}(\bx^\prime, d)$ to denote the PUP that $\bx^\prime$ lands in when scaling it with $\alpha_d$ in the limit. This definition implies that we can only scale parallel to the basis vectors (for a discussion on how restrictive this is, see Section \ref{sec:experiments}).   

Finally, in the next lemma we establish that the output distribution converges to a fixed point using the $l_2$-norm of the gradient $\nabla_{\bx} p_{\btheta}(y=c|\bx)$. Generally, in regions of the feature space where the classifier is predicting the same probability distribution over classes, small perturbations in the input $\bx$ will not change the prediction. Therefore, the gradient in these regions w.r.t. the input will be small and potentially even correspond to the zero vector, with a norm of (or close to) zero.

\begin{proposition}[Convergence of predictions in the limit]\label{proposition:overconfidence-softmax}
    Suppose that $f_{\btheta}$ is a ReLU-network. Let $\bx^\prime\in\mathbb{R}^D$, suppose $\balpha$ is a scaling vector and that the associated PUP $\mathcal{P}(\bx^\prime, d)$ has a corresponding matrix $\mathbf{V}$ with no zero entries.
    % Suppose that $p_{\btheta}$ is a component-wise strictly monotonic function on a PUP $\mathcal{P} \subseteq \mathbb{R}^D$ whose corresponding matrix $\bV$ has no zero entries. Suppose $\mathcal{P}$ is unbounded in the dimension $d$. Let $\bx^\prime \in \mathcal{P}$ be a data sample and $\balpha \in \mathbb{R}^D$ a vector s.t. $\balpha\odot\bx^\prime \in \mathcal{P}$. 
    Then it holds that
    
    \begin{equation*}
        \forall c \in \mathcal{C},\ \lim\limits_{\alpha_d \to \infty}\bbnorm\nabla_{\bx} p_{\btheta}(y=c|\bx)\bigg|_{\bx = \balpha\odot\bx^\prime}\bbnorm_2 = 0
    \end{equation*}
   % and 
    %\begin{equation*}
     %   \exists c \in \mathcal{C},\  \lim_{\alpha_d \to \infty} \bar{\sigma}(f_{\btheta}(\bx))_c = 1
    %\end{equation*}
\end{proposition}

The whole proof exceeds the available space and can be found in Appendix \ref{app:proposition1}, so we present the main intuitions here. Because of Lemma \ref{lemma:unique-pup}, we know the scaled point $\balpha\odot\bx^\prime$ will end in a unique PUP. The assumption on $f_{\btheta}$ then triggers Lemma \ref{lemma:strictly-monotonic}, from which we can infer that scaling the input in a single dimension leads all logits to $\pm\infty$. Because of the saturating property of the softmax,  this will in turn provoke the output of $p_{\btheta}$ to converge to a fixed point.
    
As an aside, we now recast Theorem 3.1 in \citet{hein2019relu} in our framework, showing that the model becomes increasingly certain in a single class, placing all its probability mass on it in the limit. The proof of this additional Proposition is in Appendix \ref{app:softmax-limit}.

\begin{proposition}\label{proposition:softmax-limit}
    Let $f_{\btheta}$ be ReLU network. Let $\bx^\prime \in \mathbb{R}^D$,  suppose $\balpha$ is a scaling vector  and that the associated PUP $\mathcal{P}(\bx^\prime, d)$ has a corresponding matrix $\mathbf{V}$ with no zero entries. Assume the $d$-th column of $\bV$ has no duplicate entries. 
    % Let $c \in \mathcal{C}$ be an arbitrary class s.t. $\forall c^\prime \neq c:\ v_{cd} > v_{c^\prime d}$.  
    Then there exists a class $c$ such that  
    
    \begin{equation*}
         \lim_{\alpha_d \rightarrow \infty} \bar{\sigma}(f_{\btheta}(\balpha\odot\bx^\prime))_c = 1
    \end{equation*}
\end{proposition}

In conclusion, we have shown in this section that the output probabilities of ReLU networks are less and less sensitive to small perturbations of the input in the limit and, under the assumptions of Proposition \ref{proposition:softmax-limit}, will converge to favor a single class, in which they appear to be very confident in. In the next section we prove that all other uncertainty metrics also converge to fixed values in the limit. 
        
\section{Convergence of Uncertainty Estimation Metrics}\label{sec:overconfidence-metrics}

In Proposition \ref{proposition:overconfidence-softmax}, we have established how the prediction of a model converges to a fixed point when feature values become extreme. We now want to show a similar property about the uncertainty estimation techniques introduced in Section \ref{sec:uncertainty-metrics}. To this end, we have to establish how the predictions coming from multiple model instances interact, a point we analyze in Lemma \ref{aggregation-theorem}. Then, we demonstrate how the uncertainty metrics also converge to a fixed value in the limit by proving the case for each of them in turn, before bundling our results in Theorem \ref{main-theorem}. We start with the easiest metric, which also applies to a single model.

\begin{lemma}{(Max. softmax probability)}\label{lemma:max-prob}
     Suppose that $f_{\btheta}$ is a ReLU-network. Let $\bx^\prime\in\mathbb{R}^D$,  suppose $\balpha$ is a scaling vector and that the associated PUP $\mathcal{P}(\bx^\prime, d)$ has a corresponding matrix $\mathbf{V}$ with no zero entries. It holds that
     
    \begin{equation*}
        \lim\limits_{\alpha_d \to \infty}\bbnorm\nabla_{\bx} \max_{c \in \mathcal{C}}\Big(p_{\btheta}(y=c|\bx)\Big)\Big|_{\bx = \balpha\odot\bx^\prime}\bbnorm_2 = 0
    \end{equation*}
\end{lemma}

\begin{proof}
    The gradient of the $\max$ function will be a specific $\nabla_{\bx}p_{\btheta}(y=c|\bx)$, which reduces this to the case already proven in Proposition \ref{proposition:overconfidence-softmax}.
\end{proof}

Note that for this metric, the combination with Proposition \ref{proposition:softmax-limit} shows that the model be fully confident in a single class in the limit. For our following lemmas, we want to consider uncertainty scores that are based on multiple instances, e.g. different ensemble members or forward passes using re-sampled dropout masks. What all of these approaches have in common is that for every $k$, the network parameters $\btheta^{(k)}$ will differ, and thus also the polytopal tesselation of the feature space. Hence, we have to adjust our assumptions accordingly. For every instance $k$, let us denote the affine function on a polytope $Q^{(k)}$ as $f_{\btheta}^{(k)}(\bx)=\bV^{(k)}\bx + \ba^{(k)}$. In order for our previous strategy to hold, we now assume $\forall k \in 1,\ldots,K$  that $\mathcal{P}^{(k)}(\bx^\prime, d)$ has a matrix $\bV^{(k)}$ which does not have any zero entries. Note that even though this assumption has to hold for all $k$, this does not mean that the matrices have to be identical.
% Firstly, we have proven in Lemma \ref{lemma:unique-pup} that scaling a point, it must land in a unique PUP. Let $\mathcal{P}^{(k)}$ be that unique PUP for every network instance $k$. Since $\balpha \odot \bx \in \mathcal{P}^{(k)}$, the scaled point will also belong to their intersection, which we are denoting $\mathcal{P}^{(K)}(\bx^\prime, d) = \bigcap_{k=1}^K \mathcal{P}^{(k)}(\bx^\prime, d)$. 
% Remember that in the single network case, we had to assume that the matrix $\bV$ corresponding to the polytope $\mathcal{P}$ does not contain zero entries. 

\begin{lemma}[Convergence of aggregated predictions in the limit]\label{aggregation-theorem}
    Suppose that $f_{\btheta}^{(1)}, \ldots, f_{\btheta}^{(K)}$ are ReLU networks. Let $\bx^\prime\in\mathbb{R}^D$, suppose $\balpha$ is a scaling vector  and that the for all $k$, the associated PUP $\mathcal{P}^{(k)}(\bx^\prime, d)$ has a corresponding matrix $\mathbf{V}^{(k)}$ with no zero entries. It holds that
    
    \begin{equation*}
        \lim\limits_{\alpha_d \to \infty}\bbnorm\nabla_{\bx}\ \Expect[\bigg]{{p(\btheta|\mathcal{D})}}{p_{\btheta}(y=c|\bx)}\bigg|_{\bx = \balpha\odot\bx^\prime}\bbnorm_2 = 0
    \end{equation*}

\end{lemma}

The full proof of this lemma can be found in appendix section \ref{app:aggregation-theorem}. The analogous lemmas for the remaining uncertainty metrics are stated and proved in appendix sections \ref{app:asymptotic-softmax-variance}, \ref{app:asymptotic-predictive-entropy} and  \ref{app:asymptotic-mutual-information}. The proof strategy for all further metrics is to simplify and reduce the uncertainty metrics such that Lemma \ref{aggregation-theorem} or Proposition \ref{proposition:overconfidence-softmax} can be applied. All of these result combined now pave the way for our central theorem.

\begin{tcolorbox}
\begin{theorem}[Convergence of uncertainty level in the limit]\label{main-theorem}
    Suppose that $f_{\btheta}^{(1)}, \ldots, f_{\btheta}^{(K)}$ are ReLU networks. Let $\bx^\prime\in\mathbb{R}^D$, suppose $\balpha$ is a scaling vector and that the for all $k$, the associated PUP $\mathcal{P}^{(k)}(\bx^\prime, d)$ has a corresponding matrix $\mathbf{V}^{(k)}$ with no zero entries. 
    %and that a neural network $p_{\btheta}$ is trained to discriminate between classes on data samples drawn from a distribution $p(\bx, y)$.
    Then, whenever uncertainty is measured via either of the following metrics
    \begin{enumerate}
        \item Maximum softmax probability
        \item Class variance
        \item Predictive entropy
        \item Approximate mutual information
    \end{enumerate}
    the network(s) will converge to fixed uncertainty scores for $\bx^\prime \odot \balpha$ in the limit of $\alpha_d \rightarrow \infty$
    %, i.e. for a vector $\balpha \in \mathbb{R}^D$ s.t. $\alpha_d > 0$, $\forall d^\prime \neq d, \alpha_{d^\prime} = 1$ and a data sample $\bx^\prime \sim p(\bx)$ s.t. $\balpha\odot\bx^\prime \in \mathcal{I}$.
    \end{theorem}

    \begin{proof}
    
    % Given lemmas \ref{lemma:strictly-monotonic} and \ref{lemma:both-monotonic-softmax}, $p_{\btheta}$ will be a component-wise strictly monotonic function on $\mathcal{P}(\bx^\prime, d)$.
    The four parts of the theorem are proven separately by lemmas \ref{lemma:max-prob}, \ref{lemma:asymptotic-softmax-variance} (\ref{app:asymptotic-softmax-variance}), \ref{lemma:asymptotic-predictive-entropy} (\ref{app:asymptotic-predictive-entropy}) and \ref{lemma:asymptotic-mutual-information} (\ref{app:asymptotic-mutual-information}).
    \end{proof}
\end{tcolorbox}

What follows from this result is that methods based on multiple instances of ReLU classifiers will suffer from the aforementioned problem as long as uncertainty is estimated with one of the techniques listed above. Next we demonstrate how these assumptions and results apply on synthetic data.

\section{Synthetic Data Experiments}\label{sec:experiments}

\begin{figure*}[ht!]
    \begin{subfigure}[t]{0.24\textwidth}
        \includegraphics[width=0.9\textwidth]{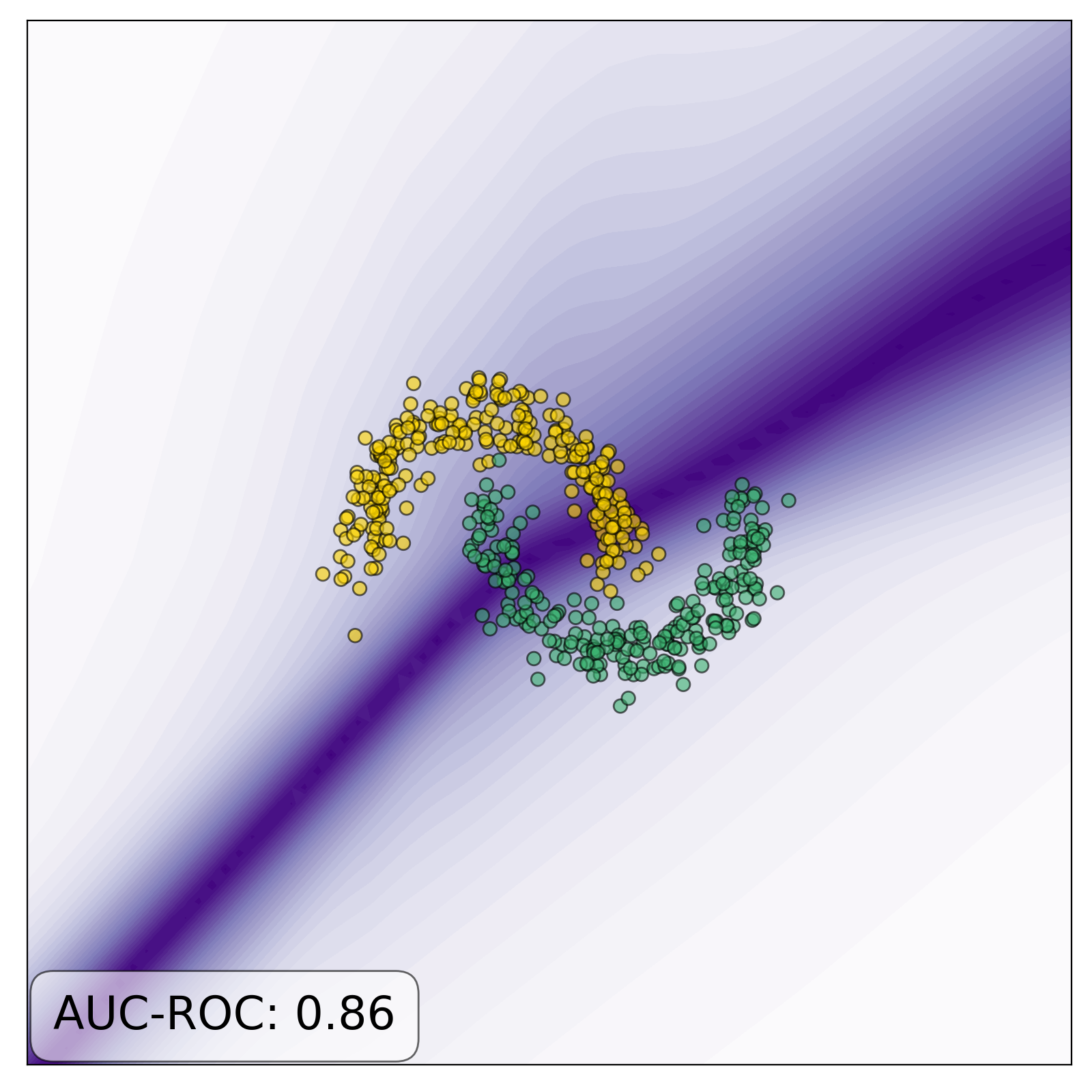}
        \includegraphics[width=\textwidth]{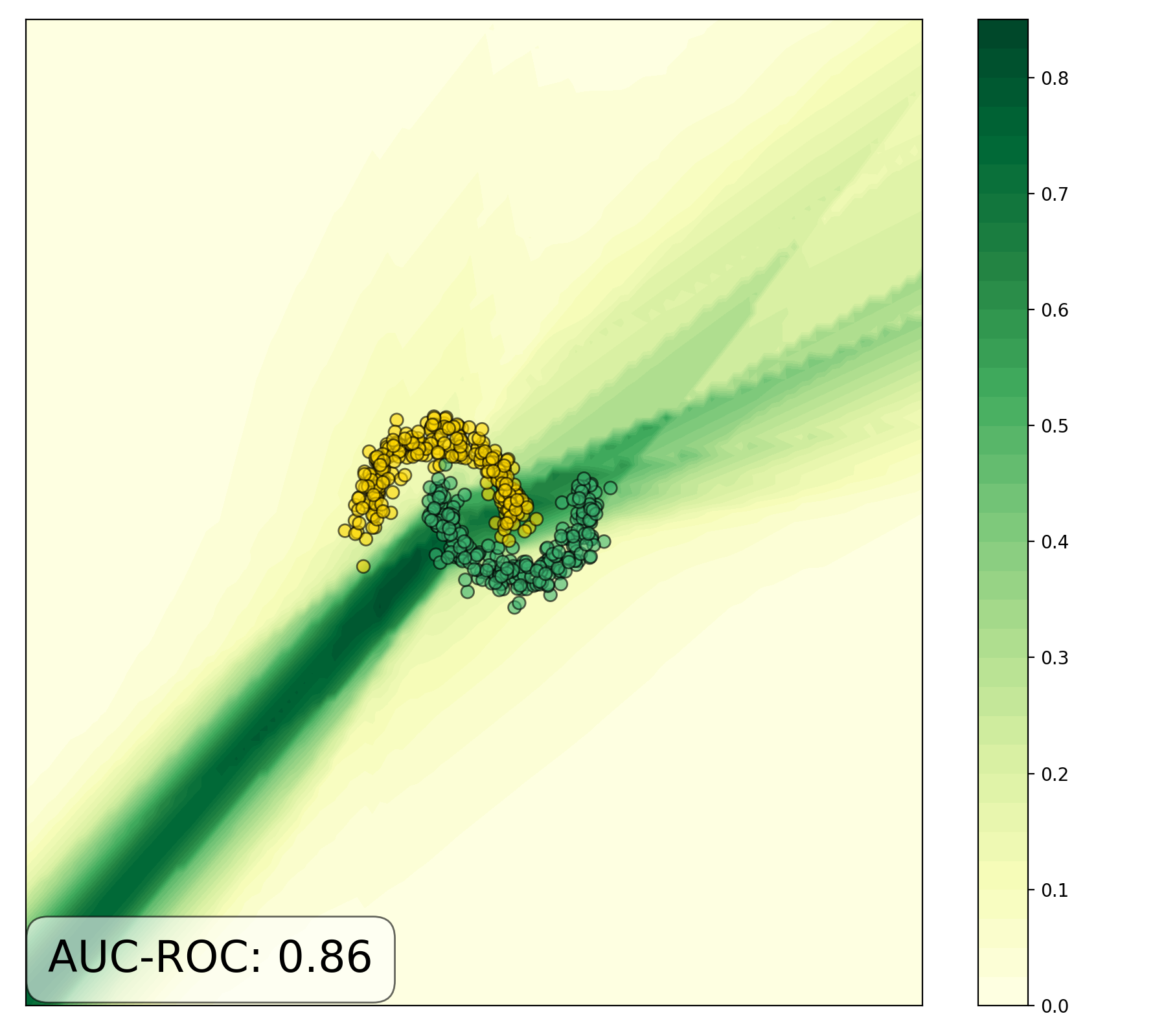}
        \caption{Neural discriminator with maximum probability \citep{hendrycks2017baseline}.}
        \label{subfig:nn-maxprob}
    \end{subfigure}
    \hfill
    \begin{subfigure}[t]{0.24\textwidth}
        \includegraphics[width=0.9\textwidth]{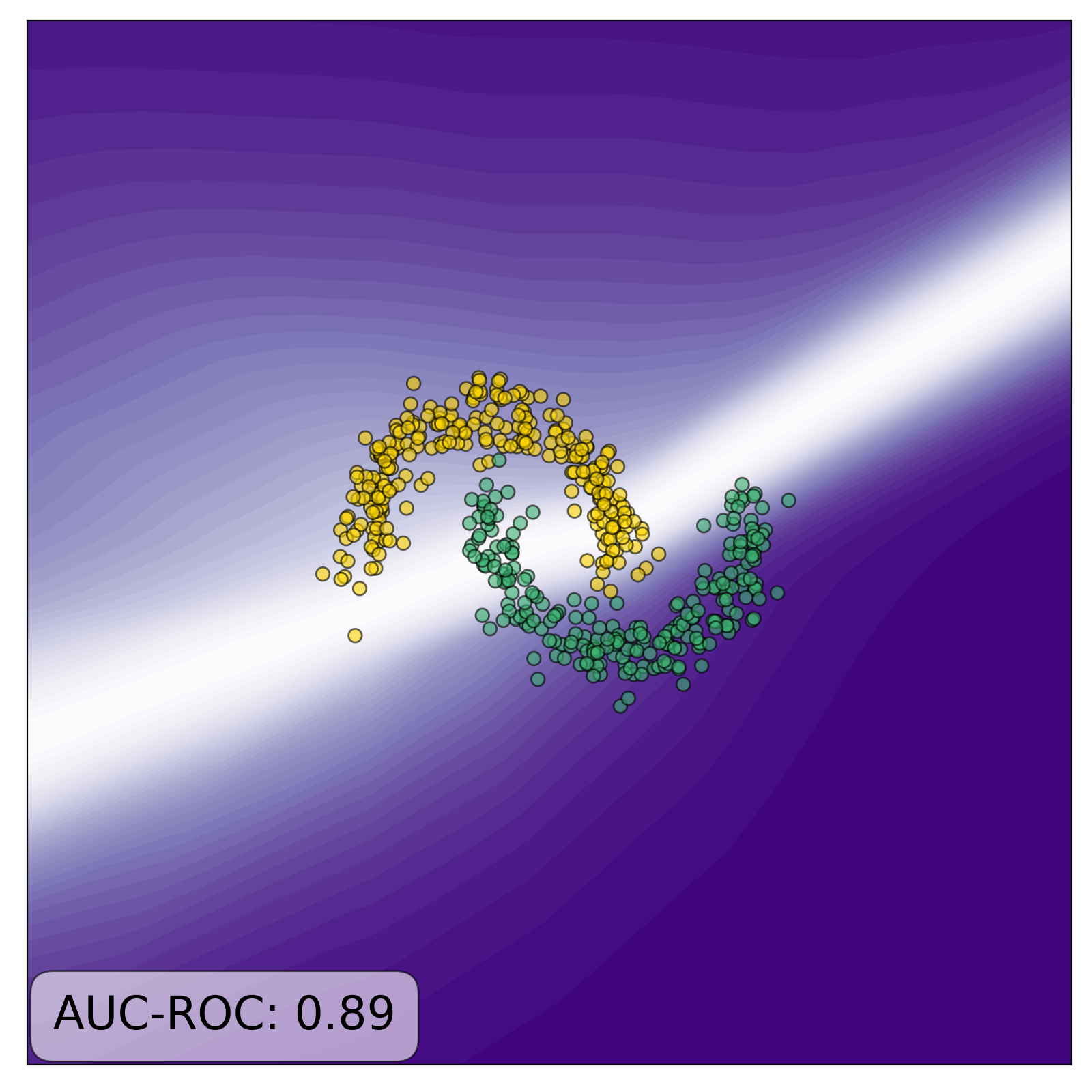}
        \includegraphics[width=\textwidth]{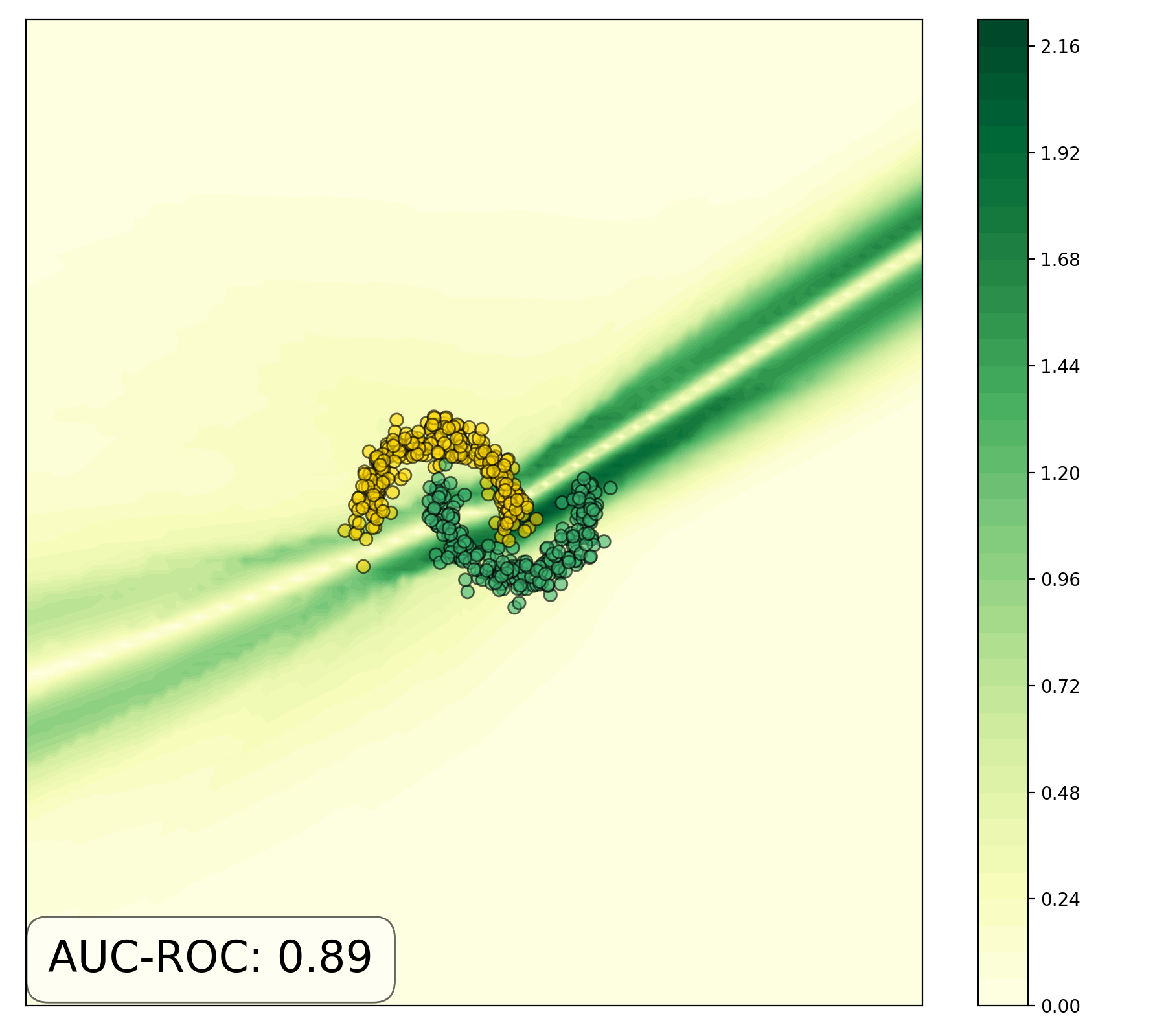}
        \caption{MC Dropout \citep{gal2016dropout} with mutual info. \citep{gal2018understanding}.}
        \label{subfig:mcdropout-mi}
    \end{subfigure}
    \hfill
    \begin{subfigure}[t]{0.24\textwidth}
        \includegraphics[width=0.9\textwidth]{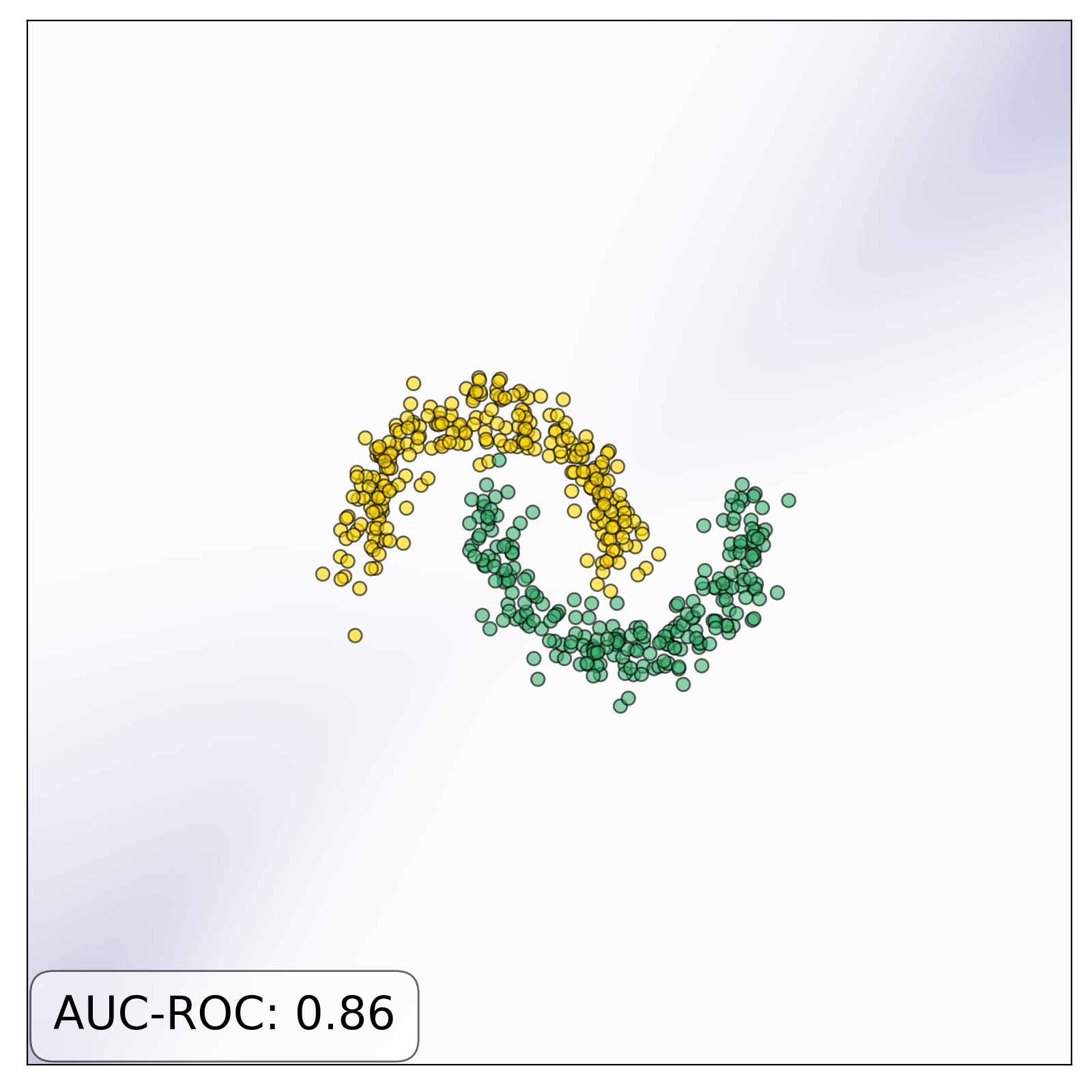}
        \includegraphics[width=\textwidth]{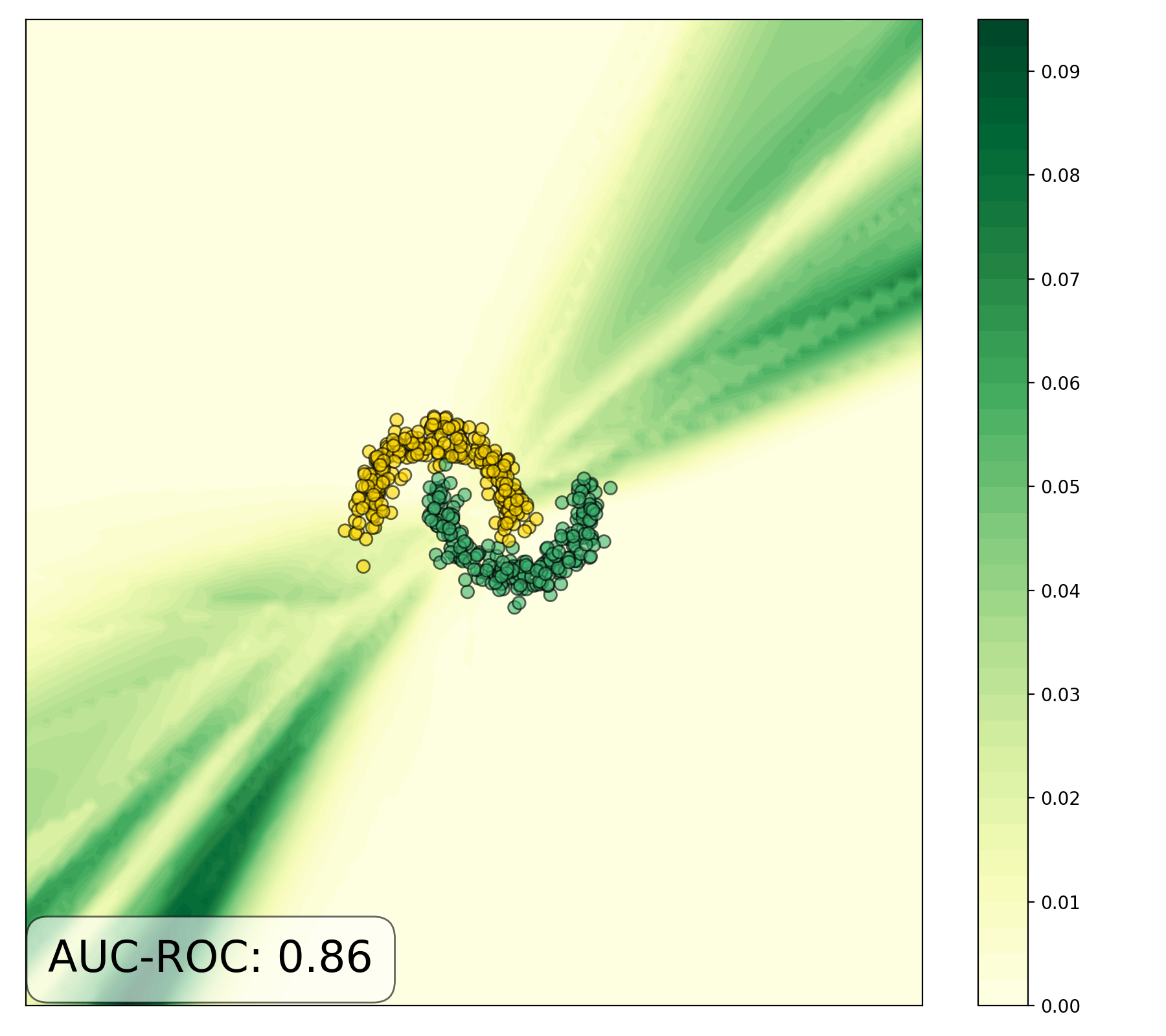}
        \caption{Neural ensemble \citep{lakshminarayanan2017simple} with class variance.}
        \label{subfig:ensemble-var}
    \end{subfigure}
    \hfill
    \begin{subfigure}[t]{0.24\textwidth}
        \includegraphics[width=0.9\textwidth]{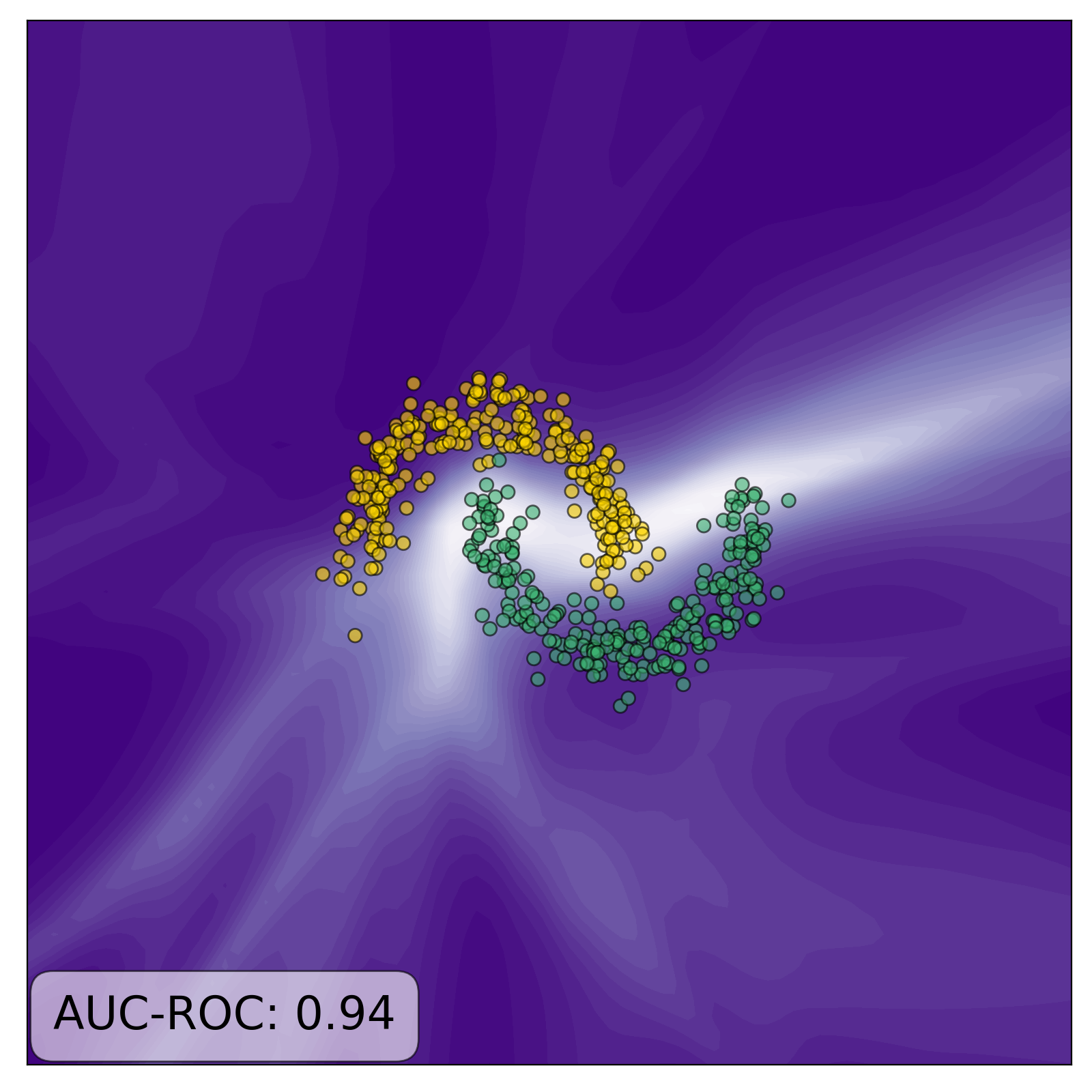}
        \includegraphics[width=\textwidth]{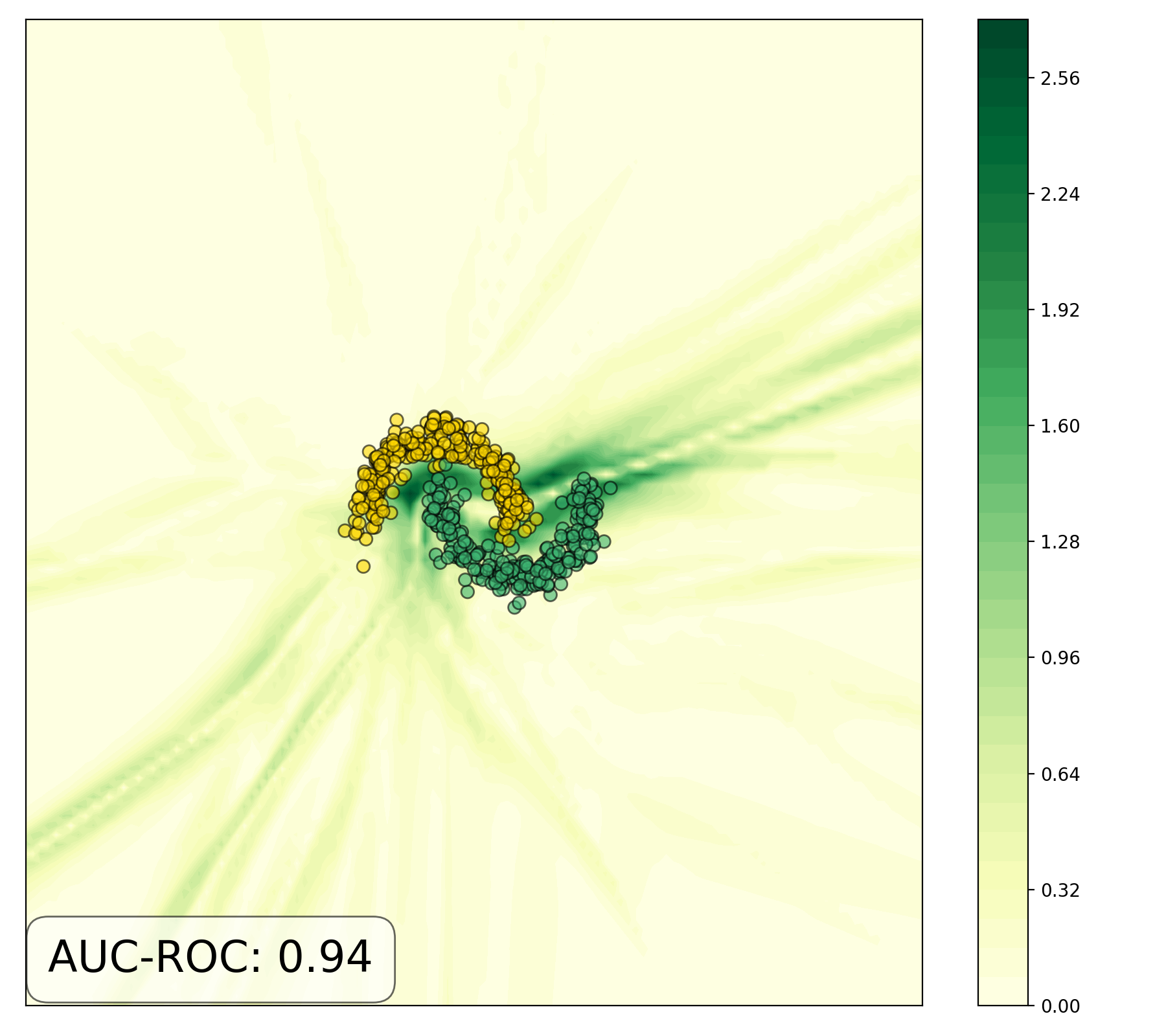}
        \caption{Anchored ensemble \citep{pearce2020uncertainty} with mutual information \citep{gal2018understanding}.}
        \label{subfig:anchoredensemble-mi}
    \end{subfigure}
    \caption{Uncertainty on the half-moon dataset, including the binary classification AUC-ROC. (Top row) The uncertainty surface is represented with increasingly darker shades of purple, with white being the lowest uncertainty. Open-ended regions of static certainty appear across different models and metrics, bein extrapolated to unseen data (see \ref{subfig:nn-maxprob}-\ref{subfig:ensemble-var}); this phenomenon is less apparent in some instances (\ref{subfig:anchoredensemble-mi}). (Bottom row) Increasing shades of green indicate the magnitude of the gradient of the uncertainty score w.r.t. the input. All metrics show open ended regions where the magnitude approaches zero.}
    \label{fig:entropy-plots}
\end{figure*}

To illustrate our findings, we plot the uncertainty surfaces and the gradient magnitudes of different models and uncertainty metric pairings on the half moons dataset, which we generate using the corresponding function in the \texttt{scikit-learn} package \citep{pedregosa2011scikit}. Detailed information about the procedure can be found in Appendix \ref{app:synthetic-data-experiments} along with additional plots.\footnote{The code used for the experiments is publicly available under \url{https://github.com/Kaleidophon/know-your-limits}.}

For a single network, we can observe in Figure \ref{subfig:nn-maxprob}) that there exist vast open-ended regions of stable confidence, confirming the findings of Theorem \ref{main-theorem}. However, in the bottom part of Figure \ref{subfig:nn-maxprob}) we can observe green regions with high gradient magnitude which do not seem to comply with our findings. In this case, we can see that these regions follow the decision boundaries. Due to the exponential function in the softmax, it is intuitive that small perturbation in these areas would have a large impact on the uncertainty score, resulting in a high gradient magnitude. But why does the magnitude not decrease in the limit as predicted by Theorem \ref{main-theorem}? We formulated our scaling vector $\balpha$ in way that only allows scaling along one of the coordinate axes. Therefore, if the decision boundaries are not parallel to the axes, by scaling we eventually escape the green areas and arrive at an area with gradient of magnitude zero. 
If the green regions were parallel to the axes then this would result in a violation of our main assumption. Traversing the input space parallel to a decision boundary in direction $d$ will not influence the prediction within the polytope, meaning that there will be entries  $v_{cd}=0$.\footnote{A decision boundary in a polytope is not the only way in which this assumption can be broken, but it still appears to hold reasonably often. For instance, just around $6.3 \%$ of plotted points in Figure \ref{fig:intro-figure} possess a matrix $\bV$ with at least one zero entry - all located in the PUP in the top right corner.}

Turning to predictions aggregated from multiple network instances in Figures \ref{subfig:mcdropout-mi}-\ref{subfig:anchoredensemble-mi}, we again observe large regions of constant uncertainty. The high-confidence region in the plots using mutual information displays a different behaviour from the others. As this metric aims to isolate epistemic uncertainty, it makes sense that uncertainty would be lowest around the training data, i.e. where the model is best specified. The character of the green regions in the bottom part of Figures \ref{subfig:ensemble-var} and \ref{subfig:anchoredensemble-mi} can again be explained by decision boundaries: In these cases, we have multiple instances with parameters $\btheta^{(k)}$, all with their own polytopal structure. When they overlap, the regions of the feature space where the assumption of our theorem is violated can either extend (Figure \ref{subfig:ensemble-var}) or shrink (Figure \ref{subfig:anchoredensemble-mi}), based on the diversity among instances. The fact that the anchored ensemble in Figure \ref{subfig:anchoredensemble-mi} does not exhibit such uniform regions of uncertainty like the vanilla ensemble could be explained by the fact that its training procedure encourages diversification between members. In turn, the difference between MC Dropout and ensemble models can be elucidated using recent insights that variational methods tend to only explore a single mode of the weight posterior $p(\btheta|\mathcal{D})$, while ensemble members often spread across multiple modes \citep{wilson2020bayesian}.

Overall, we have seen that our theorem can explain why an overgeneralization of uncertainty scores beyond the training data results in failure in OOD detection. We also explored the cases in which our assumptions are violated, i.e. by multiple, diverse model instances. In such scenarios, identification of OOD samples could in theory succeed, but often fails to do so reliably, see e.g. \citet{ovadia2019can, ulmer2020trust}. These insights can also help explain many other empirical findings in this regard on a variety of real-world datasets, e.g. \citet{gal2018understanding,kompa2020empirical}.

\section{Discussion}\label{sec:discussion}

In the past sections, we have proven that a single model will produce very confident softmax probabilities and that even for models using multiple network instances like ensembling and MC Dropout, all the uncertainty metrics analyzed will tend to fixed scores on far away samples, generalizing along open-ended polytopes induced by their architecture.
% To determine these fixed points is something that we leave to future research to explore.  
% Discussion of limit behaviour as opposed to local behaviour
The significance of this observation is that, albeit some modest success at OOD detection that might take place locally, the models we analyzed have an inherent overgeneralization bias: by extrapolating their level of uncertainty beyond the seen data, they hinder their ability to discern between in-distribution and OOD data.

For a single network, it can even be proven that the predictions attained on OOD samples is not just fixed but also unreasonably high in one class, as shown in our Proposition \ref{proposition:softmax-limit}, which adapts the main result in \citet{hein2019relu} to our framework. Our formulation shows that this result is also obtained when scaling a point along a single dimension.

% Discussion on assumption on interval I
Most of our results depend on an assumption on the matrix $\bV$ corresponding to the PUP containing the OOD sample. Like discussed in Section \ref{sec:experiments}, except for the case of decision boundaries that run parallel to a basis vector, this assumptions should rarely be broken for ReLU networks, as their are known to be resistant to the problem of vanishing gradients \citep{glorot2011deep}. We thus expect the behaviour described in our theoretical results to be very common, as confirmed by our experiments.

% This was proven in \cite{hein2019relu} under similar assumptions; however, they did not explore the implications for uncertainty metrics. 
It remains to be explored how the stable level of certainty that we derive for all metrics relates to the degree of diversity of the underlying set of model instances, but given our experimental results it appears to be dependent on the level of ``disagreement'' among them, i.e. when polytopes of different instances overlap to an increasing degree. %In this sense, OOD detection based on uncertainty can be traded off with performance (see e.g. discussion in \citet{ulmer2020trust}). 
Since it is hard to flag instances of OOD reliably this way, the aforementioned methods run a concrete risk of missing OOD samples, with potentially unintended, negative side-effects. 
% Results of this kind were already observed in the aforementioned \cite{hein2019relu},  which provides an alternative proof for a theorem akin to our Lemma \ref{proposition:overconfidence-softmax}, but does not explore the implications for uncertainty metrics. Interestingly, the assumptions needed for our findings are not required in their proof, while the conditions they impose on the networks are not appearing in our set of assumptions.
These investigations bring us closer to a theoretically-motivated explanation to observations such as reported in \citet{gal2018understanding,ovadia2019can,kompa2020empirical,ulmer2020trust} and to enable the discovery of more effective methods.

Future research might be divided in two categories: efforts to solve the problem of OOD detection, and attempts at sharpening our theoretical understanding of the issue. 
The following approaches fall into the first category. One way consists of complementing neural discriminators with density-based approaches such as in \citet{grathwohl2020classifier}. Other lines of research try to have neural discriminators parametrize Dirichlet instead of categorical distributions \citep{malinin2018predictive, joo2020being, charpentier2020posterior}
%Another very recent approach tries to mitigate the problem discussed by 
or making models distance-aware \citep{liu2020simple} or supplementing the network architecture with Bayesian capabilities \citep{kristiadi2020being}.
In the second category we mention 
% Discussion on categorical features
that more work is needed to cover the case of discrete features. Reasoning in the limit is not available for categorical or ordinal variables and one would probably have to resort to techniques that address directly the difference between the training and the new distribution. Furthermore, it remains an open question whether the results presented here can be extended be extended to GELU activations \citep{hendrycks2016gaussian}, which represent a continuous approximation of the ReLU function with the same asymptotic behavior. The GELU has recently become a very popular alternative in deep neural networks \citep{devlin2019bert, lan2020albert, brown2020language}. Lastly, the similarity between lemmas used for Theorem \ref{main-theorem} suggest that similar results could be derived for a whole family of uncertainty metrics. These investigations are left to future work.
%together with reflections on how much the monotonicity assumptions can be relaxed, are left to future work. \extodo{Re-add references}

\section*{Author Contributions}

Both authors contributed equally in writing the paper and deriving theoretical results. Dennis Ulmer implemented the code to run experiments and create plots and illustrations in this work. 

\section*{Acknowledgements}

We would like to thank Mareike Hartmann, Adam Izdebski, Natalie Schluter, Emese Tham\'{o} and Christina Winkler for their tremendously helpful feedback on this work.

%% The file named.bst is a bibliography style file for BibTeX 0.99c

\bibliography{uai2021-template}

\begin{thebibliography}{48}
\providecommand{\natexlab}[1]{#1}
\providecommand{\url}[1]{\texttt{#1}}
\expandafter\ifx\csname urlstyle\endcsname\relax
  \providecommand{\doi}[1]{doi: #1}\else
  \providecommand{\doi}{doi: \begingroup \urlstyle{rm}\Url}\fi

\bibitem[Arora et~al.(2018)Arora, Basu, Mianjy, and
  Mukherjee]{arora2018understanding}
Raman Arora, Amitabh Basu, Poorya Mianjy, and Anirbit Mukherjee.
\newblock Understanding deep neural networks with rectified linear units.
\newblock In \emph{6th International Conference on Learning Representations,
  {ICLR} 2018, Vancouver, BC, Canada, April 30 - May 3, 2018, Conference Track
  Proceedings}. OpenReview.net, 2018.

\bibitem[Blundell et~al.(2015)Blundell, Cornebise, Kavukcuoglu, and
  Wierstra]{blundell2015weight}
Charles Blundell, Julien Cornebise, Koray Kavukcuoglu, and Daan Wierstra.
\newblock Weight uncertainty in neural networks.
\newblock \emph{arXiv preprint arXiv:1505.05424}, 2015.

\bibitem[Bridle(1990)]{bridle1990probabilistic}
John~S Bridle.
\newblock Probabilistic interpretation of feedforward classification network
  outputs, with relationships to statistical pattern recognition.
\newblock In \emph{Neurocomputing}, pages 227--236. Springer, 1990.

\bibitem[Brown et~al.(2020)Brown, Mann, Ryder, Subbiah, Kaplan, Dhariwal,
  Neelakantan, Shyam, Sastry, Askell, Agarwal, Herbert{-}Voss, Krueger,
  Henighan, Child, Ramesh, Ziegler, Wu, Winter, Hesse, Chen, Sigler, Litwin,
  Gray, Chess, Clark, Berner, McCandlish, Radford, Sutskever, and
  Amodei]{brown2020language}
Tom~B. Brown, Benjamin Mann, Nick Ryder, Melanie Subbiah, Jared Kaplan,
  Prafulla Dhariwal, Arvind Neelakantan, Pranav Shyam, Girish Sastry, Amanda
  Askell, Sandhini Agarwal, Ariel Herbert{-}Voss, Gretchen Krueger, Tom
  Henighan, Rewon Child, Aditya Ramesh, Daniel~M. Ziegler, Jeffrey Wu, Clemens
  Winter, Christopher Hesse, Mark Chen, Eric Sigler, Mateusz Litwin, Scott
  Gray, Benjamin Chess, Jack Clark, Christopher Berner, Sam McCandlish, Alec
  Radford, Ilya Sutskever, and Dario Amodei.
\newblock Language models are few-shot learners.
\newblock In Hugo Larochelle, Marc'Aurelio Ranzato, Raia Hadsell,
  Maria{-}Florina Balcan, and Hsuan{-}Tien Lin, editors, \emph{Advances in
  Neural Information Processing Systems 33: Annual Conference on Neural
  Information Processing Systems 2020, NeurIPS 2020, December 6-12, 2020,
  virtual}, 2020.
\newblock URL
  \url{https://proceedings.neurips.cc/paper/2020/hash/1457c0d6bfcb4967418bfb8ac142f64a-Abstract.html}.

\bibitem[Charpentier et~al.(2020)Charpentier, Z{\"{u}}gner, and
  G{\"{u}}nnemann]{charpentier2020posterior}
Bertrand Charpentier, Daniel Z{\"{u}}gner, and Stephan G{\"{u}}nnemann.
\newblock Posterior network: Uncertainty estimation without {OOD} samples via
  density-based pseudo-counts.
\newblock In Hugo Larochelle, Marc'Aurelio Ranzato, Raia Hadsell,
  Maria{-}Florina Balcan, and Hsuan{-}Tien Lin, editors, \emph{Advances in
  Neural Information Processing Systems 33: Annual Conference on Neural
  Information Processing Systems 2020, NeurIPS 2020, December 6-12, 2020,
  virtual}, 2020.

\bibitem[Croce and Hein(2018)]{croce2018randomized}
Francesco Croce and Matthias Hein.
\newblock A randomized gradient-free attack on relu networks.
\newblock In \emph{German Conference on Pattern Recognition}, pages 215--227.
  Springer, 2018.

\bibitem[Curth et~al.(2019)Curth, Thoral, van~den Wildenberg, Bijlstra,
  de~Bruin, Elbers, and Fornasa]{curth2019transferring}
Alicia Curth, Patrick Thoral, Wilco van~den Wildenberg, Peter Bijlstra, Daan
  de~Bruin, Paul Elbers, and Mattia Fornasa.
\newblock Transferring clinical prediction models across hospitals and
  electronic health record systems.
\newblock In \emph{Joint European Conference on Machine Learning and Knowledge
  Discovery in Databases}, pages 605--621. Springer, 2019.

\bibitem[de~Br{\'{e}}bisson and Vincent(2016)]{de2015exploration}
Alexandre de~Br{\'{e}}bisson and Pascal Vincent.
\newblock An exploration of softmax alternatives belonging to the spherical
  loss family.
\newblock In Yoshua Bengio and Yann LeCun, editors, \emph{4th International
  Conference on Learning Representations, {ICLR} 2016, San Juan, Puerto Rico,
  May 2-4, 2016, Conference Track Proceedings}, 2016.

\bibitem[Der~Kiureghian and Ditlevsen(2009)]{der2009aleatory}
Armen Der~Kiureghian and Ove Ditlevsen.
\newblock Aleatory or epistemic? does it matter?
\newblock \emph{Structural safety}, 31\penalty0 (2):\penalty0 105--112, 2009.

\bibitem[Devlin et~al.(2019)Devlin, Chang, Lee, and Toutanova]{devlin2019bert}
Jacob Devlin, Ming{-}Wei Chang, Kenton Lee, and Kristina Toutanova.
\newblock {BERT:} pre-training of deep bidirectional transformers for language
  understanding.
\newblock In Jill Burstein, Christy Doran, and Thamar Solorio, editors,
  \emph{Proceedings of the 2019 Conference of the North American Chapter of the
  Association for Computational Linguistics: Human Language Technologies,
  {NAACL-HLT} 2019, Minneapolis, MN, USA, June 2-7, 2019, Volume 1 (Long and
  Short Papers)}, pages 4171--4186. Association for Computational Linguistics,
  2019.
\newblock \doi{10.18653/v1/n19-1423}.
\newblock URL \url{https://doi.org/10.18653/v1/n19-1423}.

\bibitem[Gal(2016)]{gal2016uncertainty}
Yarin Gal.
\newblock Uncertainty in deep learning.
\newblock \emph{University of Cambridge}, 1\penalty0 (3), 2016.

\bibitem[Gal and Ghahramani(2016)]{gal2016dropout}
Yarin Gal and Zoubin Ghahramani.
\newblock Dropout as a bayesian approximation: Representing model uncertainty
  in deep learning.
\newblock In \emph{International conference on Machine Learning}, pages
  1050--1059, 2016.

\bibitem[Gao and Pavel(2017)]{gao2017properties}
Bolin Gao and Lacra Pavel.
\newblock On the properties of the softmax function with application in game
  theory and reinforcement learning.
\newblock \emph{arXiv preprint arXiv:1704.00805}, 2017.

\bibitem[Glorot et~al.(2011)Glorot, Bordes, and Bengio]{glorot2011deep}
Xavier Glorot, Antoine Bordes, and Yoshua Bengio.
\newblock Deep sparse rectifier neural networks.
\newblock In \emph{Proceedings of the fourteenth international conference on
  artificial intelligence and statistics}, pages 315--323. JMLR Workshop and
  Conference Proceedings, 2011.

\bibitem[Grathwohl et~al.(2020)Grathwohl, Wang, Jacobsen, Duvenaud, Norouzi,
  and Swersky]{grathwohl2020classifier}
Will Grathwohl, Kuan{-}Chieh Wang, J{\"{o}}rn{-}Henrik Jacobsen, David
  Duvenaud, Mohammad Norouzi, and Kevin Swersky.
\newblock Your classifier is secretly an energy based model and you should
  treat it like one.
\newblock In \emph{8th International Conference on Learning Representations,
  {ICLR} 2020, Addis Ababa, Ethiopia, April 26-30, 2020}, 2020.

\bibitem[Guo et~al.(2017)Guo, Pleiss, Sun, and Weinberger]{guo2017calibration}
Chuan Guo, Geoff Pleiss, Yu~Sun, and Kilian~Q. Weinberger.
\newblock On calibration of modern neural networks.
\newblock In \emph{Proceedings of the 34th International Conference on Machine
  Learning, {ICML} 2017, Sydney, NSW, Australia, 6-11 August 2017}, pages
  1321--1330, 2017.

\bibitem[He et~al.(2019)He, Baxter, Xu, Xu, Zhou, and Zhang]{he2019practical}
Jianxing He, Sally~L Baxter, Jie Xu, Jiming Xu, Xingtao Zhou, and Kang Zhang.
\newblock The practical implementation of artificial intelligence technologies
  in medicine.
\newblock \emph{Nature medicine}, 25\penalty0 (1):\penalty0 30--36, 2019.

\bibitem[Hein et~al.(2019)Hein, Andriushchenko, and Bitterwolf]{hein2019relu}
Matthias Hein, Maksym Andriushchenko, and Julian Bitterwolf.
\newblock Why relu networks yield high-confidence predictions far away from the
  training data and how to mitigate the problem.
\newblock In \emph{Proceedings of the IEEE Conference on Computer Vision and
  Pattern Recognition}, pages 41--50, 2019.

\bibitem[Hendrycks and Gimpel(2016)]{hendrycks2016gaussian}
Dan Hendrycks and Kevin Gimpel.
\newblock Gaussian error linear units (gelus).
\newblock \emph{arXiv preprint arXiv:1606.08415}, 2016.

\bibitem[Hendrycks and Gimpel(2017)]{hendrycks2017baseline}
Dan Hendrycks and Kevin Gimpel.
\newblock A baseline for detecting misclassified and out-of-distribution
  examples in neural networks.
\newblock 2017.

\bibitem[H{\"u}llermeier and Waegeman(2019)]{hullermeier2019aleatoric}
Eyke H{\"u}llermeier and Willem Waegeman.
\newblock Aleatoric and epistemic uncertainty in machine learning: A tutorial
  introduction.
\newblock \emph{arXiv preprint arXiv:1910.09457}, 2019.

\bibitem[Joo et~al.(2020)Joo, Chung, and Seo]{joo2020being}
Taejong Joo, Uijung Chung, and Min{-}Gwan Seo.
\newblock Being bayesian about categorical probability.
\newblock In \emph{Proceedings of the 37th International Conference on Machine
  Learning, {ICML} 2020, 13-18 July 2020, Virtual Event}, volume 119 of
  \emph{Proceedings of Machine Learning Research}, pages 4950--4961. {PMLR},
  2020.

\bibitem[Jordan et~al.(2019)Jordan, Lewis, and Dimakis]{jordan2019provable}
Matt Jordan, Justin Lewis, and Alexandros~G. Dimakis.
\newblock Provable certificates for adversarial examples: Fitting a ball in the
  union of polytopes.
\newblock In Hanna~M. Wallach, Hugo Larochelle, Alina Beygelzimer, Florence
  d'Alch{\'{e}}{-}Buc, Emily~B. Fox, and Roman Garnett, editors, \emph{Advances
  in Neural Information Processing Systems 32: Annual Conference on Neural
  Information Processing Systems 2019, NeurIPS 2019, December 8-14, 2019,
  Vancouver, BC, Canada}, pages 14059--14069, 2019.

\bibitem[Kompa et~al.(2020)Kompa, Snoek, and Beam]{kompa2020empirical}
Benjamin Kompa, Jasper Snoek, and Andrew Beam.
\newblock Empirical frequentist coverage of deep learning uncertainty
  quantification procedures.
\newblock \emph{arXiv preprint arXiv:2010.03039}, 2020.

\bibitem[Kristiadi et~al.(2020)Kristiadi, Hein, and Hennig]{kristiadi2020being}
Agustinus Kristiadi, Matthias Hein, and Philipp Hennig.
\newblock Being bayesian, even just a bit, fixes overconfidence in relu
  networks.
\newblock In \emph{Proceedings of the 37th International Conference on Machine
  Learning, {ICML} 2020, 13-18 July 2020, Virtual Event}, volume 119 of
  \emph{Proceedings of Machine Learning Research}, pages 5436--5446. {PMLR},
  2020.
\newblock URL \url{http://proceedings.mlr.press/v119/kristiadi20a.html}.

\bibitem[Kull et~al.(2019)Kull, Perell{\'{o}}{-}Nieto, K{\"{a}}ngsepp,
  de~Menezes~e Silva~Filho, Song, and Flach]{kull2019beyond}
Meelis Kull, Miquel Perell{\'{o}}{-}Nieto, Markus K{\"{a}}ngsepp, Telmo
  de~Menezes~e Silva~Filho, Hao Song, and Peter~A. Flach.
\newblock Beyond temperature scaling: Obtaining well-calibrated multi-class
  probabilities with dirichlet calibration.
\newblock In Hanna~M. Wallach, Hugo Larochelle, Alina Beygelzimer, Florence
  d'Alch{\'{e}}{-}Buc, Emily~B. Fox, and Roman Garnett, editors, \emph{Advances
  in Neural Information Processing Systems 32: Annual Conference on Neural
  Information Processing Systems 2019, NeurIPS 2019, December 8-14, 2019,
  Vancouver, BC, Canada}, pages 12295--12305, 2019.
\newblock URL
  \url{https://proceedings.neurips.cc/paper/2019/hash/8ca01ea920679a0fe3728441494041b9-Abstract.html}.

\bibitem[Kumar et~al.(2018)Kumar, Sarawagi, and Jain]{kumar2018trainable}
Aviral Kumar, Sunita Sarawagi, and Ujjwal Jain.
\newblock Trainable calibration measures for neural networks from kernel mean
  embeddings.
\newblock In \emph{Proceedings of the 35th International Conference on Machine
  Learning, {ICML} 2018, Stockholmsm{\"{a}}ssan, Stockholm, Sweden, July 10-15,
  2018}, pages 2810--2819, 2018.
\newblock URL \url{http://proceedings.mlr.press/v80/kumar18a.html}.

\bibitem[Laha et~al.(2018)Laha, Chemmengath, Agrawal, Khapra, Sankaranarayanan,
  and Ramaswamy]{laha2018controllable}
Anirban Laha, Saneem~Ahmed Chemmengath, Priyanka Agrawal, Mitesh Khapra,
  Karthik Sankaranarayanan, and Harish~G Ramaswamy.
\newblock On controllable sparse alternatives to softmax.
\newblock In \emph{Advances in Neural Information Processing Systems}, pages
  6422--6432, 2018.

\bibitem[Lakshminarayanan et~al.(2017)Lakshminarayanan, Pritzel, and
  Blundell]{lakshminarayanan2017simple}
Balaji Lakshminarayanan, Alexander Pritzel, and Charles Blundell.
\newblock Simple and scalable predictive uncertainty estimation using deep
  ensembles.
\newblock In \emph{Advances in neural information processing systems}, pages
  6402--6413, 2017.

\bibitem[Lan et~al.(2020)Lan, Chen, Goodman, Gimpel, Sharma, and
  Soricut]{lan2020albert}
Zhenzhong Lan, Mingda Chen, Sebastian Goodman, Kevin Gimpel, Piyush Sharma, and
  Radu Soricut.
\newblock {ALBERT:} {A} lite {BERT} for self-supervised learning of language
  representations.
\newblock In \emph{8th International Conference on Learning Representations,
  {ICLR} 2020, Addis Ababa, Ethiopia, April 26-30, 2020}. OpenReview.net, 2020.
\newblock URL \url{https://openreview.net/forum?id=H1eA7AEtvS}.

\bibitem[Laves et~al.(2019)Laves, Ihler, Kortmann, and Ortmaier]{laves2019well}
Max{-}Heinrich Laves, Sontje Ihler, Karl{-}Philipp Kortmann, and Tobias
  Ortmaier.
\newblock Well-calibrated model uncertainty with temperature scaling for
  dropout variational inference.
\newblock \emph{4th workshop on Bayesian Deep Learning (NeurIPS 2019),
  Vancouver, Canada}, 2019.

\bibitem[Liu et~al.(2020)Liu, Lin, Padhy, Tran, Bedrax~Weiss, and
  Lakshminarayanan]{liu2020simple}
Jeremiah Liu, Zi~Lin, Shreyas Padhy, Dustin Tran, Tania Bedrax~Weiss, and
  Balaji Lakshminarayanan.
\newblock Simple and principled uncertainty estimation with deterministic deep
  learning via distance awareness.
\newblock \emph{Advances in Neural Information Processing Systems}, 33, 2020.

\bibitem[Malinin and Gales(2018)]{malinin2018predictive}
Andrey Malinin and Mark Gales.
\newblock Predictive uncertainty estimation via prior networks.
\newblock In \emph{Advances in Neural Information Processing Systems}, pages
  7047--7058, 2018.

\bibitem[Martins and Astudillo(2016)]{martins2016softmax}
Andre Martins and Ramon Astudillo.
\newblock From softmax to sparsemax: A sparse model of attention and
  multi-label classification.
\newblock In \emph{International Conference on Machine Learning}, pages
  1614--1623, 2016.

\bibitem[McDermid et~al.(2019)McDermid, Jia, and Habli]{mcdermid2019towards}
John~Alexander McDermid, Yan Jia, and Ibrahim Habli.
\newblock Towards a framework for safety assurance of autonomous systems.
\newblock In \emph{Artificial Intelligence Safety 2019}, pages 1--7. CEUR
  Workshop Proceedings, 2019.

\bibitem[Moreno-Torres et~al.(2012)Moreno-Torres, Raeder, Alaiz-Rodr{\'\i}Guez,
  Chawla, and Herrera]{moreno2012unifying}
Jose~G Moreno-Torres, Troy Raeder, Roc{\'\i}O Alaiz-Rodr{\'\i}Guez, Nitesh~V
  Chawla, and Francisco Herrera.
\newblock A unifying view on dataset shift in classification.
\newblock \emph{Pattern recognition}, 45\penalty0 (1):\penalty0 521--530, 2012.

\bibitem[Mukhoti et~al.(2021)Mukhoti, Kirsch, van Amersfoort, Torr, and
  Gal]{mukhoti2021deterministic}
Jishnu Mukhoti, Andreas Kirsch, Joost van Amersfoort, Philip~HS Torr, and Yarin
  Gal.
\newblock Deterministic {N}eural {N}etworks with {A}ppropriate {I}nductive
  {B}iases {C}apture {E}pistemic and {A}leatoric {U}ncertainty.
\newblock \emph{arXiv preprint arXiv:2102.11582}, 2021.

\bibitem[Nalisnick et~al.(2019)Nalisnick, Matsukawa, Teh, G{\"{o}}r{\"{u}}r,
  and Lakshminarayanan]{nalisnick2018deep}
Eric~T. Nalisnick, Akihiro Matsukawa, Yee~Whye Teh, Dilan G{\"{o}}r{\"{u}}r,
  and Balaji Lakshminarayanan.
\newblock Do deep generative models know what they don't know?
\newblock In \emph{7th International Conference on Learning Representations,
  {ICLR} 2019, New Orleans, LA, USA, May 6-9, 2019}. OpenReview.net, 2019.

\bibitem[Neklyudov et~al.(2020)Neklyudov, Welling, Egorov, and
  Vetrov]{involutive2020neklyudov}
Kirill Neklyudov, Max Welling, Evgenii Egorov, and Dmitry~P. Vetrov.
\newblock Involutive {MCMC:} a unifying framework.
\newblock In \emph{Proceedings of the 37th International Conference on Machine
  Learning, {ICML} 2020, 13-18 July 2020, Virtual Event}, pages 7273--7282,
  2020.

\bibitem[Ovadia et~al.(2019)Ovadia, Fertig, Ren, Nado, Sculley, Nowozin,
  Dillon, Lakshminarayanan, and Snoek]{ovadia2019can}
Yaniv Ovadia, Emily Fertig, Jie Ren, Zachary Nado, David Sculley, Sebastian
  Nowozin, Joshua Dillon, Balaji Lakshminarayanan, and Jasper Snoek.
\newblock Can you trust your model's uncertainty? evaluating predictive
  uncertainty under dataset shift.
\newblock In \emph{Advances in Neural Information Processing Systems}, pages
  13991--14002, 2019.

\bibitem[Pearce et~al.(2020)Pearce, Leibfried, and
  Brintrup]{pearce2020uncertainty}
Tim Pearce, Felix Leibfried, and Alexandra Brintrup.
\newblock Uncertainty in neural networks: Approximately bayesian ensembling.
\newblock In \emph{International Conference on Artificial Intelligence and
  Statistics}, pages 234--244, 2020.

\bibitem[Pedregosa et~al.(2011)Pedregosa, Varoquaux, Gramfort, Michel, Thirion,
  Grisel, Blondel, Prettenhofer, Weiss, Dubourg, et~al.]{pedregosa2011scikit}
Fabian Pedregosa, Ga{\"e}l Varoquaux, Alexandre Gramfort, Vincent Michel,
  Bertrand Thirion, Olivier Grisel, Mathieu Blondel, Peter Prettenhofer, Ron
  Weiss, Vincent Dubourg, et~al.
\newblock Scikit-learn: Machine learning in python.
\newblock \emph{the Journal of machine Learning research}, 12:\penalty0
  2825--2830, 2011.

\bibitem[Ramachandran et~al.(2018)Ramachandran, Zoph, and
  Le]{ramachandran2017searching}
Prajit Ramachandran, Barret Zoph, and Quoc~V. Le.
\newblock Searching for activation functions.
\newblock In \emph{6th International Conference on Learning Representations,
  {ICLR} 2018, Vancouver, BC, Canada, April 30 - May 3, 2018, Workshop Track
  Proceedings}. OpenReview.net, 2018.

\bibitem[Shimodaira(2000)]{shimodaira2000improving}
Hidetoshi Shimodaira.
\newblock Improving predictive inference under covariate shift by weighting the
  log-likelihood function.
\newblock \emph{Journal of statistical planning and inference}, 90\penalty0
  (2):\penalty0 227--244, 2000.

\bibitem[Smith and Gal(2018)]{gal2018understanding}
Lewis Smith and Yarin Gal.
\newblock Understanding measures of uncertainty for adversarial example
  detection.
\newblock In \emph{Proceedings of the Thirty-Fourth Conference on Uncertainty
  in Artificial Intelligence, {UAI} 2018, Monterey, California, USA, August
  6-10, 2018}, pages 560--569, 2018.

\bibitem[Ulmer et~al.(2020)Ulmer, Meijerink, and Cin{\`a}]{ulmer2020trust}
Dennis Ulmer, Lotta Meijerink, and Giovanni Cin{\`a}.
\newblock Trust issues: Uncertainty estimation does not enable reliable ood
  detection on medical tabular data.
\newblock In \emph{Machine Learning for Health Workshop (NeurIPS 2020)}, pages
  341--354. PMLR, 2020.

\bibitem[Welling and Teh(2011)]{welling2011bayesian}
Max Welling and Yee~W Teh.
\newblock Bayesian learning via stochastic gradient langevin dynamics.
\newblock In \emph{Proceedings of the 28th international conference on machine
  learning (ICML-11)}, pages 681--688, 2011.

\bibitem[Wilson and Izmailov(2020)]{wilson2020bayesian}
Andrew~Gordon Wilson and Pavel Izmailov.
\newblock Bayesian deep learning and a probabilistic perspective of
  generalization.
\newblock In Hugo Larochelle, Marc'Aurelio Ranzato, Raia Hadsell,
  Maria{-}Florina Balcan, and Hsuan{-}Tien Lin, editors, \emph{Advances in
  Neural Information Processing Systems 33: Annual Conference on Neural
  Information Processing Systems 2020, NeurIPS 2020, December 6-12, 2020,
  virtual}, 2020.

\end{thebibliography}

\appendix

\section{Additional Proofs}
This appendix section contains additional proofs and derivations that could not be included in the main paper due to spatial constraints. 

\subsection{Connection between Softmax and Sigmoid}\label{app-softmax-sigmoid-connection}

In this section we briefly outline the connection between the softmax and the sigmoid function, which was originally shown in \citet{bridle1990probabilistic}. Let the sigmoid function be defined as 
\begin{equation*}
    \sigma(x) = \frac{\exp(x)}{1 + \exp(x)}
\end{equation*}

and softmax according to the definition Section \ref{sec:uncertainty-metrics}. The output of $f_{\btheta}$ in a multi-class classification problem with $C$ classes corresponds to a $C$-dimensional column vector that is based on an affine transformation of the network's last intermediate hidden representation $\bx_L$, such that $f_{\btheta}(\bx) = \bW_L\bx_L$.\footnote{The bias term $\bb_L$ was omitted here for clarity.} Correspondingly, the output of $f_{\btheta}$ for a single class $c$ can be written as the dot product between $\bx_L$ and the corresponding row vector of $\bW_L$ denoted as $\bw_L^{(c)}$, such that $f_{\btheta}(\bx)_c \equiv {\bw_L^{(c)T}}\bx_L$. For a classification problem with $C=2$ classes, we can now rewrite the softmax probabilities in the following way:\footnote{The following argument holds without loss of generality for $p_{\btheta}(y=0|\bx)$.}
\begin{equation*}\begin{aligned}
    p_{\btheta}(y=1|\bx) & = \frac{\exp({\bw_L^{(1)T}}\bx_L)}{\exp({\bw_L^{(0)T}}\bx_L) + \exp({\bw_L^{(1)T}}\bx_L)} \\
\end{aligned}\end{equation*} 

Subtracting a constant from the weight term inside the exponential function does not change the output of the softmax function. Using this property, we can show the sigmoid function to be a special case of the softmax for binary classification:
\begin{equation*}\begin{aligned}
    & p_{\btheta}(y=1|\bx) \\
    & = \frac{\exp((\bw_L^{(1)} - \bw_L^{(0)})^T\bx_L)}{\exp((\bw_L^{(0)} - \bw_L^{(0)})^T\bx_L) + \exp((\bw_L^{(1)} - \bw_L^{(0)})^T\bx_L)} \\
    & = \frac{\exp((\bw_L^{(1)} - \bw_L^{(0)})^T\bx_L)}{1 + \exp((\bw_L^{(1)} - \bw_L^{(0)})^T\bx_L}  = \frac{\exp({\bw_L^{*T}}\bx_L)}{1 + \exp({\bw_L^{*T}}\bx_L)}
\end{aligned}\end{equation*} 

where $\bw_L^* = \bw_L^{(1)} - \bw_L^{(0)}$ corresponds to the new parameter vector which is used to parametrize a single output unit for a network in the binary classification setting.

\subsection{Linearization of ReLU networks}\label{app:relu-linearization}

In the section we give a more detailed version of the derivation of the linearization $f_{\btheta}(\bx) = \bV(\bx)\bx + \ba(\bx)$ with 

\begin{equation*}
    \bV(\bx) = \bW_L\bigg(\prod_{l=1}^{L-1}\bPhi_l(\bx)\bW_{L-l} \bigg)
\end{equation*}
\begin{equation*}
    \ba(\bx) = \bb_L + \sum_{l=1}^{L-1}\bigg(\prod_{l^\prime=1}^{L-l}\bW_{L+1-l^\prime}\bPhi_{L-l^\prime}(\bx)\bigg)\bb_l
\end{equation*}

We start from Equation \ref{eq:relu-net-replaced}:

\begin{equation*}\begin{aligned}
    f_{\btheta}(\bx) = & \bW_L\bPhi_{L-1}(\bx)\big(\bW_{L-1}\bPhi_{L-2}(\bx)\big(\ldots \\
     & \bPhi_1(\bx)\big(\bW_1\bx + \bb_1\big) \ldots\big) + \bb_{L-1} \big)+ \bb_L \\
\end{aligned}\end{equation*}

To make the steps more intuitive and to retain readability, we illustrate the necessary steps on a simple three layer network:

\begin{equation*}\begin{aligned}
    f_{\btheta}(\bx) = & \bW_3 \bPhi_{2}(\bx)\big(\bW_2\bPhi_1(\bx)\big(\bW_1\bx + \bb_1) + \bb_2) + \bb_3 \\
    = & \bW_3 \bPhi_{2}(\bx)\big(\bW_2\bPhi_1(\bx)\bW_1\bx + \bW_2\bPhi_1(\bx)\bb_1)\\
    & + \bb_2) + \bb_3 \\
    = & \underbrace{\bW_3 \bPhi_{2}(\bx)\bW_2\bPhi_1(\bx)\bW_1}_{=\bV(x)}\bx \\
    & + \underbrace{\bW_3 \bPhi_{2}(\bx)\bW_2\bPhi_1(\bx)\bb_1 + \bW_3\bPhi_{2}(\bx)\bb_2 + \bb_3}_{=\ba(\bx)} \\
\end{aligned}\end{equation*}

which we can identify as the parts of the affine transformation above.

\subsection{Construction of polytopal regions}\label{app:polytopes}

In this section, we reiterate the reasoning by \citet{hein2019relu} behind the construction the polytopal regions. For this purpose, the authors define an additional diagonal matrix $\bDelta_l(\bx)$ per layer $l$: 

\begin{equation*}
    \bDelta_l(\bx) = \begin{bmatrix}
    \text{sign}({f_{\btheta}^l(\bx)_1}) & \cdots & 0  \\
    \vdots &  \ddots & \vdots \\
    0 & \cdots & \text{sign}({f_{\btheta}^l(\bx)_{n_l}}) \\
    \end{bmatrix}
\end{equation*}

Together with the linearization of the network at $\bx$ explained in Appendix \ref{app:relu-linearization}, this is used to define a set of half-spaces for every neuron in the network:

\begin{equation*}
    \mathcal{H}_{l, i}(\bx) = \Big\{\bz \in \mathbb{R}^d\ \Big|\ \bDelta_l(\bx)\big(\bV_l(\bx)_i\bz + \ba_l(\bx)_i\big) \ge 0 \Big\}
\end{equation*}

Here, $\bV_l(\bx)_i$ and $\bb_l(\bx)_i$ denote the parts of the affine transformation obtained for the $i$-th neuron of the $l$-th layer, so the $i$-th row vector in $\bV_l(\bx)$ and the $i$-th scalar in $\bb_l(\bx)$, respectively. Finally, the polytope $Q$ containing $\bx$ is obtained by taking the intersection of all half-spaces induced by every neuron in the network: 

\begin{equation*}
    Q(\bx) = \bigcap_{l \in 1, \ldots, L}\bigcap_{i \in 1, ..., n_l} \mathcal{H}_{l, i}(\bx) 
\end{equation*}

\subsection{Proof of Proposition 1}\label{app:proposition1}

\begin{figure}
    \centering
    \includegraphics[width=0.9\columnwidth]{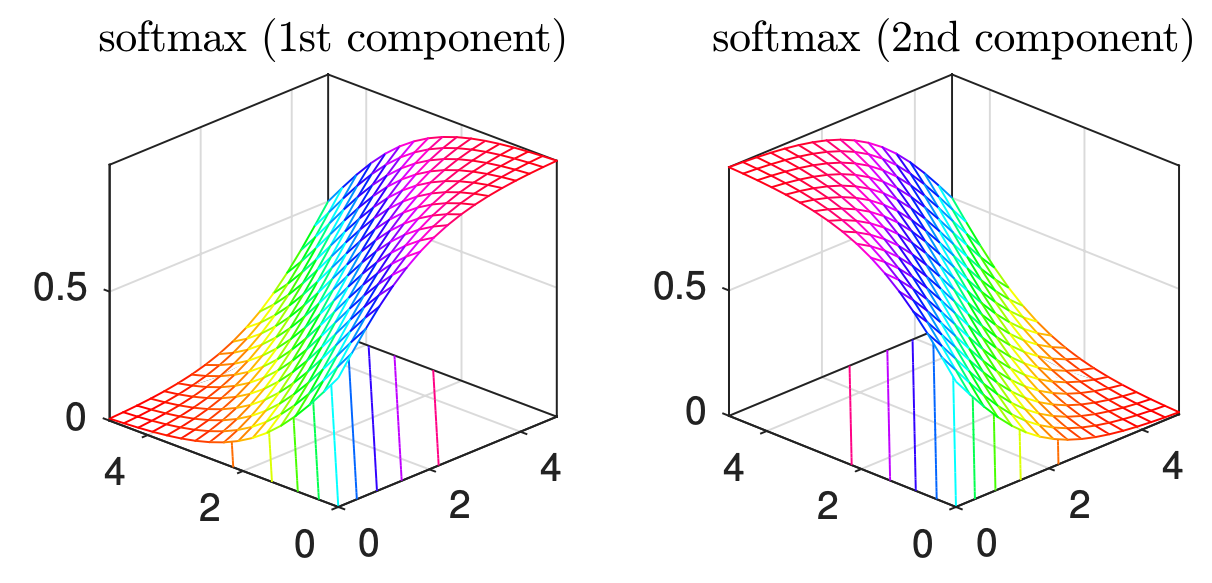}
    \caption{Illustration taken from the work of \citet{gao2017properties}, illustrating the interplay of softmax probabilities between components for $C=2$ in $\mathbb{R}^2$.}
    \label{fig:softmax-components}
\end{figure}

 We proceed to analyze the behaviour of gradients in the limit via two more lemmas; First, we establish the saturating property of the softmax In Lemma \ref{lemma:softmax-properties}, i.e. the model doesn't change its decision anymore in the limit.

\begin{customlemma}{9}\label{lemma:softmax-properties}
    Let $c, c^\prime \in \mathcal{C}$ be two arbitrary classes. It then holds for their corresponding output components (logits) that
    \begin{equation}\label{eq:softmax-asymptotic-behavior}
          \lim_{f_{\btheta}(\bx)_c \rightarrow \pm \infty} \frac{\partial}{\partial f_{\btheta}(\bx)_{c^\prime}}\bar{\sigma}(f_{\btheta}(\bx))_c = 0
    \end{equation}
\end{customlemma}

\begin{proof}
    Here, we first begin by evaluating the derivative of one component of the function w.r.t to an arbitrary component:
     \begin{equation*}\begin{aligned}
        &\ \ \ \frac{\partial}{\partial f_{\btheta}(\bx)_{c^\prime}}\bar{\sigma}(f_{\btheta}(\bx))_c = \frac{\partial}{\partial f_{\btheta}(\bx)_{c^\prime}} \frac{\exp(f_{\btheta}(\bx)_c)}{\sum_{c^{\pprime} \in \mathcal{C}} \exp(f_{\btheta}(\bx)_{c^{\pprime}})}  \\
        & =  \frac{\indicator{c = c^\prime} \exp(f_{\btheta}(\bx)_c)}{\sum_{c^{\pprime} \in \mathcal{C}} \exp(f_{\btheta}(\bx)_{c^{\pprime}})} - \frac{\exp(f_{\btheta}(\bx)_c)\exp(f_{\btheta}(\bx)_{c^\prime})}{\big(\sum_{c^{\pprime} \in \mathcal{C}} \exp(f_{\btheta}(\bx)_{c^{\pprime}})\big)^2} \\
        %& = \indicator{c = c^\prime}\bar{\sigma}(f_{\btheta}(\bx))_c - \bar{\sigma}(f_{\btheta}(\bx))_c\bar{\sigma}(f_{\btheta}(\bx))_{c^\prime} \\
        %& = \bar{\sigma}(f_{\btheta}(\bx))_c\big(\indicator{c = c^\prime} - \bar{\sigma}(f_{\btheta}(\bx))_{c^\prime} \big) 
    \end{aligned}\end{equation*}
    
    This implies that $\frac{\partial}{\partial f_{\btheta}(\bx)_{c^\prime}}\bar{\sigma}(f_{\btheta}(\bx))_c=$
    \begin{equation}\label{eq:softmax-derivative-cases}
        \begin{cases}\begin{aligned}
        \displaystyle
        & - \frac{\exp(2f_{\btheta}(\bx)_c)}{\big(\sum_{c^{\pprime} \in \mathcal{C}} \exp(f_{\btheta}(\bx)_{c^{\pprime}})\big)^2}\\
        & + \frac{\exp(f_{\btheta}(\bx)_{c})}{\sum_{c^{\pprime} \in \mathcal{C}}\exp(f_{\btheta}(\bx)_{c^{\pprime}})}  & \quad\text{If }c=c^\prime\\[0.5cm]
        \displaystyle
        & - \frac{\exp(f_{\btheta}(\bx)_c + f_{\btheta}(\bx)_{c^\prime})}{\big(\sum_{c^{\pprime}\in \mathcal{C}} \exp(f_{\btheta}(\bx)_{c^{\pprime}})\big)^2} & \quad\text{If }c\neq c^\prime\\
        \end{aligned}\end{cases}
    \end{equation}
    or more compactly:
    \begin{equation*}
        \frac{\partial}{\partial f_{\btheta}(\bx)_{c^\prime}}\bar{\sigma}(f_{\btheta}(\bx))_c  = \bar{\sigma}(f_{\btheta}(\bx))_c\big(\indicator{c = c^\prime} - \bar{\sigma}(f_{\btheta}(\bx))_{c^\prime} \big) 
    \end{equation*}
    
    Based on Equation \ref{eq:softmax-derivative-cases}, we can now investigate the asymptotic behavior for $f_{\btheta}(\bx)_c \rightarrow \infty$ more easily, starting with the $c = c^\prime$ case:
    \begin{equation}\begin{aligned}\label{eq:softmax-derivative-equal}
       & \lim_{f_{\btheta}(\bx)_c \rightarrow \infty} \frac{\partial}{\partial f_{\btheta}(\bx)_{c^\prime}}\bar{\sigma}(f_{\btheta}(\bx))_c \\
       & = \underbrace{ \lim_{f_{\btheta}(\bx)_c \rightarrow \infty}-\frac{\exp(f_{\btheta}(\bx)_c)}{\sum_{c^{\pprime} \in \mathcal{C}}\exp(f_{\btheta}(\bx)_{c^{\pprime}})}\frac{\exp(f_{\btheta}(\bx)_c)}{\sum_{c^{\pprime} \in \mathcal{C}}\exp(f_{\btheta}(\bx)_{c^{\pprime}})}}_{\text{-1}} \\
       & + \underbrace{ \lim_{f_{\btheta}(\bx)_c \rightarrow \infty}\frac{\exp(f_{\btheta}(\bx)_c)}{\sum_{c^{\pprime} \in \mathcal{C}}\exp(f_{\btheta}(\bx)_{c^{\pprime}})}}_{1} = 0 \\
    \end{aligned}\end{equation}
    
    With the numerator and denominator being dominated by the exponentiated $f_{\btheta}(\bx)_c$ in Equation \ref{eq:softmax-derivative-equal}, the first term will tend to $-1$, while the second term will tend to $1$, resulting in a derivative of $0$. The $c \neq c^\prime$ can be analyzed the following way:
    \begin{equation}\begin{aligned}\label{eq:softmax-derivative-unequal}
         & \lim_{f_{\btheta}(\bx)_c \rightarrow \infty} \frac{\partial}{\partial f_{\btheta}(\bx)_{c^\prime}}\bar{\sigma}(f_{\btheta}(\bx))_c \\
         & =  \underbrace{\lim_{f_{\btheta}(\bx)_c \rightarrow \infty}\bigg(-\frac{\exp(f_{\btheta}(\bx)_c)}{\sum_{c^{\pprime} \in \mathcal{C}}\exp(f_{\btheta}(\bx)_{c^{\pprime}})} \bigg)}_{-1} \\ 
         & \cdot \underbrace{\lim_{f_{\btheta}(\bx)_c \rightarrow  \infty}\bigg(\frac{\exp(f_{\btheta}(\bx)_{c^\prime})}{\sum_{c^{\pprime} \in \mathcal{C}}\exp(f_{\btheta}(\bx)_{c^{\pprime}})}\bigg)}_{0} = 0\\ 
    \end{aligned}\end{equation}
    
    Again, we factorize the fraction in Equation \ref{eq:softmax-derivative-unequal} into the product of two softmax functions, one for component $c$, one for $c^\prime$. The first factor will again tend to $-1$ as in the other case, however the second will approach $0$, as only the sum in the denominator will approach infinity. As the limit of a product is the products of its limits, this lets the whole expression approach $0$ in the limit.\\
    
    When $f_{\btheta}(\bx)_c \rightarrow -\infty$, both cases approach $0$ due to the exponential function, which proves the lemma.
\end{proof}
How to interplay between different softmax components produces zero gradients in the limit is illustrated in Figure \ref{fig:softmax-components}. In Lemma \ref{lemma:growth-rate-softmax}, we compare the rate of growth of different components of $p_{\btheta}$. We show that for the decomposed function $p_{\btheta}$, the rate at which the softmax function converges to its output distribution in the limit outpaces the change in the underlying logits w.r.t. the network input. 

\begin{customlemma}{10}\label{lemma:growth-rate-softmax}
    Suppose that $f_{\btheta}$ is a ReLU-network. Let $\bx^\prime\in\mathbb{R}^D$,  suppose $\balpha$ is a scaling vector and that the associated PUP $\mathcal{P}(\bx^\prime, d)$ has a corresponding matrix $\mathbf{V}$ with no zero entries. Then it holds that
    \begin{equation}\begin{aligned}\label{eq:growth-rate-softmax}
        \forall c^\prime \in \mathcal{C}, & \lim\limits_{\alpha_d \to \infty} \bigg(\frac{\partial}{\partial f_{\btheta}(\bx)_{c^\prime}}\bar{\sigma}(f_{\btheta}(\bx))_c\bigg)^{-1}\bigg|_{\bx = \balpha\odot\bx^\prime}\\
        & - \bigg(\frac{\partial}{\partial x_d}f_{\btheta}(\bx)_{c^\prime}\bigg)\bigg|_{\bx = \balpha\odot\bx^\prime}= \infty
    \end{aligned}\end{equation}
\end{customlemma}

\begin{proof}
    We evaluate the first term of Equation \ref{eq:growth-rate-softmax} to show that it grows exponentially in the limit. By Lemma \ref{lemma:unique-pup} we know that in the limit $\alpha_d \to \infty$ the vector $\balpha\odot\bx^\prime$ will remain within $\mathcal{P}(\bx^\prime, d)$. Since the matrix associated with this PUP has no zero entries, we know by Lemma \ref{lemma:strictly-monotonic} that the gradient of $f_{\btheta}(\bx)_c$ on dimension $d$ is either always positive or negative, hence $f_{\btheta}(\bx)_c \rightarrow \pm \infty$.  Given Lemma \ref{lemma:softmax-properties} describing the asymptotic behavior in the limit, it follows that 
    \begin{equation*}
        \lim_{f_{\btheta}(\bx)_c \rightarrow \pm \infty} \bigg(\frac{\partial}{\partial f_{\btheta}(\bx)_{c^\prime}}\bar{\sigma}(f_{\btheta}(\bx))_c\bigg)^{-1} = \infty
    \end{equation*}
    
    where we can see that the result is a symmetrical function displaying exponential growth in the limit of $f_{\btheta}(\bx)_c \rightarrow \pm \infty$. 
    We now show that because we assumed $f_{\btheta}$ to be a neural network consisting of $L$ affine transformations with ReLU activation functions, the output of the final layer is only going to be a linear combination of its inputs.\footnote{Here we make the argument for the whole function $f_{\btheta}: \mathbb{R}^D \rightarrow \mathbb{R}^C$, but the conclusions also applies to every output component of the function $f_{\btheta}(\bx)_c$.} This can be proven by induction. Let us first look at the base case $L=1$. In the rest of this proof, we denote $\bx_l$ as the input to layer $l$, with $\bx_1 \equiv \bx$, and $\bW_l, \bb_l$ the corresponding layer parameters. $\ba_l$ signifies the result of the affine transformation that is then fed into the activation function.
    \begin{equation}\begin{aligned}\label{eq:linear-trans-derivative}
        f_{\btheta}(\bx) & = \phi(\ba_1) = \phi(\bW_1\bx_1 + \bb_1) \\
        \frac{\partial f_{\btheta}(\bx)}{\partial \bx_1} & = \frac{\phi(\ba_1) }{\partial \ba_1} \frac{\partial \ba_1}{\partial \bx_1}= \bind(\bx_1 > \mathbf{0})^T\bW_1\ \\
        \frac{\partial f_{\btheta}(\bx)}{\partial x_{1d}} & = \indicator{x_d > 0}w_{1d}
    \end{aligned}\end{equation}
    
    where $\bind(\bx_1 > \mathbf{0}) = [\indicator{x_{11} > 0}, \ldots, \indicator{x_{1d} > 0}]^T$, $w_{1d}$ denotes the $d$-th column of $\bW_1$. This is a linear function, which proves the base case. Let now $\frac{\partial \bx_l}{\partial \bx_1}$ denote the partial derivative of the input to the $l$-th layer w.r.t to the input and suppose that it is linear by the inductive hypothesis. Augmenting the corresponding network by another linear adds another term akin to the second expression in Equation \ref{eq:linear-trans-derivative} to the chain of partial derivatives:
    \begin{equation}\label{eq:layer-induction-step}
        \frac{\partial \bx_{l+1}}{\partial \bx_1} = \frac{\partial \bx_{l+1}}{\partial \bx_l}\frac{\partial \bx_l}{\partial \bx_1}
    \end{equation}
    
    which is also a linear function of, proving the induction step. Because we know that both terms of the product in Equation \ref{eq:layer-induction-step} are linear, the second term of the Equation \ref{eq:growth-rate-softmax} is as well. Together with the previous insight that the first term is exponential, this implies that it will outgrow the second in the limit, creating an infinitively-wide gap between them and thereby proving the lemma.
\end{proof}

Equipped with the results of Lemmas \ref{lemma:softmax-properties} and \ref{lemma:growth-rate-softmax}, we can finally prove the proposition:
\begin{proof}
    We show that one scalar factor contained in the factorization of the gradient $\nabla_{\bx}p_{\btheta}(y=c|\bx)$ tends to zero under the given assumptions, having the whole gradient become the zero vector in the limit. We begin by again factorizing the gradient $\nabla_{\bx} p_{\btheta}(y=c|\bx)$ using the multivariate chain rule:
    \begin{equation}\label{eq:fac-gradient-softmax}
        \nabla_{\bx}p_{\btheta}(y=c|\bx) = \sum_{c^\prime=1}^C \frac{\partial}{\partial  f_{\btheta}(\bx)_{c^\prime}}\bar{\sigma}(f_{\btheta}(\bx))_c\cdot  \nabla_{\bx}f_{\btheta}(\bx)_{c^\prime}
    \end{equation}
    By Lemma \ref{lemma:strictly-monotonic} and \ref{lemma:unique-pup} we know that $f_{\btheta}$ is a component-wise strictly monotonic function on $\mathcal{P}(\bx^\prime, d)$, which implies for the limit of $\alpha_d \rightarrow \infty$ that $\forall c \in \mathcal{C}: f_{\btheta}(\bx)_c \rightarrow \pm \infty$. Then, Lemma \ref{lemma:softmax-properties} implies that the first factor of every part in the sum of Equation \ref{eq:fac-gradient-softmax} will tend to zero in the limit. Lemma \ref{lemma:growth-rate-softmax} ensures that the first factor approximates zero quicker than every component of the gradient $\nabla_{\bx}f_{\btheta}(\bx)_{c^\prime}$ potentially approaching infinity, causing the product to result in the zero vector. As this results in a sum over $C$ zero vectors in the limit, this proves the lemma.
    \end{proof}

\subsection{Proof of Proposition 2}\label{app:softmax-limit}
%\ref{proposition:softmax-limit}}\label{app:softmax-limit}

\begin{proof}
    We start by rewriting the softmax probability for the $c$-th logit:
    \begin{equation*}\begin{aligned}
        \bar{\sigma}(f_{\btheta}(\bx))_c & = \frac{\exp(f_{\btheta}(\bx)_c)}{\sum_{c^\prime \in \mathcal{C}}\exp(f_{\btheta}(\bx)_{c^\prime})} \\
         & = 1 - \frac{\sum_{c^{\prime\prime} \in \mathcal{C} \setminus \{ c\}}\exp(f_{\btheta}(\bx)_{c^{\prime\prime}})}{\sum_{c^\prime \in \mathcal{C}}\exp(f_{\btheta}(\bx)_{c^\prime})}
    \end{aligned}\end{equation*}
    % Let $c \in \mathcal{C}$ be an arbitrary class s.t. $\forall c^\prime \neq c:\ v_{cd} > v_{c^\prime d}$.  
    %We proceed to prove the case of the limit $\alpha_d \rightarrow \infty$. 
    By Lemma \ref{lemma:strictly-monotonic} and \ref{lemma:unique-pup} we have shown that $f_{\btheta}$ is a component-wise strictly monotonic function on $\mathcal{P}(\bx^\prime, d)$, %by the previous Lemma we know that $f_{\btheta}$.
    so we know that $\forall c^\prime \in \mathcal{C}: f_{\btheta}(\bx)_{c^\prime} \rightarrow \pm \infty$ as $\alpha_d \rightarrow \infty$. We now treat the two limits $\pm \infty$ in order.
    Because of the assumption that $d$-column of $\mathbf{V}$ has no duplicate entries, this implies that  there must be a $c \in \mathcal{C}$ s.t. $\forall c^\prime \neq c:\ v_{cd} > v_{c^\prime d}$. Thus, in the limit of $f_{\btheta}(\bx)_c \rightarrow \infty$, the sum in the \emph{denominator} of the fraction including the logit of $c$ will tend to infinity faster than the the sum in the \emph{numerator} not including $c$'s logit, and thus the fraction itself will tend to $0$, proving this case. In the case of $f_{\btheta}(\bx)_c \rightarrow -\infty$, the \emph{numerator} of the fraction will tend to $0$ faster than the \emph{denominator}, having the fraction approach $0$ in the limit as well, proving the second case and therefore the lemma. 
\end{proof}

\subsection{Proof of Lemma 4}\label{app:aggregation-theorem}
%\ref{aggregation-theorem}}
\begin{proof}
    \begin{align*}
        & \lim\limits_{\alpha \to \infty}\bbnorm\nabla_{\bx}\ \Expect[\bigg]{{p(\btheta|\mathcal{D})}}{p_{\btheta}(y=c|\bx)}\bigg|_{\bx = \balpha\odot\bx^\prime}\bbnorm_2\\
        \intertext{Linearity of gradient:}
        = & \lim\limits_{\alpha \to \infty}\bbnorm \Expect[\bigg]{{p(\btheta|\mathcal{D})}}{\nabla_{\bx} p_{\btheta}(y=c|\bx)}\bigg|_{\bx = \balpha\odot\bx^\prime}\bbnorm_2 \\
        \intertext{Utilize Jensen's inequality  $\phi(\Expect{}{\bx}) \le \Expect{}{\phi(\bx)}$ as $l_2$-norm is a convex function and Proposition \ref{proposition:overconfidence-softmax}:}
        \le & \lim\limits_{\alpha \to \infty} \Expect[\bigg]{{p(\btheta|\mathcal{D})}}{\ \underbrace{\bbnorm \nabla_{\bx} p_{\btheta}(y=c|\bx)\bigg|_{\bx = \balpha\odot\bx^\prime}\bbnorm_2}_{= 0\ (\text{Proposition } \ref{proposition:overconfidence-softmax})}\ } = 0 \\
    \end{align*}
    Because the last expression is an upper bound to the original expression and the $l_2$ norm is lower-bounded by $0$, this proves the lemma.
\end{proof}

\subsection{Proof of Lemma 5}\label{app:asymptotic-softmax-variance} %\ref{lemma:asymptotic-softmax-variance}}\label{app:asymptotic-softmax-variance}

\begin{lemma}{(Asymptotic behavior with softmax variance)}\label{lemma:asymptotic-softmax-variance}
      Suppose that $f_{\btheta}^{(1)}, \ldots, f_{\btheta}^{(K)}$ are ReLU networks. Let $\bx^\prime\in\mathbb{R}^D$,  suppose $\balpha$ is a scaling vector and that for all $k$, the associated PUP $\mathcal{P}^{(k)}(\bx^\prime, d)$ has a corresponding matrix $\mathbf{V}^{(k)}$ with no zero entries. It holds that 
    \begin{equation*}\begin{aligned}
        & \lim\limits_{\alpha_d \to \infty}\bbnorm\nabla_{\bx}\ \frac{1}{C}\sum_{c=1}^C \Expect[\bigg]{{p(\btheta|\mathcal{D})}}{\Big(p_{\btheta}(y=c|\bx)\Big)^2}\\ 
        & - \Expect[\bigg]{{p(\btheta|\mathcal{D})}}{p_{\btheta}(y=c|\bx)}^2\bigg|_{\bx = \balpha\odot\bx^\prime}\bbnorm_2 = 0
    \end{aligned}\end{equation*}

\end{lemma}

\begin{proof}   
    %\begin{equation}\begin{aligned}
    \begin{align*}
         & \lim\limits_{\alpha \to \infty}\bbnorm \nabla_{\bx} \frac{1}{C}\sum_{c=1}^C \Expect[\bigg]{{p(\btheta|\mathcal{D})}}{ \Big(p_{\btheta}(y=c|\bx)\Big)^2} \\
          - & \Expect[\bigg]{{p(\btheta|\mathcal{D})}}{p_{\btheta}(y=c|\bx)}^2\bigg|_{\bx = \balpha\odot\bx^\prime}\bbnorm_2 \\
          \intertext{Linearity of gradient:}
          = & \lim\limits_{\alpha_d \to \infty}\bbnorm\frac{1}{C}\sum_{c=1}^C \nabla_{\bx} \Expect[\bigg]{{p(\btheta|\mathcal{D})}}{\Big(p_{\btheta}(y=c|\bx)\Big)^2} \\ 
          - & \nabla_{\bx}\Expect[\bigg]{{p(\btheta|\mathcal{D})}}{p_{\btheta}(y=c|\bx)}^2\bigg|_{\bx = \balpha\odot\bx^\prime}\bbnorm_2 \\
        \intertext{Apply triangle inequality $||x + y|| \le ||x|| + ||y||$ to sum over all $c$:}
         \le & \lim\limits_{\alpha_d \to \infty}\frac{1}{C}\sum_{c=1}^C \bbnorm \nabla_{\bx} \Expect[\bigg]{{p(\btheta|\mathcal{D})}}{\Big(p_{\btheta}(y=c|\bx)\Big)^2}\\
         - & \nabla_{\bx}\Expect[\bigg]{{p(\btheta|\mathcal{D})}}{p_{\btheta}(y=c|\bx)}^2\bigg|_{\bx = \balpha\odot\bx^\prime}\bbnorm_2 \\
          \intertext{On the first term use linearity of gradients and apply chain rule, do it in the reverse order on the second term:}
          = & \lim\limits_{\alpha_d \to \infty}\\ 
          & \frac{1}{C}\sum_{c=1}^C \bbnorm  \Expect[\bigg]{p(\btheta|\mathcal{D})}{2p_{\btheta}(y=c|\bx)\underbrace{\mystrut{0.325cm}{\nabla_{\bx}p_{\btheta}(y=c|\bx)}}_{= \bm{0} \text{ (Proposition \ref{proposition:overconfidence-softmax})}}\bigg|_{\bx = \balpha\odot\bx^\prime}} \\
            - & \Bigg(2\Expect[\bigg]{{p(\btheta|\mathcal{D})}}{p_{\btheta}(y=c|\bx)}\Bigg)\\
           \cdot & \Expect[\bigg]{{p(\btheta|\mathcal{D})}}{\underbrace{\mystrut{0.4cm}{\nabla_{\bx}p_{\btheta}(y=c|\bx)}}_{= \bm{0} \text{ (Proposition  \ref{proposition:overconfidence-softmax})}}\bigg|_{\bx = \balpha\odot\bx^\prime}}\bbnorm_2 = 0 \\
    \end{align*} 
    We can see that due to an intermediate result of Proposition \ref{proposition:overconfidence-softmax}, i.e. that $\nabla_{\bx}p_{\btheta}(y=c|\bx)$ approaches the zero vector in the limit, the innermost gradients tend to zero, bringing the whole expression to 0.

    Because the final is an upper bound to the original expression and because the $l_2$ norm has a lower bound of $0$, this proves the lemma.
\end{proof}

\subsection{Proof of Lemma 6}\label{app:asymptotic-predictive-entropy} %\ref{lemma:asymptotic-predictive-entropy}}\label{app:asymptotic-predictive-entropy}

\begin{lemma}{(Asymptotic behavior for predictive entropy)}\label{lemma:asymptotic-predictive-entropy}
      Suppose that $f_{\btheta}^{(1)}, \ldots, f_{\btheta}^{(K)}$ are ReLU networks. Let $\bx^\prime\in\mathbb{R}^D$, suppose $\balpha$ is a scaling vector  and that for all $k$, the associated PUP $\mathcal{P}^{(k)}(\bx^\prime, d)$ has a corresponding matrix $\mathbf{V}^{(k)}$ with no zero entries. It holds that
    \begin{equation*}
        \lim\limits_{\alpha_d \to \infty}\bbnorm\nabla_{\bx}\mathbb{H}\bigg[\Expect[\Big]{{p(\btheta|\mathcal{D})}}{p_{\btheta}(y|\bx)}\bigg]\bigg|_{\bx = \balpha\odot\bx^\prime}\bbnorm_2 = 0
    \end{equation*}
\end{lemma}

\begin{proof}
    \begin{align*}
        & \lim\limits_{\alpha_d \to \infty}\bbnorm\nabla_{\bx}\mathbb{H}\bigg[\Expect[\Big]{{p(\btheta|\mathcal{D})}}{p_{\btheta}(y|\bx)}\bigg]\bigg|_{\bx = \balpha\odot\bx^\prime}\bbnorm_2 \\
         & = \lim\limits_{\alpha_d \to \infty}\bbnorm\nabla_{\bx}\Bigg(\sum_{c=1}^C\ \Expect[\Big]{{p(\btheta|\mathcal{D})}}{p_{\btheta}(y=c|\bx)}\\
         & \cdot \log\bigg(\Expect[\Big]{{p(\btheta|\mathcal{D})}}{p_{\btheta}(y=c|\bx)}\bigg)\Bigg)\bigg|_{\bx = \balpha\odot\bx^\prime}\bbnorm_2 \\
         \intertext{Linearity of gradient:}
         & = \lim\limits_{\alpha_d \to \infty}\bbnorm\sum_{c=1}^C\ \nabla_{\bx}\Bigg(\Expect[\Big]{{p(\btheta|\mathcal{D})}}{p_{\btheta}(y=c|\bx)}\\
         & \cdot\log\bigg(\Expect[\Big]{{p(\btheta|\mathcal{D})}}{p_{\btheta}(y=c|\bx)}\bigg)\Bigg)\bigg|_{\bx = \balpha\odot\bx^\prime}\bbnorm_2 \\
         \intertext{Apply product rule:}
         & = \lim\limits_{\alpha_d \to \infty}\bbnorm\Bigg(\sum_{c=1}^C\ \Expect[\Big]{{p(\btheta|\mathcal{D})}}{p_{\btheta}(y=c|\bx)}\\
         & \cdot\bigg(\Expect[\Big]{{p(\btheta|\mathcal{D})}}{p_{\btheta}(y=c|\bx)}\bigg)^{-1}\\ & \cdot\nabla_{\bx}\Bigg(\Expect[\Big]{{p(\btheta|\mathcal{D})}}{p_{\btheta}(y=c|\bx)}\Bigg) 
        \\
        & + \nabla_{\bx}\bigg(\Expect[\Big]{{p(\btheta|\mathcal{D})}}{p_{\btheta}(y=c|\bx)}\bigg)\\
        & \cdot \log\bigg(\Expect[\Big]{{p(\btheta|\mathcal{D})}}{p_{\btheta}(y=c|\bx)}\bigg) \bigg|_{\bx = \balpha\odot\bx^\prime}\bbnorm_2 \\
        \intertext{Factor out gradient:}
        & = \lim\limits_{\alpha_d \to \infty}\bbnorm\sum_{c=1}^C\ \nabla_{\bx}\Expect[\Big]{{p(\btheta|\mathcal{D})}}{p_{\btheta}(y=c|\bx)}\\
        & \cdot \bigg(1 + \log\bigg(\Expect[\Big]{{p(\btheta|\mathcal{D})}}{p_{\btheta}(y=c|\bx)}\bigg)\bigg) \bigg|_{\bx = \balpha\odot\bx^\prime}\bbnorm_2 \\
        \intertext{Apply triangle inequality to sum over all $c$:}
        \le & \lim\limits_{\alpha_d \to \infty}\sum_{c=1}^C\bbnorm \nabla_{\bx}\Expect[\bigg]{{p(\btheta|\mathcal{D})}}{p_{\btheta}(y=c|\bx)} \\
        & \cdot \bigg(1 + \log\bigg(\Expect[\bigg]{{p(\btheta|\mathcal{D})}}{p_{\btheta}(y=c|\bx)}\bigg)\bigg) \bigg|_{\bx = \balpha\odot\bx^\prime}\bbnorm_2 \\
        \intertext{As the log expectation just evaluates to a scalar, it can be pulled out of the norm and we can apply Lemma \ref{aggregation-theorem}}
         & = \lim\limits_{\alpha_d \to \infty}\sum_{c=1}^C\ \underbrace{\bigg(1 + \log\bigg( \Expect[\Big]{{p(\btheta|\mathcal{D})}}{p_{\btheta}(y=c|\bx)}\bigg)\bigg)}_{\text{Scalar}}\\
         & \cdot\underbrace{\bbnorm\nabla_{\bx}\Expect[\Big]{{p(\btheta|\mathcal{D})}}{ p_{\btheta}(y=c|\bx)}\bigg|_{\bx = \balpha\odot\bx^\prime}\bbnorm_2}_{=0\ (\text{Lemma } \ref{aggregation-theorem})} = 0 \\
    \end{align*}
    As the final result is an upper bound to the original expression and is lower-bounded by $0$ due to the $l_2$ norm, this proves the lemma.
\end{proof}

\subsection{Proof of Lemma 7}\label{app:asymptotic-mutual-information} %\ref{lemma:asymptotic-mutual-information}}\label{app:asymptotic-mutual-information}

\begin{lemma}{(Asymptotic behavior for approximate mutual information)}\label{lemma:asymptotic-mutual-information}
      Suppose that $f_{\btheta}^{(1)}, \ldots, f_{\btheta}^{(K)}$ are ReLU networks. Let $\bx^\prime\in\mathbb{R}^D$, suppose $\balpha$ is a scaling vector and that for all $k$, the associated PUP $\mathcal{P}^{(k)}(\bx^\prime, d)$ has a corresponding matrix $\mathbf{V}^{(k)}$ with no zero entries. It holds that
    \begin{equation*}\begin{aligned}
        & \lim\limits_{\alpha_d \to \infty}\bbnorm\nabla_{\bx}\bigg(\mathbb{H}\bigg[\Expect[\Big]{{p(\btheta|\mathcal{D})}}{p_{\btheta}(y|\bx)}\bigg]\\
        & -  \Expect[\bigg]{{p(\btheta|\mathcal{D})}}{\mathbb{H}\Big[p_{\btheta}(y|\bx)\Big]}\bigg)\bigg|_{\bx = \balpha\odot\bx^\prime}\bbnorm_2 = 0
    \end{aligned}\end{equation*}
    
\end{lemma}

\begin{proof}
    \begin{align*}
        & \lim\limits_{\alpha_d \to \infty}\bbnorm\nabla_{\bx}\bigg(\mathbb{H}\bigg[\Expect[\Big]{{p(\btheta|\mathcal{D})}}{p_{\btheta}(y|\bx)}\bigg]\\
        & - \Expect[\bigg]{{p(\btheta|\mathcal{D})}}{\mathbb{H}\Big[p_{\btheta}(y|\bx)\Big]}\bigg)\bigg|_{\bx = \balpha\odot\bx^\prime}\bbnorm_2
        \intertext{Linearity of gradients:}
        \le & \lim\limits_{\alpha_d \to \infty}\bbnorm\bigg(\nabla_{\bx}\mathbb{H}\bigg[\Expect[\Big]{{p(\btheta|\mathcal{D})}}{p_{\btheta}(y|\bx)}\bigg]\\
        & - \nabla_{\bx}\Expect[\bigg]{{p(\btheta|\mathcal{D})}}{\mathbb{H}\Big[p_{\btheta}(y|\bx)\Big]}\bigg)\bigg|_{\bx = \balpha\odot\bx^\prime}\bbnorm_2 \\
        \intertext{Linearity of gradients on second part of difference:}
        = & \lim\limits_{\alpha_d \to \infty}\bbnorm\bigg(\nabla_{\bx}\mathbb{H}\bigg[\Expect[\Big]{{p(\btheta|\mathcal{D})}}{p_{\btheta}(y|\bx)}\bigg] \\
        & - \Expect[\bigg]{{p(\btheta|\mathcal{D})}}{\nabla_{\bx}\mathbb{H}\Big[p_{\btheta}(y|\bx)\Big]}\bigg)\bigg|_{\bx = \balpha\odot\bx^\prime}\bbnorm_2 \\
        \intertext{Applying chain rule and intermediate result of Proposition \ref{proposition:overconfidence-softmax}:}
        = & \lim\limits_{\alpha_d \to \infty}\bbnorm\nabla_{\bx}\mathbb{H}\bigg[\Expect[\Big]{{p(\btheta|\mathcal{D})}}{p_{\btheta}(y|\bx)}\bigg]\bigg|_{\bx = \balpha\odot\bx^\prime} \\
        & - \Expect[\bigg]{{p(\btheta|\mathcal{D})}}{\sum_{c=1}^C\Big( 1 + \log p_{\btheta}(y=c|\bx)\Big)\underbrace{\mystrut{0.275cm}{\nabla_{\bx}p_{\btheta}(y=c|\bx)}}_{= \bm{0} \text{  Proposition \ref{proposition:overconfidence-softmax})}}}\\
        & \bigg|_{\bx = \balpha\odot\bx^\prime}\bbnorm_2 \\
        \intertext{Because this lets the entire second term become the zero vector in the limit, the remaining part reduces to the case proven in Lemma \ref{lemma:asymptotic-predictive-entropy}:}
        = & \underbrace{\lim\limits_{\alpha_d \to \infty}\bbnorm\nabla_{\bx}\mathbb{H}\bigg[\Expect[\Big]{{p(\btheta|\mathcal{D})}}{p_{\btheta}(y|\bx)}\bigg]\bigg|_{\bx = \balpha\odot\bx^\prime} \bbnorm_2}_{\text{Lemma } \ref{lemma:asymptotic-predictive-entropy}} = 0 \\
    \end{align*}
   As the final result is an upper bound to the original expression and the $l_2$ norm provides a lower bound of $0$, this proves the lemma.
\end{proof}

\section{Synthetic Data Experiments}\label{app:synthetic-data-experiments}

We perform our experiments on the half-moons dataset, using the corresponding function to generate the dataset in \texttt{scikit-learn} \citep{pedregosa2011scikit}, producing $500$ samples for training and $250$ samples for validation using a noise level of $.125$.

We do hyperparameter search using the ranges listed in Table \ref{tab:hyperparameters-search-space}, settling on the values given in Table \ref{tab:best_hyperparameters} after $200$ evaluation runs per model (for \texttt{NN} and \texttt{MCDropout}; the hyperparameters found for \texttt{NN} were then used for \texttt{PlattScalingNN}, \texttt{AnchoredNNEnsemble}, \texttt{NNEnsemble} as well). We also performed a similar hyperparameter search for the Bayes-by-backprop \citep{blundell2015weight} model, which seemed to not have yielded a suitable configuration even after extensive search, which is why results were omitted here. All models were trained with a batch size of $64$ and for $20$ epochs at most using early stopping with a patience of $5$ epochs and the Adam optimizer. 

All of the plots produced can be found in Figure \ref{fig:app-single-pred-nn} and \ref{fig:app-multiple-pred-nn}, where uncertainty values where plotted for different ranges depending on the metric (variance: $0$-$0.25$; (negative) entropy: $0$-$1$; mutual information: $4-5$; (1 -) max. prob: $0 - 0.5$), with deep purple signifying high uncertainty and white signifying low uncertainty / high certainty.

\begin{table}[h!]
    \centering
    \caption{Best hyperparameters found on the half-moon dataset.}
    \resizebox{0.9\columnwidth}{!}{%
        \begin{tabular}{rrl}
            \toprule
            Model & Hyperparameter & Value  \\
            \midrule
            \texttt{NN} & \texttt{hidden\_sizes} & $[25, 25, 25]$  \\
            \texttt{NN} & \texttt{dropout\_rate} & $0.014552$ \\
            \texttt{NN} & \texttt{lr} & $0.000538$  \\
            \texttt{MCDropout} & \texttt{hidden\_sizes} & $[25, 25, 25, 25]$  \\
            \texttt{MCDropout} & \texttt{dropout\_rate} & $0.205046$  \\
            \texttt{MCDropout} & \texttt{lr} & $0.000526$\\
            %\texttt{BBB} & \texttt{hidden\_sizes} & $[25]$ \\
            %\texttt{BBB} & \texttt{dropout\_rate} & $0.267797$ \\
            %\texttt{BBB} & \texttt{lr} & $0.00689$  \\
            %\texttt{BBB} & \texttt{posterior\_mu\_init} & $0.284989$ \\
            %\texttt{BBB} & \texttt{posterior\_rho\_init} & $-7.649002$  \\
            %\texttt{BBB} & \texttt{prior\_pi} & $0.196422$  \\
            %\texttt{BBB} & \texttt{prior\_sigma\_1} & $0.606531$  \\
            %\texttt{BBB} & \texttt{prior\_sigma\_2} & $0.904837$  \\
            \bottomrule
        \end{tabular}% 
    }
    \label{tab:best_hyperparameters}
\end{table}

\begin{table}[ht!]
    \caption{Distributions or options that hyperparameters were sampled from during the random hyperparameter search.} 
    \resizebox{\columnwidth}{!}{%
        \begin{tabular}{rrrl}
            \toprule
            Hyperparameter & Description & Chosen from \\
            \midrule
            \tt{hidden\_sizes} & Hidden layers & 1-5 layers of $15$, $20$, $25$ \\
            \tt{lr} & Learning rate &  $\mathcal{U}(\log(10^{-4}), \log(0.1))$ \\
            \tt{dropout\_rate} & Dropout rate & $\mathcal{U}(0, 0.5)$ \\ 
            %\tt{posterior\_rho\_init} & Variance parameter of weight posterior$^{(*)}$ & \tt{BBB} & $\mathcal{U}(-8, -2)$ \\
            %\tt{posterior\_mu\_init} & Mean parameter of weight posterior$^{(*)}$ & \tt{BBB} & $\mathcal{U}(-0.6, 0.6)$ \\
            %\tt{prior\_pi} & Mixture component of prior$^{(*)}$ & \tt{BBB} & $\mathcal{U}(\exp(0.1), \exp(0.9))$\\
            %\tt{prior\_sigma\_1} & Variance of prior mixture component 1$^{(*)}$ & \tt{BBB} & $\mathcal{U}(\exp(-0.8), \exp(0.1))$\\
            %\tt{prior\_sigma\_2} & Variance of prior mixture component 2$^{(*)}$ & \tt{BBB} & $\mathcal{U}(\exp(-0.8), \exp(0.1))$\\
            \bottomrule
        \end{tabular}%
    }
    %$^{(*)}$For more information about these hyperparameters the reader is referred to the work of \cite{blundell2015weight}.}
     \label{tab:hyperparameters-search-space}
\end{table}

We can see in Figure \ref{fig:app-single-pred-nn} that maximum probability and predictive entropy behave quite similarly, forming a tube-like region of high uncertainty along what appear to be the decision boundary. In both cases, the region appears to be sharper in the case of maximum probability (right column) and also more defined after additional temperature scaling (bottom row). For all models and metrics, we see that the gradient magnitude decreases and approaches zero away from the training data (yellow / green plots), except for the cases discussed in Section \ref{sec:experiments}.

\begin{figure}[h!]
    \centering
    \resizebox{\columnwidth}{!}{%
    \begin{tabular}{rll}
         & \multicolumn{1}{c}{Predictive Entropy} & \multicolumn{1}{c}{Maximum probability} \\
        \rotatebox{90}{\hspace{-0.5cm}\texttt{NN}}             & \includegraphics[width=0.28\textwidth]{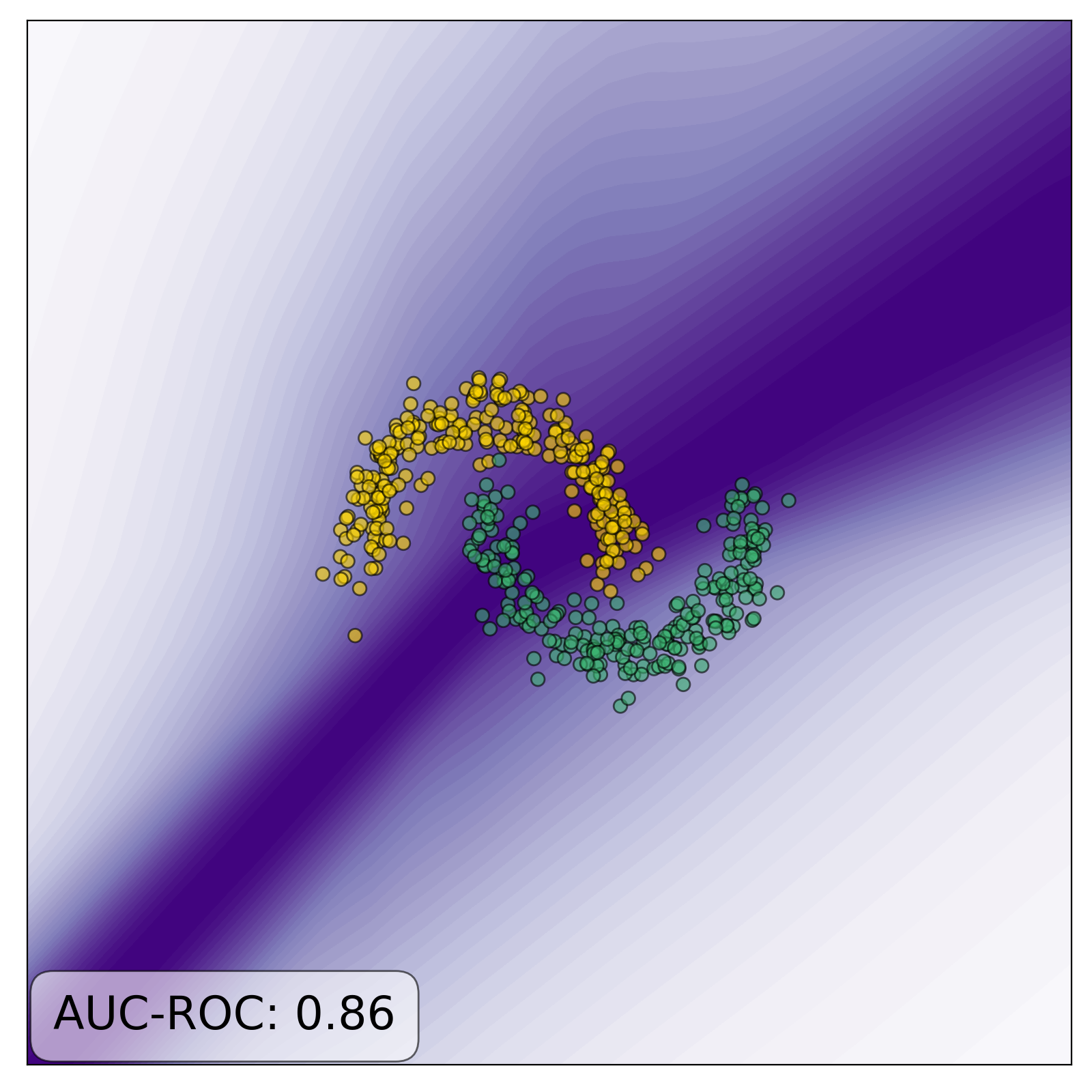} & \includegraphics[width=0.28\textwidth]{img/nn_max_prob.png} \\
        & \includegraphics[width=0.3\textwidth]{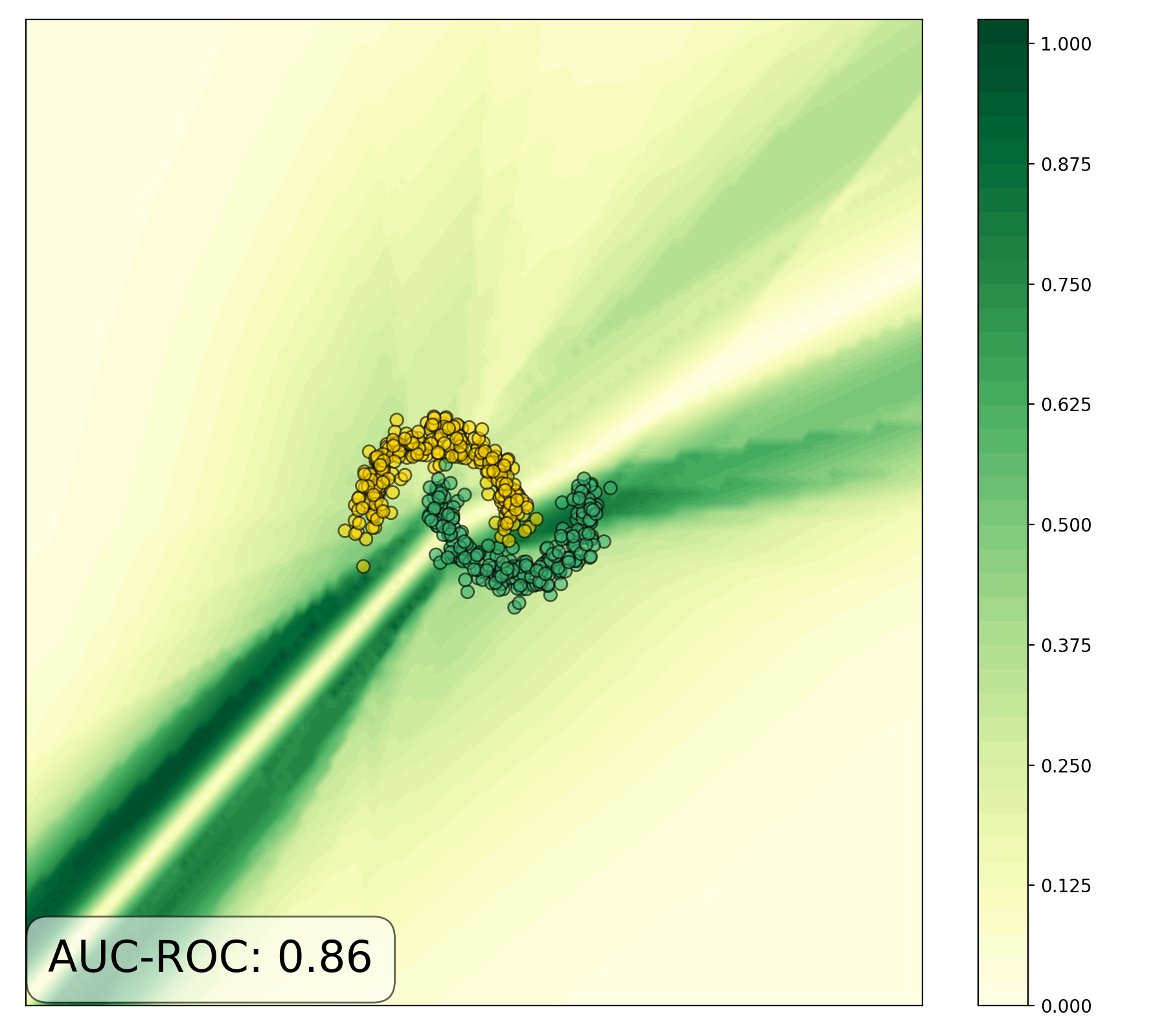} & \includegraphics[width=0.3\textwidth]{img/nn_max_prob_grads.png} \\
        \midrule
        \rotatebox{90}{\hspace{-1.5cm}\texttt{PlattScalingNN}} & \includegraphics[width=0.28\textwidth]{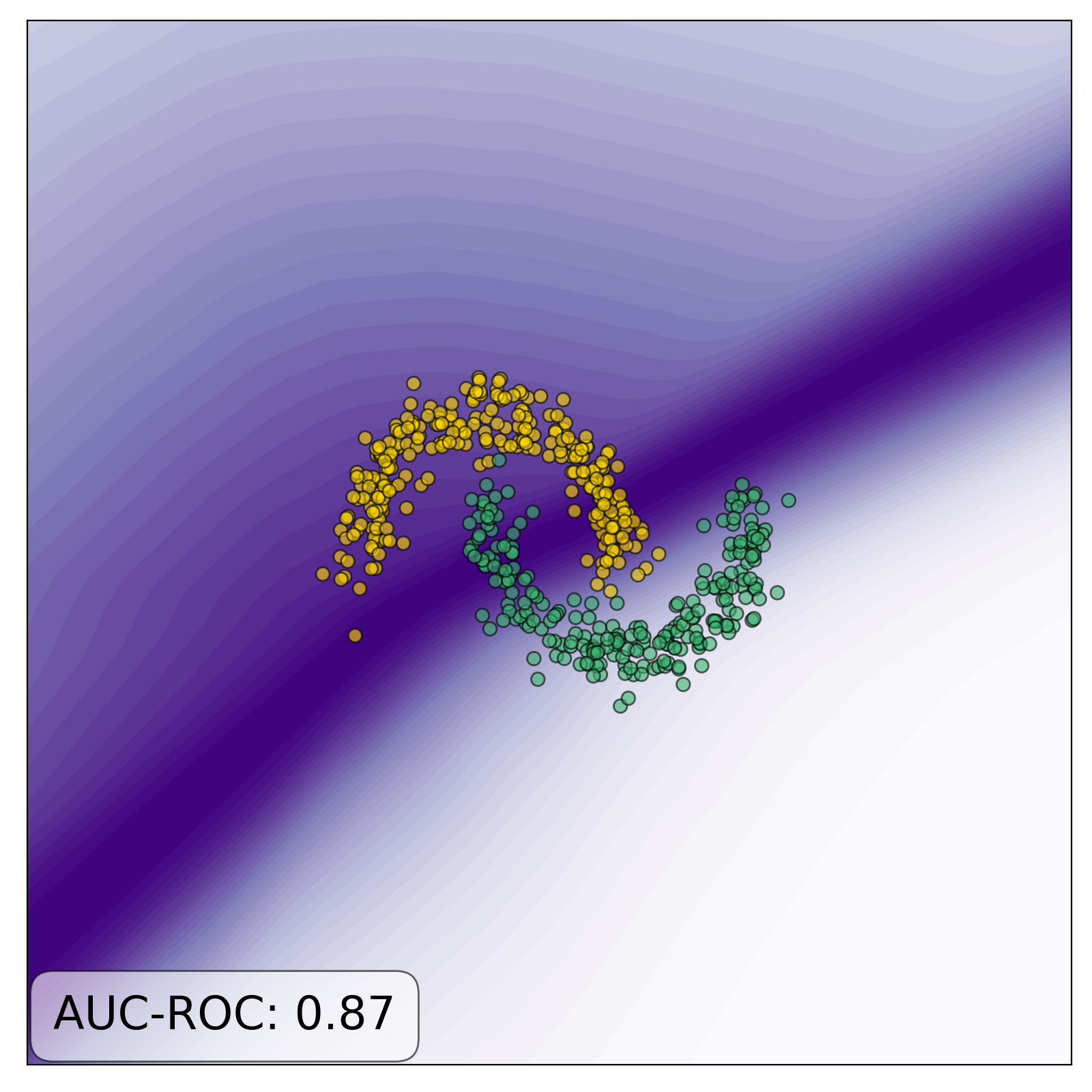} & \includegraphics[width=0.28\textwidth]{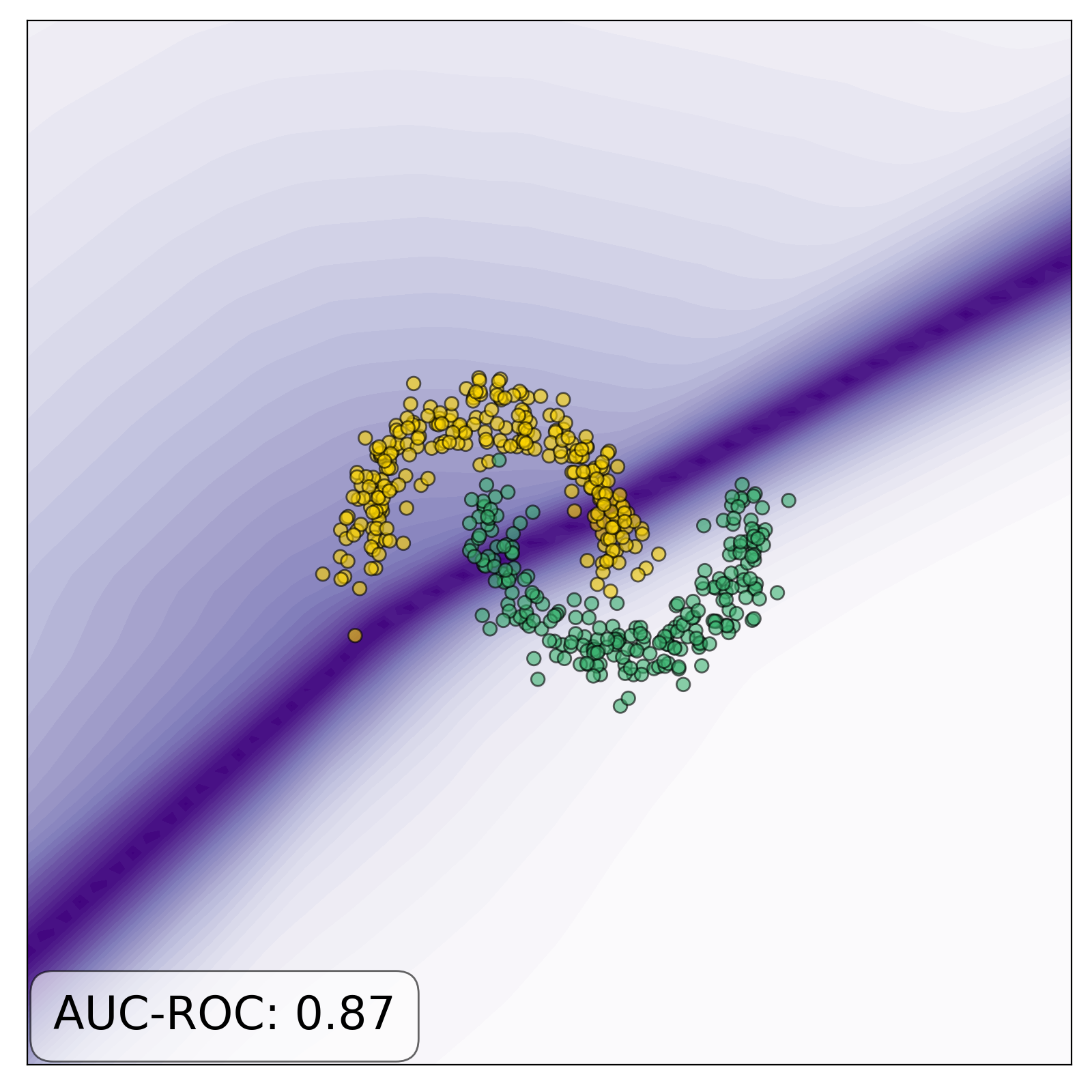}  \\
        & \includegraphics[width=0.3\textwidth]{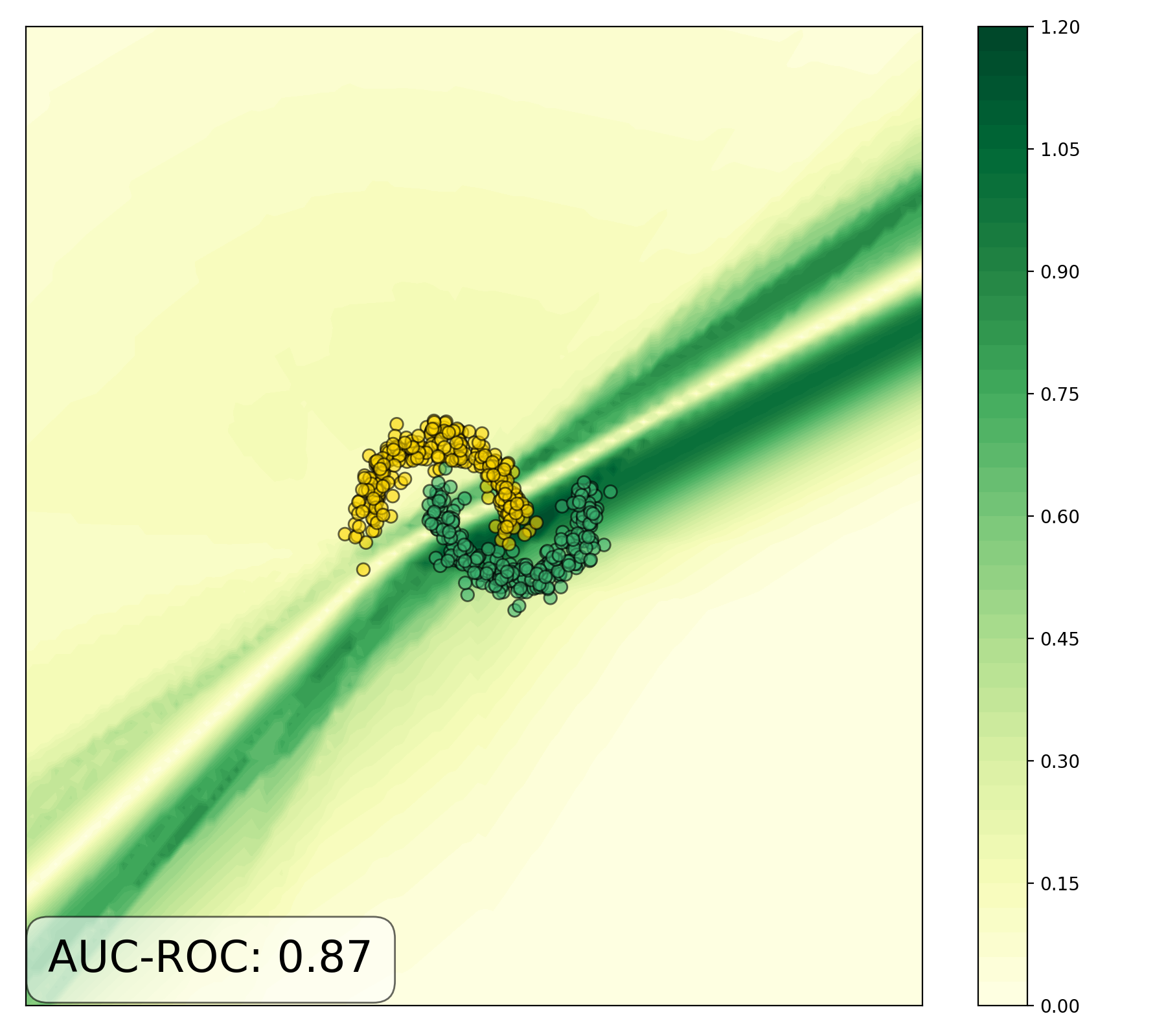} & \includegraphics[width=0.3\textwidth]{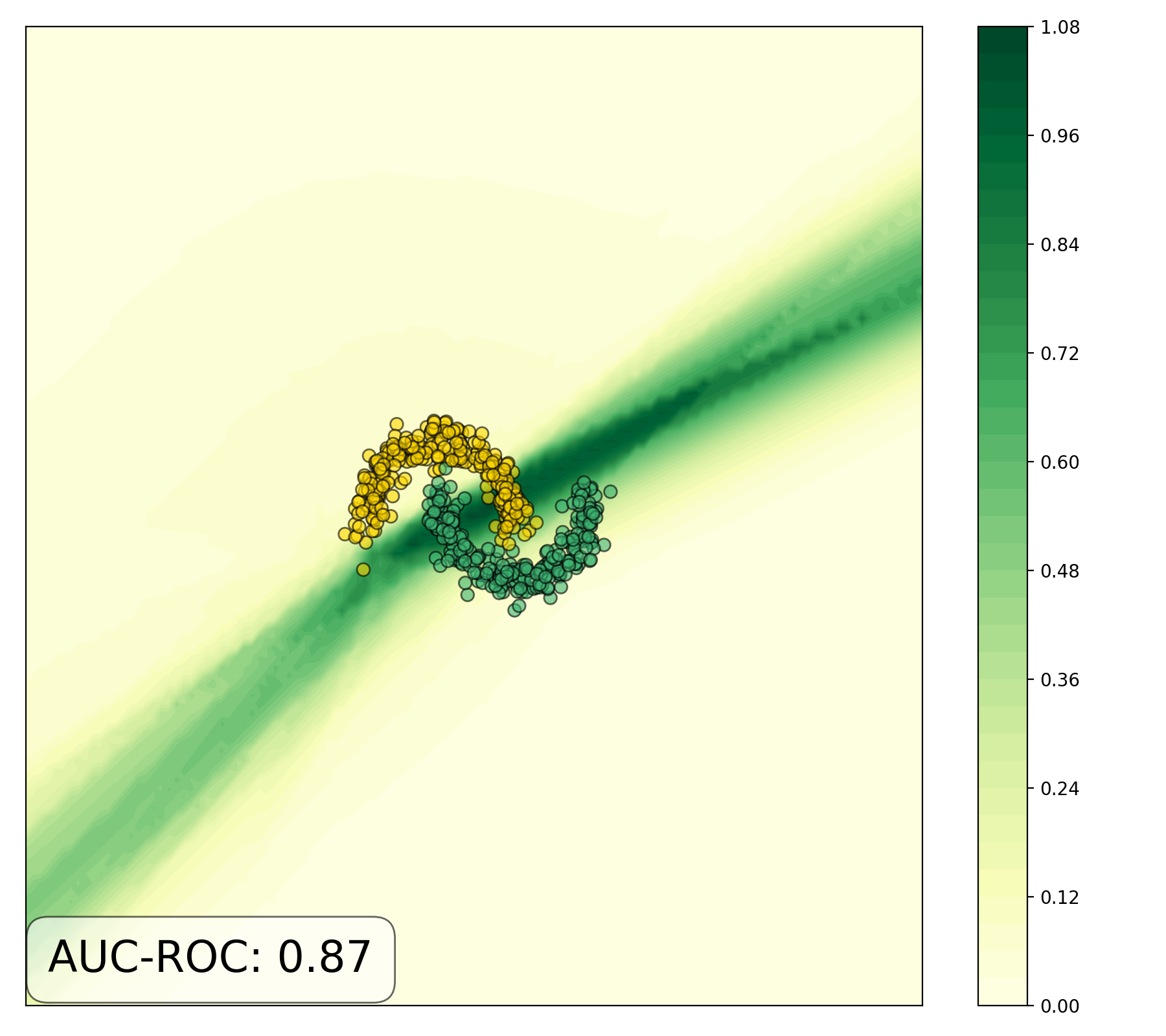}  \\
    \end{tabular}%
    }
    \caption{Uncertainty measured by different metrics for single-instance models (purple plots) and their gradient magnitude (yellow / green plots).}
    \label{fig:app-single-pred-nn}
\end{figure}

In the next figure, Figure \ref{fig:app-multiple-pred-nn}, we observe the uncertainty surfaces for models using multiple network instances. 
%as the AUC-ROC falls behind the other approaches and the resulting surfaces appear mostly flat, implying that the model didn't converge to a good local minimum. 
For the remaining models it is interesting to see that class variance (left column) didn't seem to produce significantly different values across the feature space except for the anchored ensemble. For predictive entropy (central column), we can see a similar behaviour compared to the single-instances models. Interestingly, the ``fuzziness'' of the high-uncertainty region increases with the ensemble and becomes increasing large with its anchored variant. Nevertheless, regions with static levels of certainty still exist in this case. For the mutual information plots (right column), epistemic uncertainty is lowest around the training data, where the model is best specified, which creates another tube-like region of high confidence even where there is no training data, an effect that is reduced with the neural ensemble and almost completely solved by the anchored ensemble. For all metrics, we see a magnitude close to zero for the uncertainty gradient away from the training data, except for the decision boundaries, as discussed in Section \ref{sec:experiments}.

\begin{figure*}[h!]
    \centering
    \resizebox{1.6\columnwidth}{!}{%
        \begin{tabular}{rlll}
             & \multicolumn{1}{c}{Class variance} & \multicolumn{1}{c}{Predictive Entropy} & \multicolumn{1}{c}{Mutual Information} \\
             \multirow{2}{*}{\rotatebox{90}{\texttt{MCDropout}\hspace{-0.5cm}}}        & \includegraphics[width=0.24\textwidth]{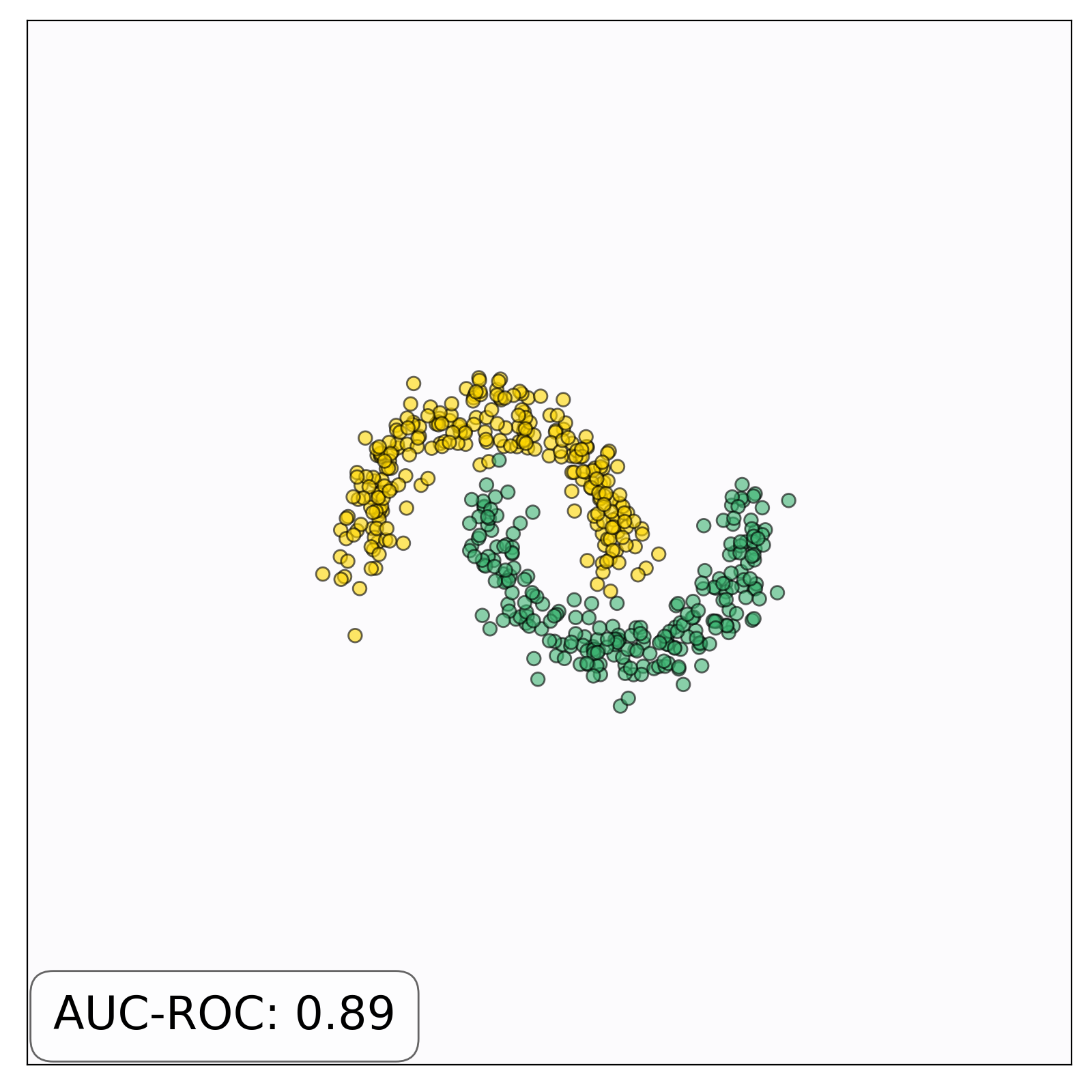} & \includegraphics[width=0.24\textwidth]{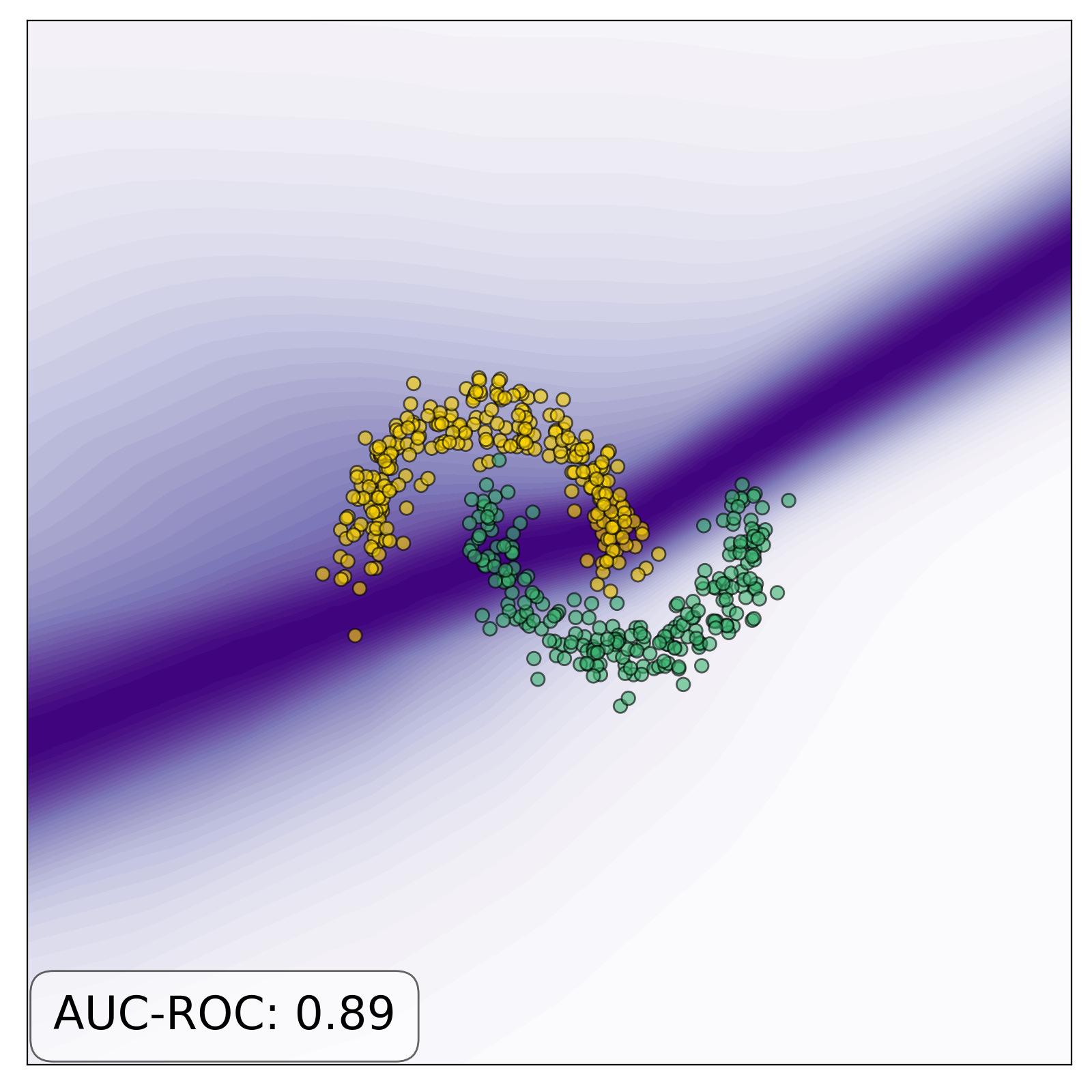}  & \includegraphics[width=0.24\textwidth]{img/mcdropout_mutual_information.png} \\
             & \includegraphics[width=0.25\textwidth]{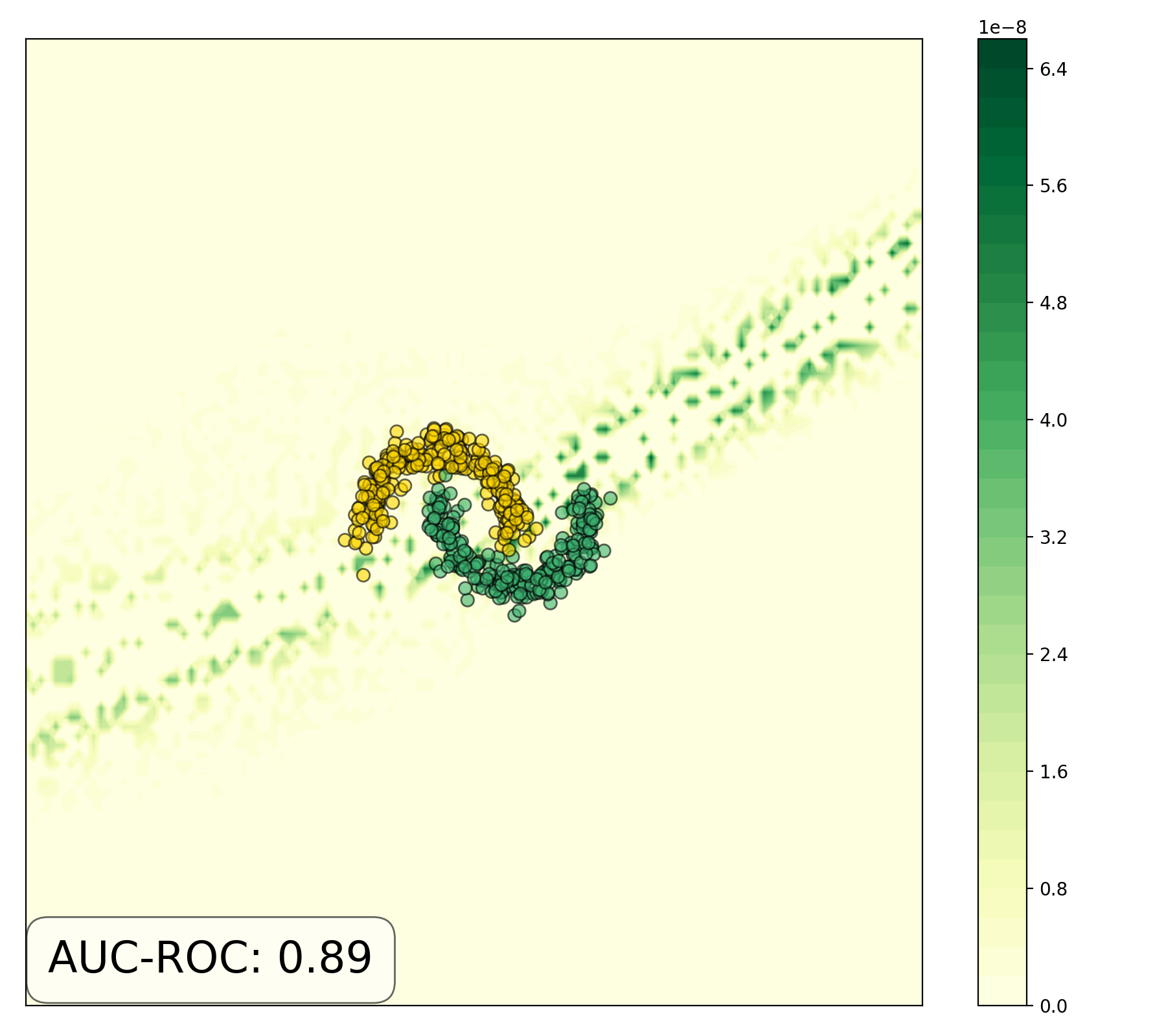} & \includegraphics[width=0.25\textwidth]{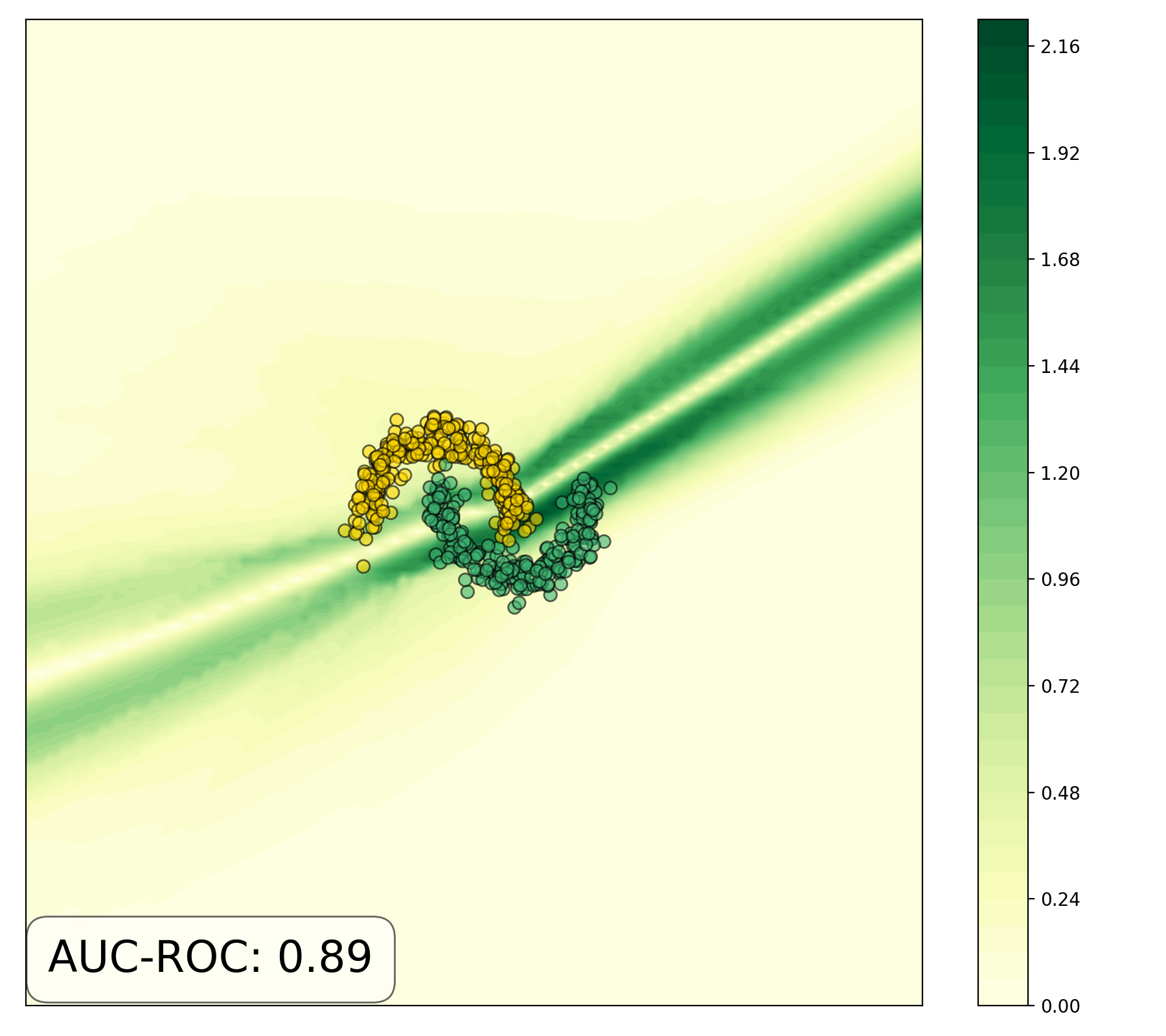}  & \includegraphics[width=0.25\textwidth]{img/mcdropout_mutual_information_grads.png} \\
             \midrule
             %\rotatebox{90}{\hspace{2.5cm}\texttt{BBB}}               & \includegraphics[width=0.25\textwidth]{img/bbb_var.png} & \includegraphics[width=0.25\textwidth]{img/bbb_entropy.png}  & \includegraphics[width=0.25\textwidth]{img/bbb_mutual_information.png}  \\
             \multirow{2}{*}{\rotatebox{90}{\texttt{NNEnsemble}}\hspace{-3.5cm}}         & \includegraphics[width=0.24\textwidth]{img/nnensemble_var.png} & \includegraphics[width=0.24\textwidth]{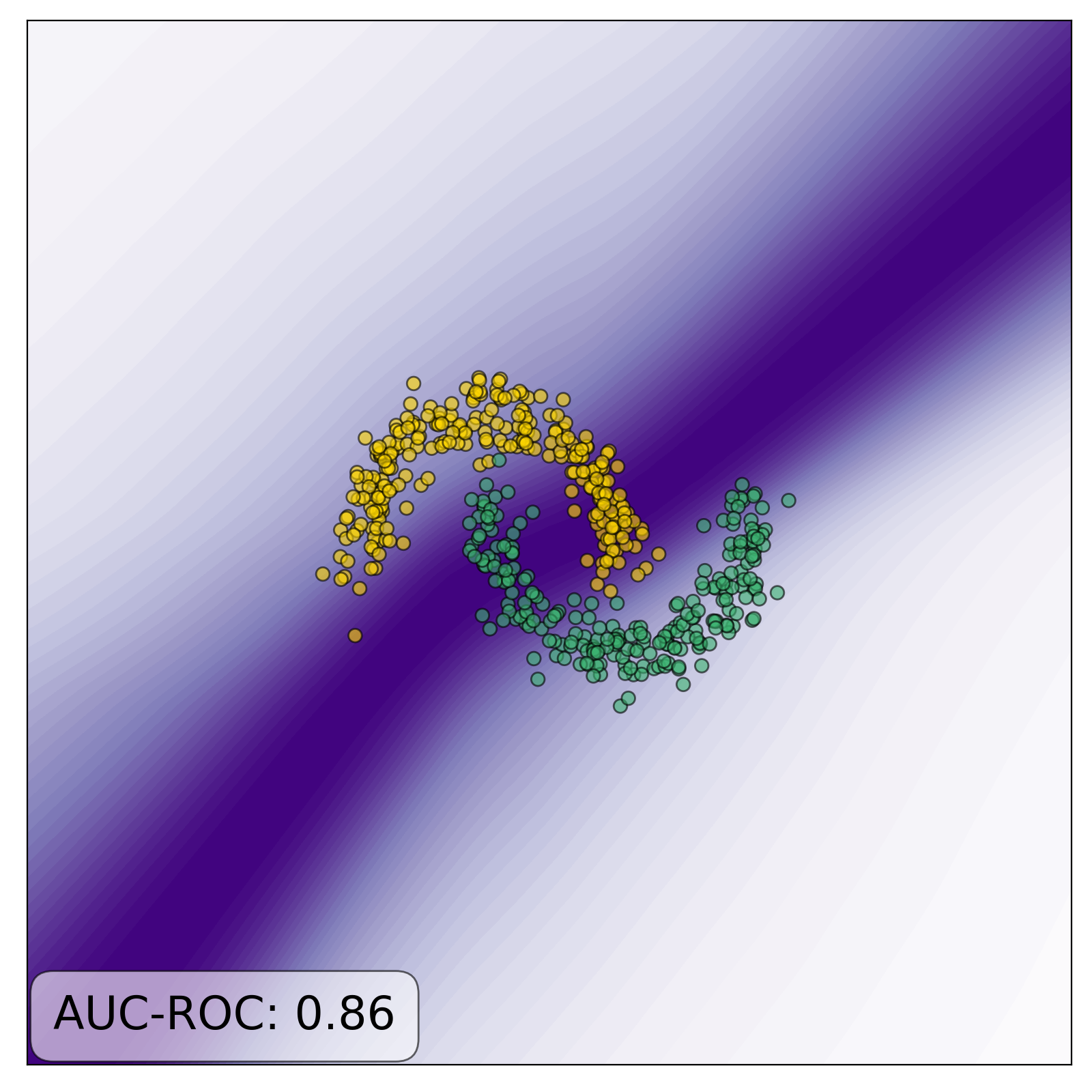}  & \includegraphics[width=0.24\textwidth]{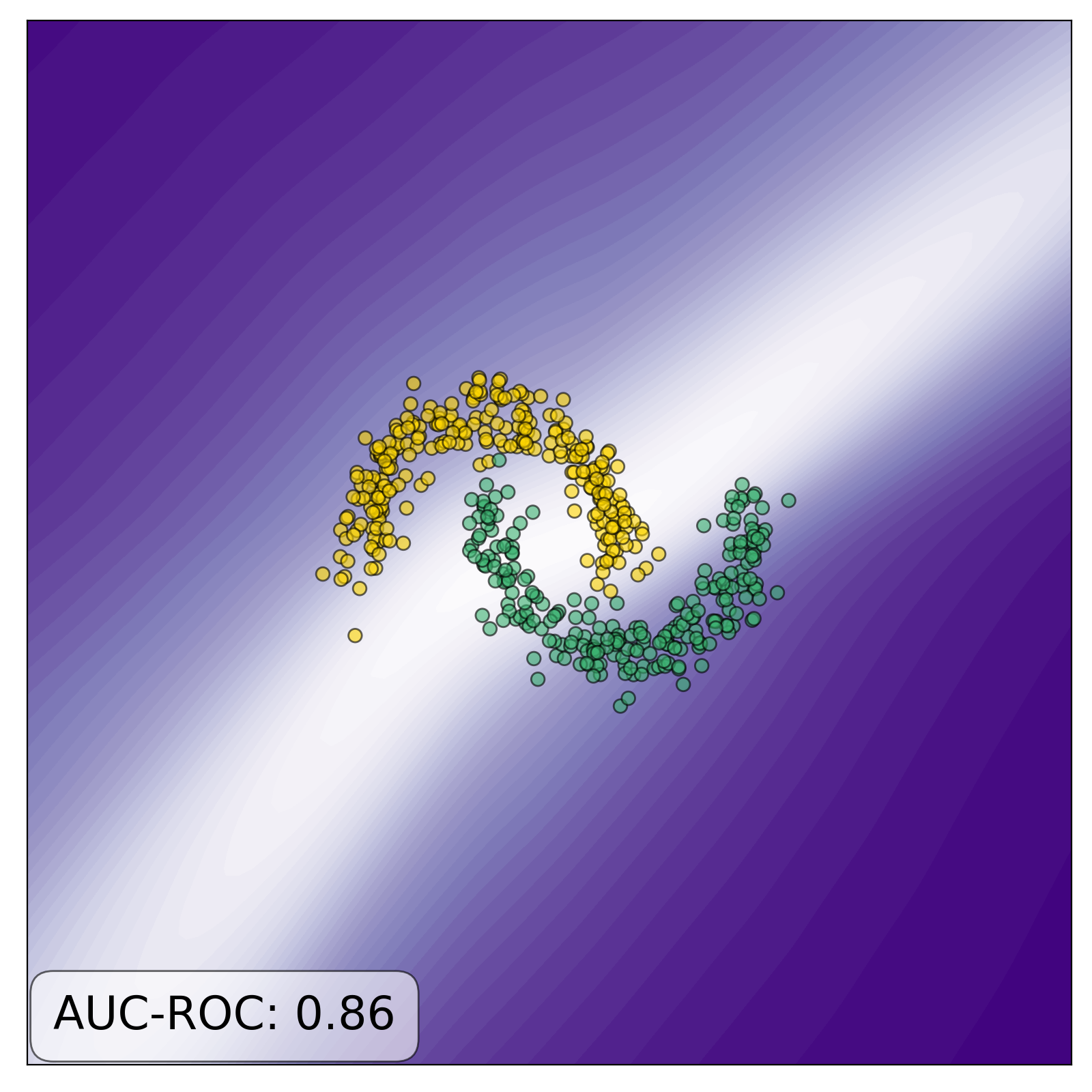} \\
              & \includegraphics[width=0.25\textwidth]{img/nnensemble_var_grads.png} & \includegraphics[width=0.25\textwidth]{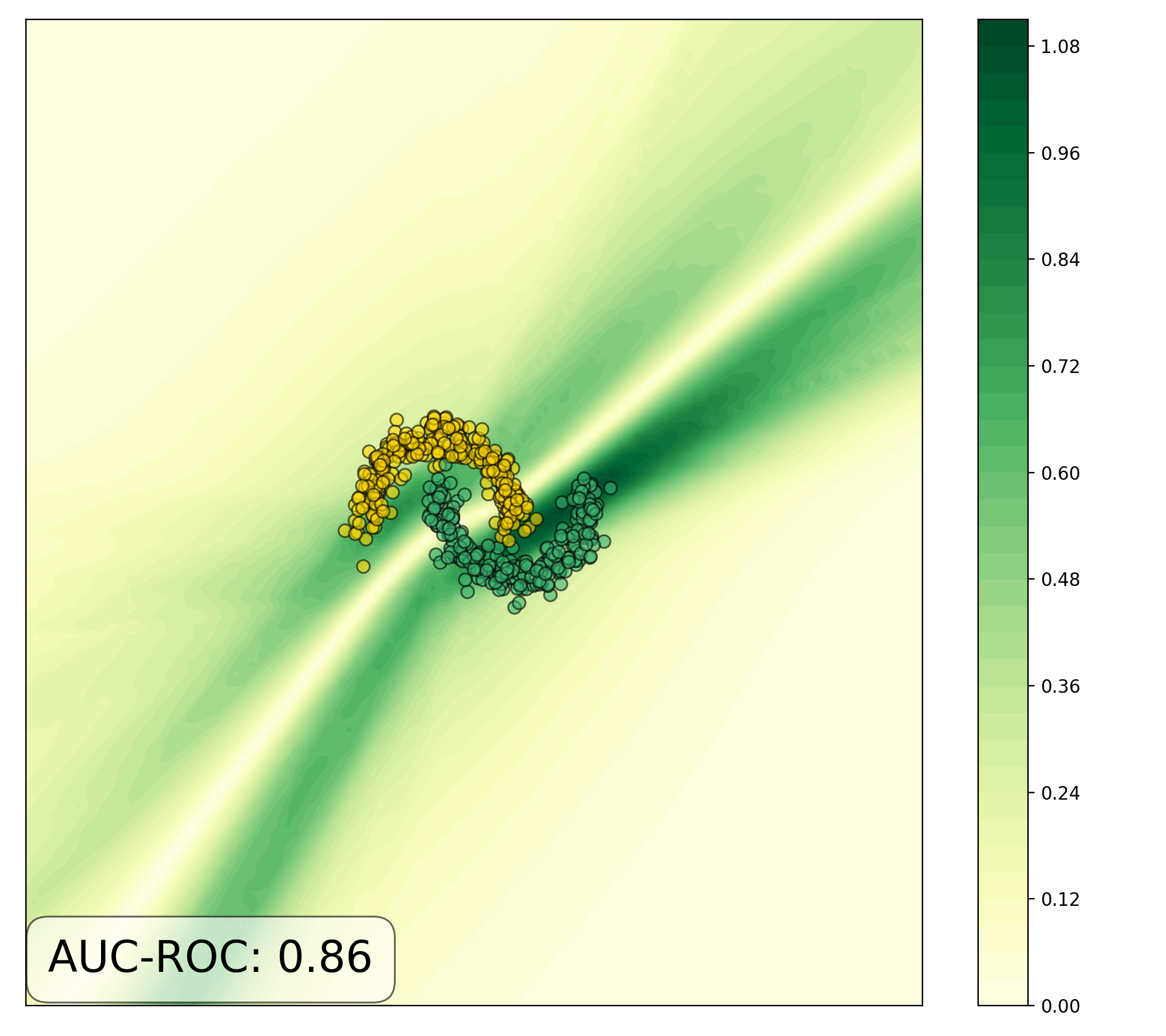}  & \includegraphics[width=0.25\textwidth]{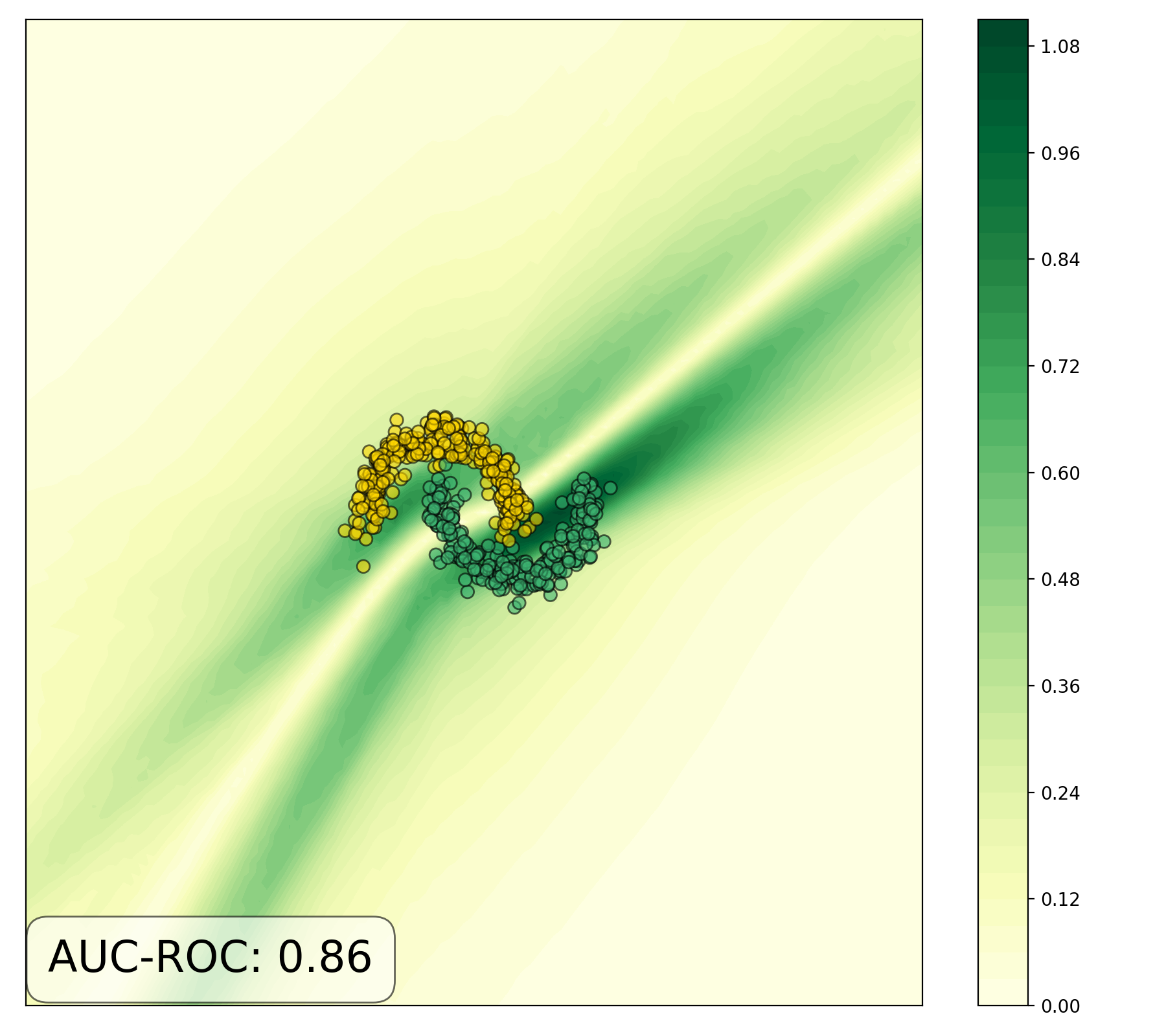} \\
              \midrule
            \rotatebox{90}{\hspace{-1.5cm}\texttt{AnchoredNNEnsemble}}  & \includegraphics[width=0.24\textwidth]{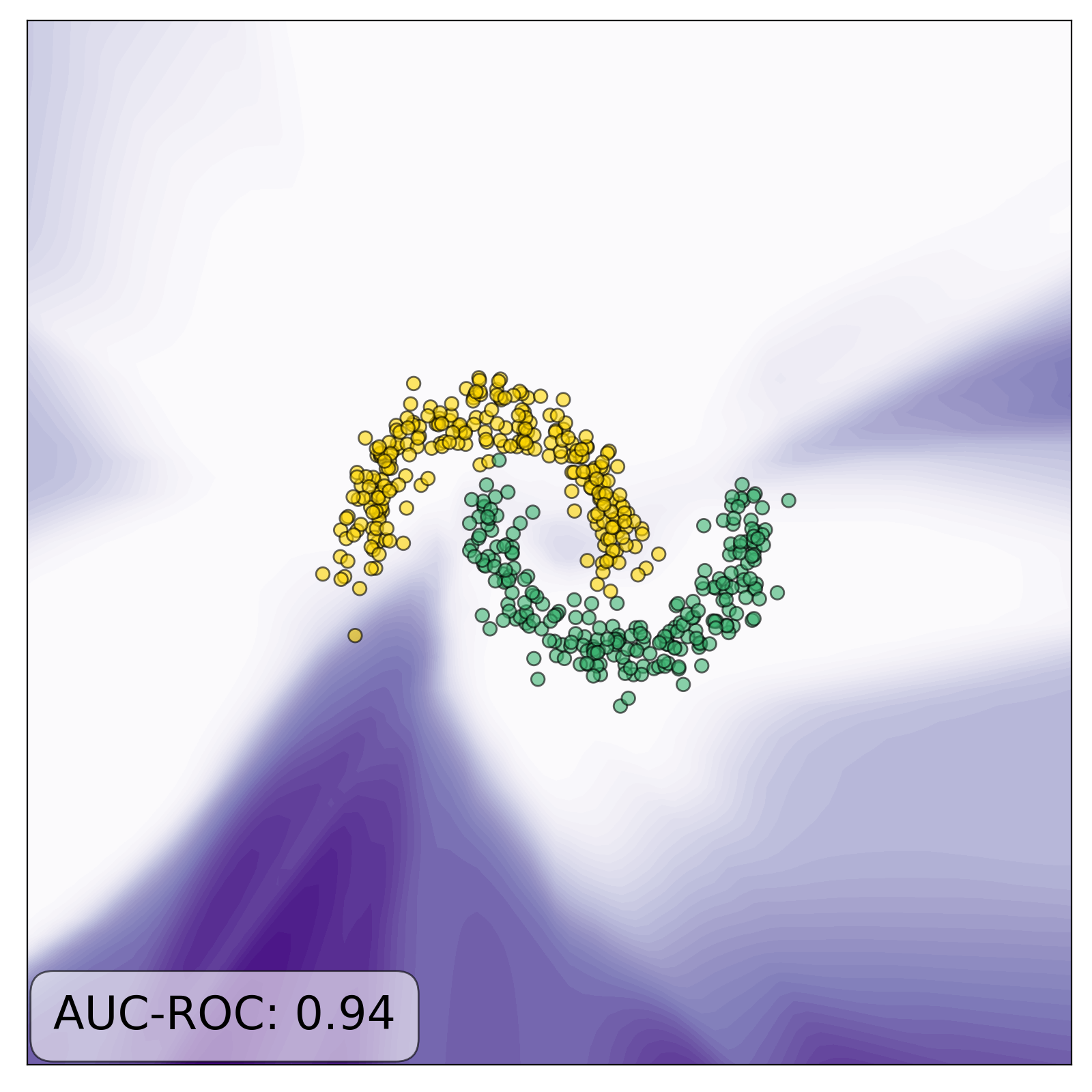} & \includegraphics[width=0.24\textwidth]{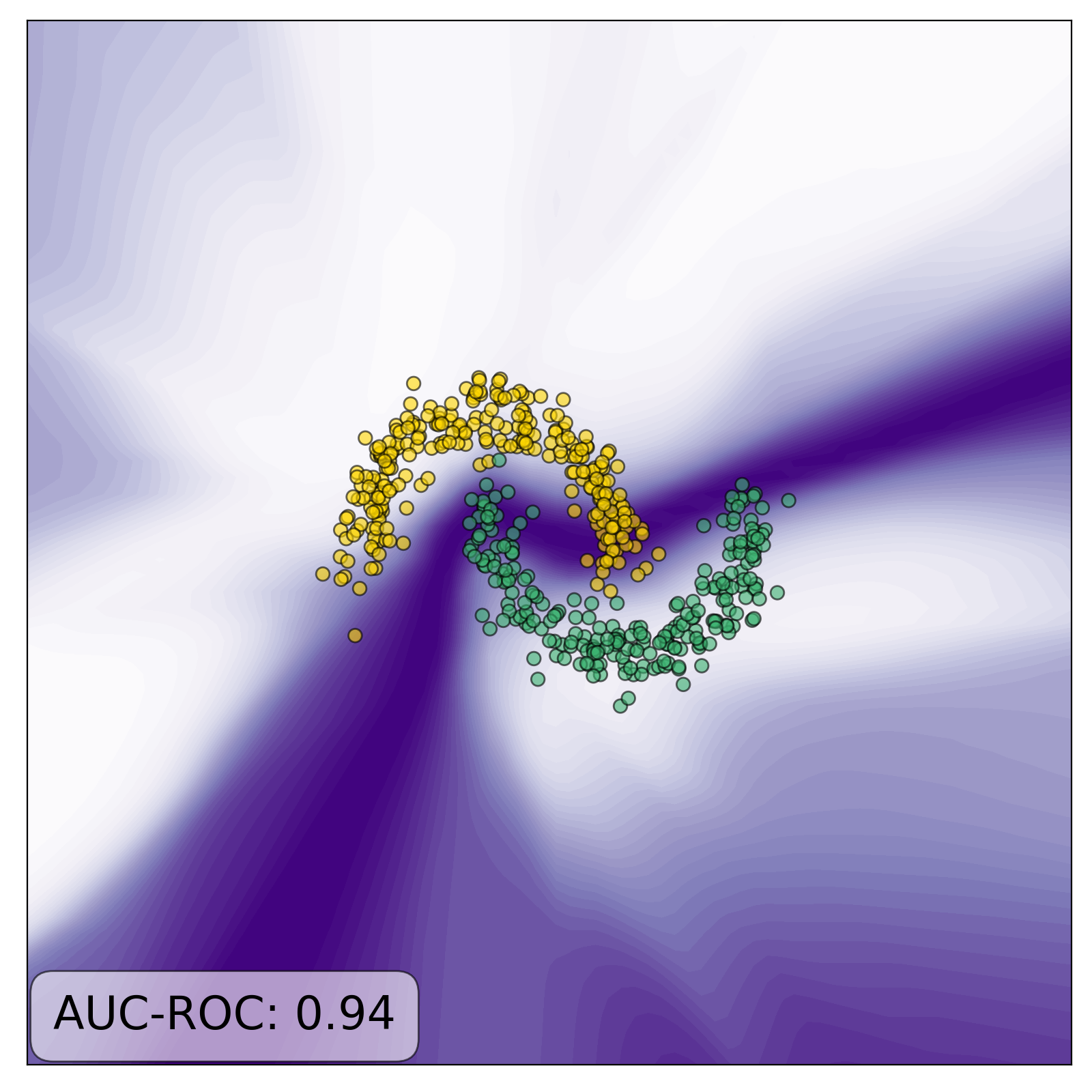}  & \includegraphics[width=0.24\textwidth]{img/anchorednnensemble_mutual_information.png} \\
            & \includegraphics[width=0.25\textwidth]{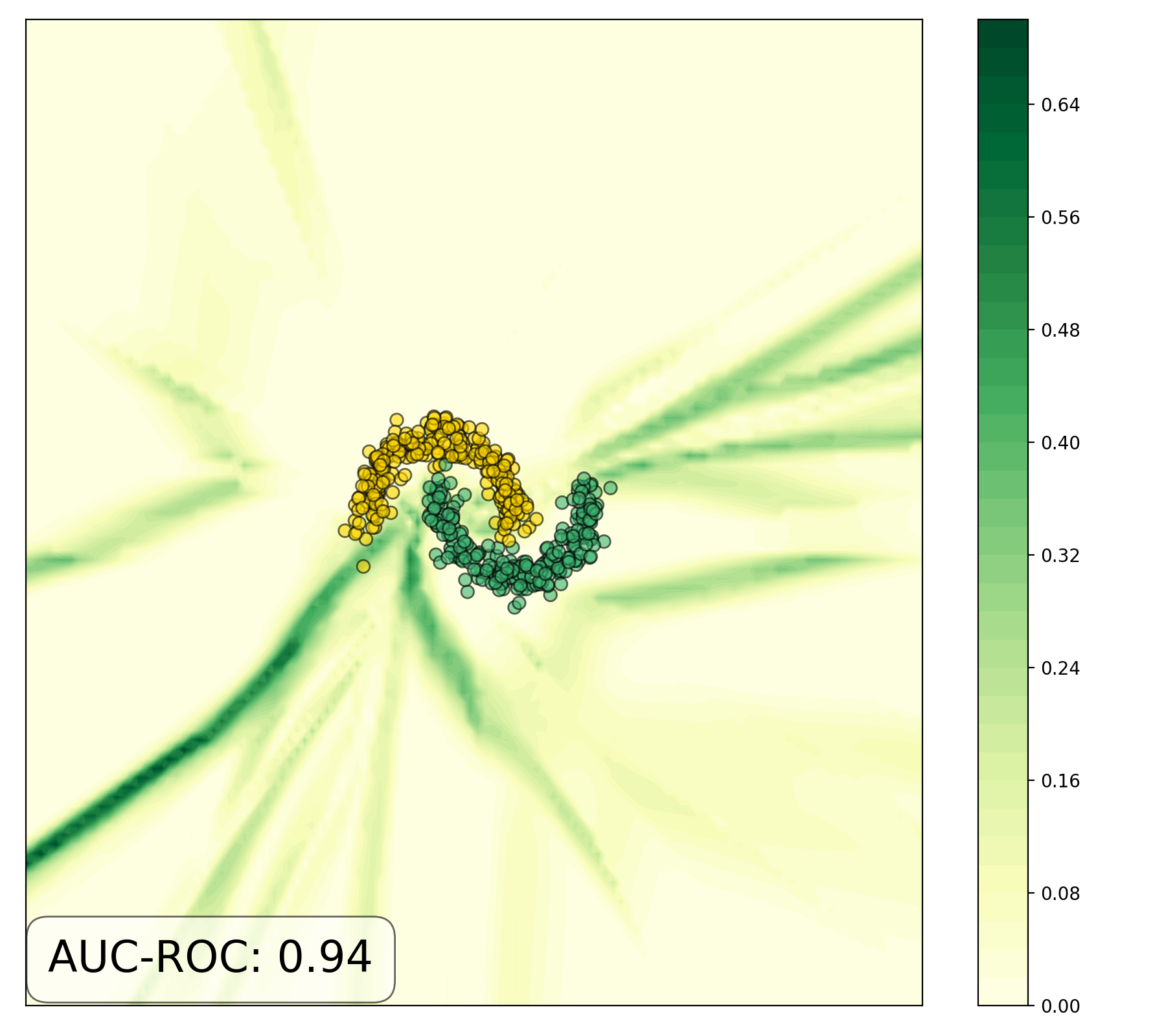} & \includegraphics[width=0.25\textwidth]{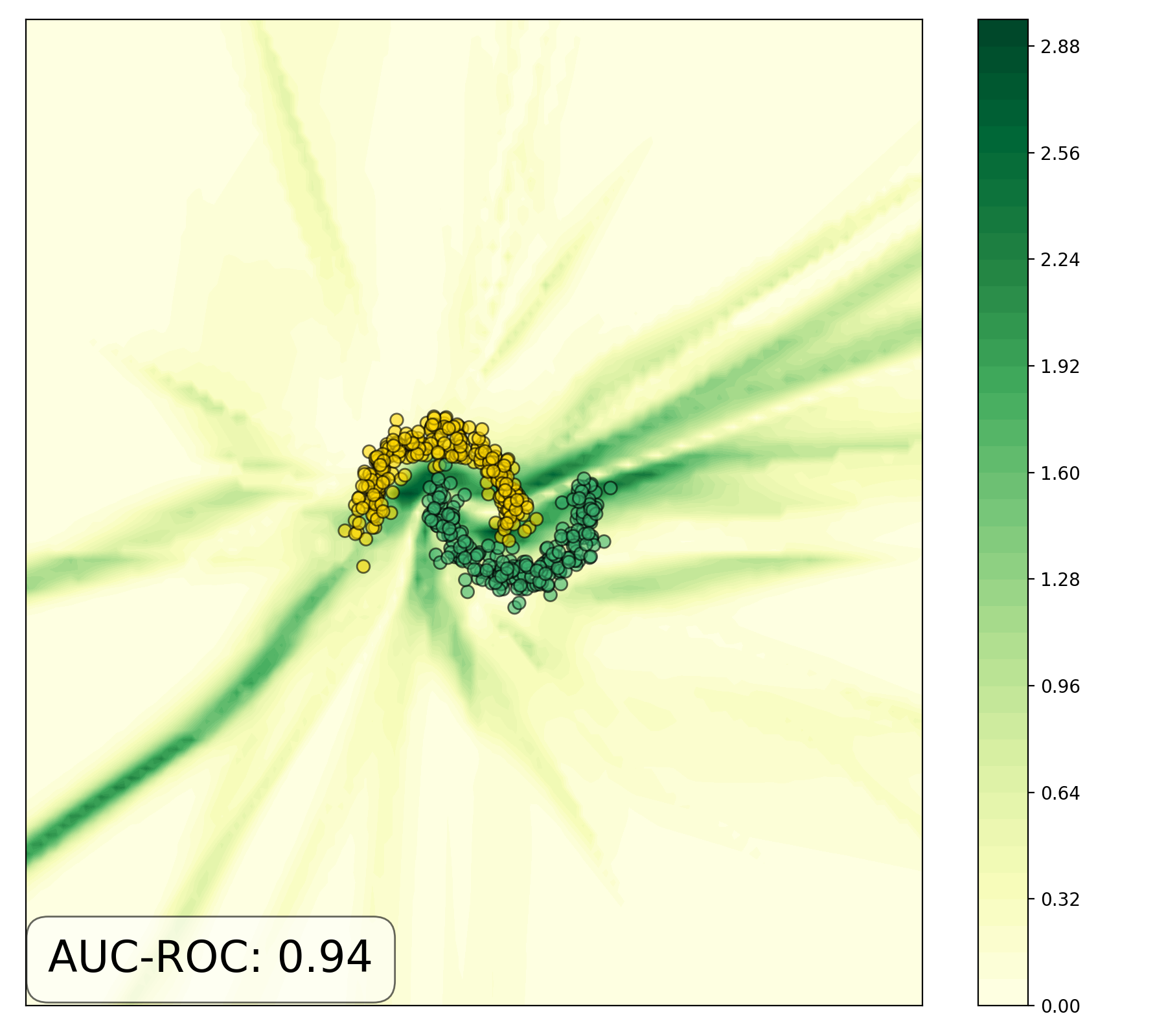}  & \includegraphics[width=0.25\textwidth]{img/anchorednnensemble_mutual_information_grads.png} \\
        \end{tabular}
    }
    \caption{Uncertainty measured by different metrics for multi-instance models (purple plots) and the gradient of the uncertainty score w.r.t to the input (yellow / green plot).}
    \label{fig:app-multiple-pred-nn}

\end{figure*}
\end{document}